\documentclass[10pt, conference, twocolumn, compsocconf]{llncs}

\usepackage{cite}
\usepackage{amsmath,amssymb,amsfonts}

\usepackage{graphicx}
\usepackage{textcomp}

\usepackage{alltt}
\usepackage{amssymb}
\usepackage{caption}
\usepackage{paralist}
\usepackage{nicefrac}
\usepackage{booktabs}
\usepackage{xspace}
\usepackage{amsmath}
\usepackage{marvosym}
\usepackage{wasysym}
\usepackage{verbatim}
\usepackage{float}
\usepackage{hyperref}
\usepackage{multirow}
\usepackage{cite}
\usepackage[capitalize]{cleveref}
\usepackage[per-mode=symbol,detect-all]{siunitx}
\usepackage{graphics}
\usepackage{amssymb}
\usepackage{enumitem}

\usepackage{booktabs}
\usepackage{colortbl}
\usepackage{ifthen}
\usepackage{mathdots}
\usepackage{arydshln}
\usepackage[outline]{contour}

\usepackage{wrapfig}
\usepackage{subfig}
\usepackage{algorithm}
\usepackage[noend]{algpseudocode}
\usepackage{MnSymbol}
\usepackage{placeins}
\usepackage{diagbox}

\usepackage{xcolor,pifont}
\newcommand*\colourcheck[1]{%
	\expandafter\newcommand\csname #1check\endcsname{\textcolor{#1}{\ding{52}}}%
}
\newcommand*\colourxmark[1]{%
	\expandafter\newcommand\csname #1xmark\endcsname{\textcolor{#1}{\ding{56}}}%
}
\colourcheck{green}
\colourxmark{red}

\definecolor{navy}{RGB}{0,0,128}

\usepackage{tikz,ifthen,pgfplots}
\usetikzlibrary{arrows,trees,backgrounds,automata,shapes,decorations,plotmarks,fit,calc,positioning,shadows,chains}
\usetikzlibrary{quotes,arrows.meta}
\tikzstyle{every pin edge}=[<-,shorten <=1pt]
\tikzstyle{neuron}=[circle,fill=black!25,minimum size=17pt,inner sep=0pt]
\tikzstyle{input neuron}=[neuron, fill=green!50]
\tikzstyle{output neuron}=[neuron, fill=red!50]
\tikzstyle{hidden neuron}=[neuron, fill=blue!50]
\tikzstyle{small neuron}        =[hidden neuron, draw, minimum size=15pt]
\tikzstyle{small input neuron}  =[input neuron , draw, minimum size=15pt]
\tikzstyle{small output neuron} =[output neuron, draw, minimum size=15pt]
\tikzstyle{annot} = [text width=4em, text centered]
\tikzstyle{nnedge} = [-{stealth},shorten >=0.1cm, shorten <=0.05cm,line width=0.8pt,black]
\tikzstyle{edge} = [->,line width = 0.3pt, shorten >=0.2cm]
\tikzstyle{edgeWide} = [->,line width = 2pt, , shorten >=0.2cm]
\usetikzlibrary{calc}

\tikzset{every picture/.style={line width=0.75pt}} 
\tikzstyle{BadSquare}=[rectangle,fill=red!30!white,minimum size=25pt,inner 
sep=0pt]
\tikzstyle{InitSquare}=[rectangle,fill=green!30!white,minimum size=25pt,inner 
sep=0pt]

\newcommand{\mysubsection}[1]{\medskip\noindent\textbf{#1}}

\newcommand{\relu}{\text{ReLU}\xspace}

\newcommand{\verifyexplanation}{\text{\verify((Explanation$\setminus$\{f\})=v,N,$Q_{\neg
			c}$)}\xspace}

\newcommand{\verifymultistepexplanation}{\text{\verify($(E_1,\ldots,E_i\setminus
 f)$=$\mathcal{E}_{S_{[i]}},N,Q_{\neg a_i}$)}\xspace}

\newcommand{\verifycontrastiveexamplesingle}{\text{\verify($F\setminus 
c=s_i$,$N$,$Q_{\neg
			a_i}$)}\xspace}

\newcommand{\verifymultistepminimumexplanation}{\text{\verify($E\cdot 
(F')$=$\mathcal{E}_{S_{[i]}},N,Q_{\neg a_i}$)}\xspace}

\newcommand{\explanation}{\text{Explanation}\xspace}

\newcommand{\lime}{\texttt{LIME}\xspace}
\newcommand{\shap}{\texttt{SHAP}\xspace}

\newcommand{\sat}{\texttt{SAT}\xspace}

\newcommand{\unsat}{\texttt{UNSAT}\xspace}

\newcommand{\rcTwo}{\texttt{RC-2}\xspace}
\newcommand{\maxSat}{\texttt{MaxSAT}\xspace}
\newcommand{\pysat}{\texttt{PySat}\xspace}
\newcommand{\verify}{\texttt{verify}\xspace}

\newcommand{\up}{\texttt{UP}\xspace}
\newcommand{\down}{\texttt{DOWN}\xspace}

\newcommand{\allExplanationRecursiveSearchWithLineBreak}{\texttt{All-Explanation-\\Recursive-Search}\xspace}
\newcommand{\RIE}{\texttt{RIE}\xspace}

\newcommand{\enumeratecxpssinglestepWithLineBreak}{\texttt{Enumerate-All-Cxps-In-\\Single-Step}\xspace}

\newcommand{\forwardOutput}{\texttt{FORWARD}\xspace}
\newcommand{\leftOutput}{\texttt{LEFT}\xspace}
\newcommand{\rightOutput}{\texttt{RIGHT}\xspace}

\newcommand{\marabou}{\textit{Marabou}\xspace}

\usepackage{amsthm}
\newtheorem{lemma}{Lemma}

\usepackage{bussproofs}
\usepackage{gensymb}

\newif\ifcomments
\commentstrue

\newif\ifoutline
\outlinefalse

\newif\iflong
\longtrue

\begin{document}

\IEEEoverridecommandlockouts
\DeclareRobustCommand{\IEEEauthorrefmark}[1]{\smash{\textsuperscript{\footnotesize
 #1}}}

	
	\title{Formally Explaining Neural Networks\\within Reactive Systems} 
	
	\author{
		\IEEEauthorblockN{
			Shahaf Bassan\IEEEauthorrefmark{1,$^*$},
			Guy Amir\IEEEauthorrefmark{1,$^*$},
			Davide Corsi\IEEEauthorrefmark{2},
			Idan Refaeli\IEEEauthorrefmark{1},
			and
			Guy Katz\IEEEauthorrefmark{1}
		}
		\IEEEauthorblockA{
			\IEEEauthorrefmark{1}The Hebrew University of Jerusalem,
			 \texttt {\{shahaf, guyam, idan0610, guykatz\}@cs.huji.ac.il}
			\\ 
			\IEEEauthorrefmark{2}University of Verona, \texttt 
			{davide.corsi@univr.it}
		}
		
		\thanks{$*$ Both authors contributed equally.}
	}

\maketitle
\thispagestyle{plain}
\pagestyle{plain}

\begin{abstract}
  Deep neural networks (DNNs) are increasingly being used as
  controllers in reactive systems. However, DNNs are highly opaque,
  which renders it difficult to explain and justify their actions. To
  mitigate this issue, there has been a surge of interest in explainable AI
  (XAI) techniques, capable of  pinpointing the input features that caused the DNN to
  act as it did. Existing XAI techniques typically
  face two limitations:
  \begin{inparaenum}[(i)]
  \item they are heuristic, and do not provide
    formal guarantees that the explanations are correct; and
  \item they often apply to ``one-shot'' systems,
    where the DNN is invoked independently of past invocations, as
    opposed to reactive systems.
  \end{inparaenum}
  Here, we begin bridging this gap, and propose a formal
  DNN-verification-based XAI
  technique for reasoning about multi-step, reactive systems. We suggest methods for
  efficiently calculating succinct explanations, by exploiting the
  system's transition constraints in order to curtail the search space
  explored by the underlying verifier.  We evaluate our approach on
  two popular benchmarks from the domain of automated navigation; and
  observe that our methods allow the efficient computation of
  minimal and minimum explanations, significantly outperforming the
  state of the art. We also demonstrate that our methods
  produce formal explanations that are more reliable than competing,
  non-verification-based XAI techniques.
\end{abstract}


\section{Introduction}
\label{sec:Introduction}

Deep neural networks (DNNs)~\cite{lecun2015deep} are used in numerous key
domains, such as computer vision~\cite{krizhevsky2017imagenet},
natural language processing~\cite{devlin2018bert}, computational
biology~\cite{angermueller2016deep}, and more~\cite{CoMaPoFaCa23}. However, despite their tremendous
success, DNNs remain ``black boxes'', uninterpretable by
humans. This issue is concerning, as DNNs are prone to critical
errors~\cite{staff2019case, zhou2019metamorphic} and unexpected
behaviors~\cite{angwin2016machine, goodfellow2014explaining}.

DNN opacity has prompted significant research on explainable AI (XAI)
techniques~\cite{ribeiro2016should, ribeiro2018anchors,
  lundberg2017unified}, aimed at explaining the decisions made by
DNNs, in order to increase their trustworthiness and reliability. Modern XAI methods are useful and scalable, but
they are typically heuristic; i.e., there is no provable guarantee that the
produced explanation is correct~\cite{ignatiev2019validating, camburu2019can}. This
hinders the applicability of these approaches to critical systems, 
where regulatory bars are high~\cite{marques2022delivering}.

These limitations provide ample motivation for 
\emph{formally} explaining DNN
decisions~\cite{ignatiev2020towards, marques2022delivering,
  hoffman2018metrics, camburu2019can}. And indeed, the formal
verification community has suggested harnessing recent developments in DNN
verification~\cite{wang2021beta, MaCoCiFa23, GeMiDrTsChVe18, PoAbKr20,
  OkWaSeHa20, HuKwWaWu17, muller2021prima,SoTh19, SeDeDrFrGhKiShVaYu18, GoKaPaBa18, WuOzZeIrJuGoFoKaPaBa20, BaShShMeSa19, CoMaFa21}
to produce  provable explanations for
DNNs~\cite{ignatiev2019abduction, bassan2022towards,
  ignatiev2020towards}. Typically,
such approaches consider a particular input to the DNN, and return a
subset of its features that caused the DNN to classify the input as it
did. These subsets are called \emph{abductive explanations},
\emph{prime implicants} or
\emph{PI-explanations}~\cite{shih2018symbolic, ignatiev2019abduction,
  bassan2022towards}. This line of work constitutes a promising step
towards more reliable XAI; but so far, existing work has focused on
explaining decisions of ``one-shot'' DNNs, such as image and tabular data classifiers~\cite{ignatiev2019abduction, bassan2022towards, ignatiev2020contrastive},
and has not addressed more complex systems.

Modern DNNs are often used as controllers within elaborate reactive
systems, where a DNN's decisions affect its future invocations.
A prime example is \emph{deep reinforcement learning}
(\emph{DRL})~\cite{Li17}, where DNNs learn
control policies for complex systems~\cite{silver2017mastering,
  zhang2019end, lekharu2020deep, luong2019applications, brunke2022safe,
  aradi2020survey, PoCoMa21}. Explaining the decisions of DRL agents
(XRL)~\cite{puiutta2020explainable, madumal2020explainable,
  heuillet2021explainability, juozapaitis2019explainable} is an
important domain within XAI; but here too, modern XRL
techniques are heuristic, and do not provide formally correct explanations.
    
In this work, we make a first attempt at formally defining abductive
explanations for \emph{multi-step decision processes}. We propose
novel methods for computing such explanations and supply the
theoretical groundwork for justifying the soundness of these
methods. Our framework is model-agnostic, and could be applied to
diverse kinds of models; but here, we focus on DNNs, where producing
abductive explanations is known to be quite
challenging~\cite{barcelo2020model, ignatiev2019abduction,
  bassan2022towards}. With DNNs, our technique allows us to reduce the
number of times a network has to be unrolled, circumventing a
potential exponential blow-up in runtime; and also allows us to
exploit the reactive system's transition constraints, as well as the
DNN's sensitivity to small input perturbations, to curtail the search
space even further.
    
For evaluation purposes, we implemented our approach as a
proof-of-concept tool, which is publicly available as an artifact accompanying 
this paper~\cite{ArtifactRepository}. We used this tool to automatically 
generate explanations
for two popular DRL benchmarks: a navigation system on an abstract,
two-dimensional grid, and a real-world robotic navigation system. Our
evaluation demonstrates that our methods significantly outperform
state-of-the-art, rigorous methods for generating abductive
explanations, both in terms of efficiency and in the size of
the explanation generated. When comparing our approach to modern,
heuristic-based XAI approaches, our explanations were found to be
significantly more precise. We regard these results as strong evidence
of the usefulness of applying verification in the context of XAI.
    
The rest of this paper is organized as follows:
Sec.~\ref{sec:background} contains background on DNNs, their
verification, and their formal
explainability. Sec.~\ref{formal-k-step-explanations-section} contains our 
definitions for formal abductive explanations and contrastive examples
for reactive systems. In
Sec.~\ref{sec:formal_k_explanations_computation} we propose
different methods for computing such abductive explanations.
We then evaluate these approaches in Sec.~\ref{sec:evaluation},
followed by a discussion of related work in
Sec.~\ref{sec:RelatedWork}; and we 
conclude in Sec.~\ref{sec:Conclusion}.

\section{Background}
\label{sec:background}
        
\mysubsection{DNNs.} Deep neural networks (DNNs)~\cite{lecun2015deep}
are directed, layered graphs, whose nodes are referred to as 
\emph{neurons}. They propagate data from the first (\emph{input})
layer, through intermediate (\emph{hidden}) layers, and
finally onto an \emph{output} layer. A DNN's output is
calculated by assigning values (representing input \emph{features}) to the 
input layer, and then
iteratively calculating the neurons' values in subsequent layers. In 
classification, each output neuron corresponds to a
\emph{class}, and the input is classified as the class matching the 
greatest output.
Fig.~\ref{fig:neural_network_example_1} depicts a toy DNN.  The input
layer has three neurons and is followed by a weighted-sum layer that
 calculates an affine transformation of the
 input values. For example, given input $V_1=[1,1,1]^T$, the second layer
evaluates to $V_2=[7,8,11]^T$. This is followed by a \relu{} layer,
which applies the $\relu(x)=\max(0,x)$ function to each
value in the previous layer, resulting in $V_3=[7,8,11]^T$. The output
layer computes the weighted sum
$V_4=[15,-4]^T$. Because the first output neuron has the greatest
value, $V_1$ is classified as the
output class corresponding to that neuron.

\begin{figure}[h]
	\begin{center}
		\includegraphics[width=0.3\textwidth]{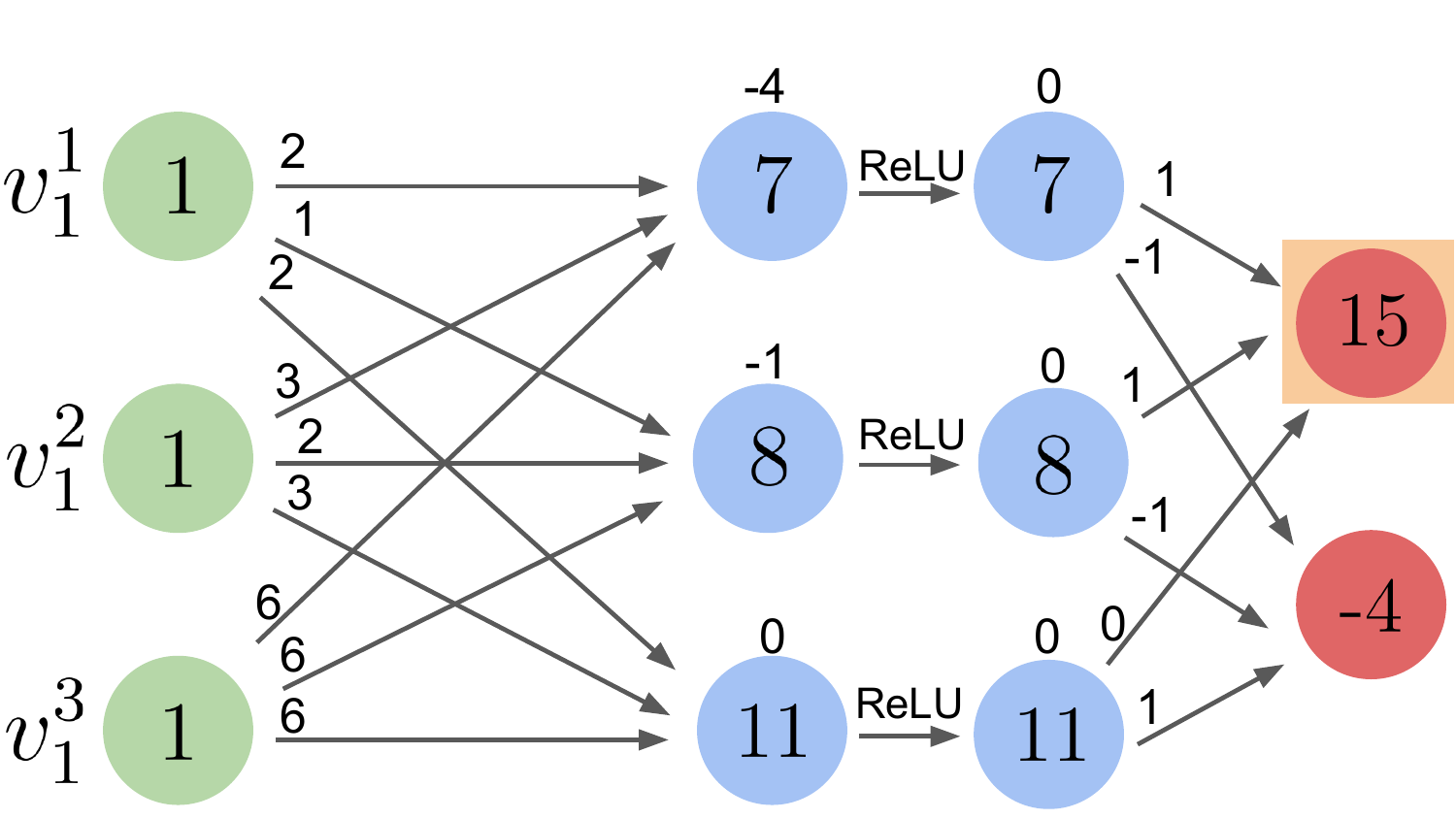}
		\caption{A toy DNN.}
		\label{fig:neural_network_example_1}
	\end{center} 
\end{figure}

\mysubsection{DNN Verification.}  We define a DNN verification query
as a tuple $\langle P, N, Q\rangle$, where $N$ is a DNN that maps an 
input vector $x$ to an output vector $y=N(x)$, $P$ is a predicate over
$x$, and $Q$ is a predicate over $y$~\cite{katz2017reluplex}. A DNN
verifier needs to answer whether there exists some input $x'$ that
satisfies $P(x') \wedge Q(N(x'))$ (a \sat{} result) or not (an
\unsat{} result). It is common to express $P$ and $Q$ in the logic of
real arithmetic~\cite{LiArLaBaKo20}. The problem of verifying DNNs is
known to be NP-Complete~\cite{katz2017reluplex}.
	
\mysubsection{Formal Explanations for Classification DNNs.}
A classification problem
is a tuple $\langle F, D, K, N\rangle$, where
\begin{inparaenum}[(i)]
\item $F=\{1,\ldots,m\}$ is the feature set;
\item $D=\{D_1,D_2,\ldots,D_m\}$ are the domains of
  individual features, and the
  entire feature space is
  $\mathbb{F}=({D_1 \times D_2 \times \ldots\times D_m}$);
\item $K=\{c_1,c_2,\ldots,c_n\}$ represents the set of all classes; and
\item $N:\mathbb{F}\to K$ is the classification function, represented by a neural network.
\end{inparaenum}
 A \emph{classification instance} is a pair $(v,c)$, where
$v\in \mathbb{F}$, $c\in K$, and $c=N(v)$. Intuitively, this 
means that  $N$ maps the input $v$ to class $c$.
 
 Formally explaining the instance $(v,c)$ entails determining \emph{why}  $v$ is
classified as $c$. An \emph{explanation} (also known as an \emph{abductive explanation}) is
defined as a subset of features, $E\subseteq F$, such that fixing
these features to their values in $v$ guarantees that the input is
classified as $c$, regardless of features in
$F\setminus E$. The features \emph{not}
part of the explanation are \emph{``free''} to take on any arbitrary
value, but cannot affect the classification. 
Formally, given 
 an input $v=(v_1,\ldots,v_m)\in \mathbb{F}$ classified by the neural
network to $N(v)=c$, we define an explanation  as a subset of
features $E\subseteq F$, such that:
\begin{equation}
  \label{eq:explanation}
  \forall x\in \mathbb{F}.\quad \bigwedge_{i\in E}(x_{i}=v_{i})\to(N(x)=c)
\end{equation}

We demonstrate formal explanations using the running example from
Fig.~\ref{fig:neural_network_example_1}.  For simplicity, assume
that each input can only take the values $0$ or $1$.
Fig.~\ref{fig:neural_network_explanation} shows that the set
$\{ v_1^1, v_1^2 \}$ is an explanation for the input vector
$V_1=[1,1,1]^T$: setting the first two features in
$V_1$  to $1$ ensures that the classification is unchanged, regardless
of the values the third feature takes.

\begin{figure}[h]
  \centering
  {\includegraphics[width=0.3\textwidth]{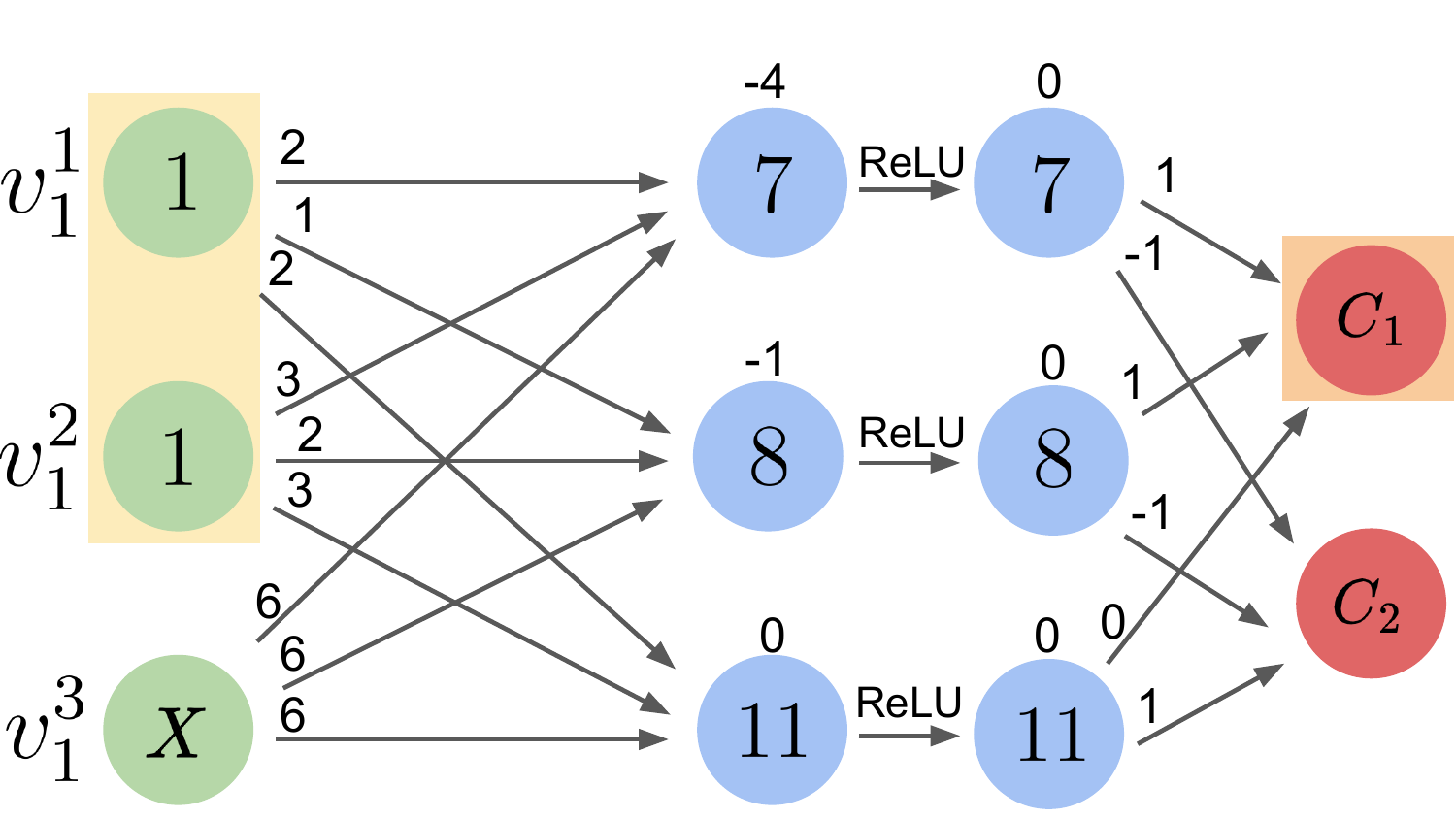}}
  \hfill
  \caption{$\{ v_1^1, v_1^2 \}$ is  an explanation for input $V_1=[1,1,1]^T$.}
  \label{fig:neural_network_explanation}
\end{figure}

A candidate explanation $E$ can be verified through a 
verification query
$\langle P,N,Q\rangle = \langle E=v,N,Q_{\neg c}\rangle$, where $E=v$ means that all of the features
in $E$ are set to their corresponding values in $v$, and  $Q_{\neg c}$
implies that the classification of this query is \emph{not} $c$.  If this query is  \unsat, then $E$ is a
valid explanation for the instance $(v,c)$.

It is straightforward to show that the  set of all features is
a trivial explanation. However, smaller
explanations typically provide more meaningful information
regarding the decision of the classifier; and we thus 
focus on finding \emph{minimal} and \emph{minimum} explanations. 
 A \emph{minimal explanation} is an explanation $E\subseteq F$
 that ceases to be an
 explanation if any of its features are removed:
 \begin{equation}
   \begin{aligned}
     &(\forall x\in \mathbb{F}.\quad \bigwedge_{i\in
       E}(x_{i}=v_{i})\to(N(x)=c))\ \wedge \\
     &(\forall j\in E.\  \exists y\in \mathbb{F}.\quad \bigwedge_{i\in E\setminus j}(y_{i}=v_{i})\wedge(N(y)\neq c))
   \end{aligned}
\end{equation}
 
A minimal explanation for our running example, $\{ v_1^1, v_1^2 \}$, is depicted in
Fig.~\ref{fig:neural_network_minimal_explanation} of the appendix.

A \emph{minimum explanation} 
is a subset
$E\subseteq F$ which is a minimal explanation of minimum size; i.e.,
 there is no other minimal
explanation $E^{\prime}\neq E$ such that $|E'|<|E|$.
Fig.~\ref{fig:neural_network_minimum_explanation} of the appendix shows that
$\{v_1^3 \}$
is a minimal explanation of minimal cardinality, and is hence 
 a minimum explanation in our example.

    \mysubsection{Contrastive Examples.}
\label{sec:contrastive-examples}
We define a contrastive example (also known as a
\emph{contrastive explanation (CXP)}) as a subset of features
$C\subseteq F$, whose alteration may
cause the classification of $v$ to change. More formally:
\begin{equation}
  \label{eq:contrastive_examples}
  \exists x\in \mathbb{F}.\quad \bigwedge_{i\in F\setminus C}(x_{i}=v_{i})\wedge(N(x)\neq c)
\end{equation}
A contrastive example for our running example appears in
Fig.~\ref{fig:neural_network_contrastive_example}.

    \begin{figure}[h]
    	\begin{center}
    		\includegraphics[width=0.30\textwidth]{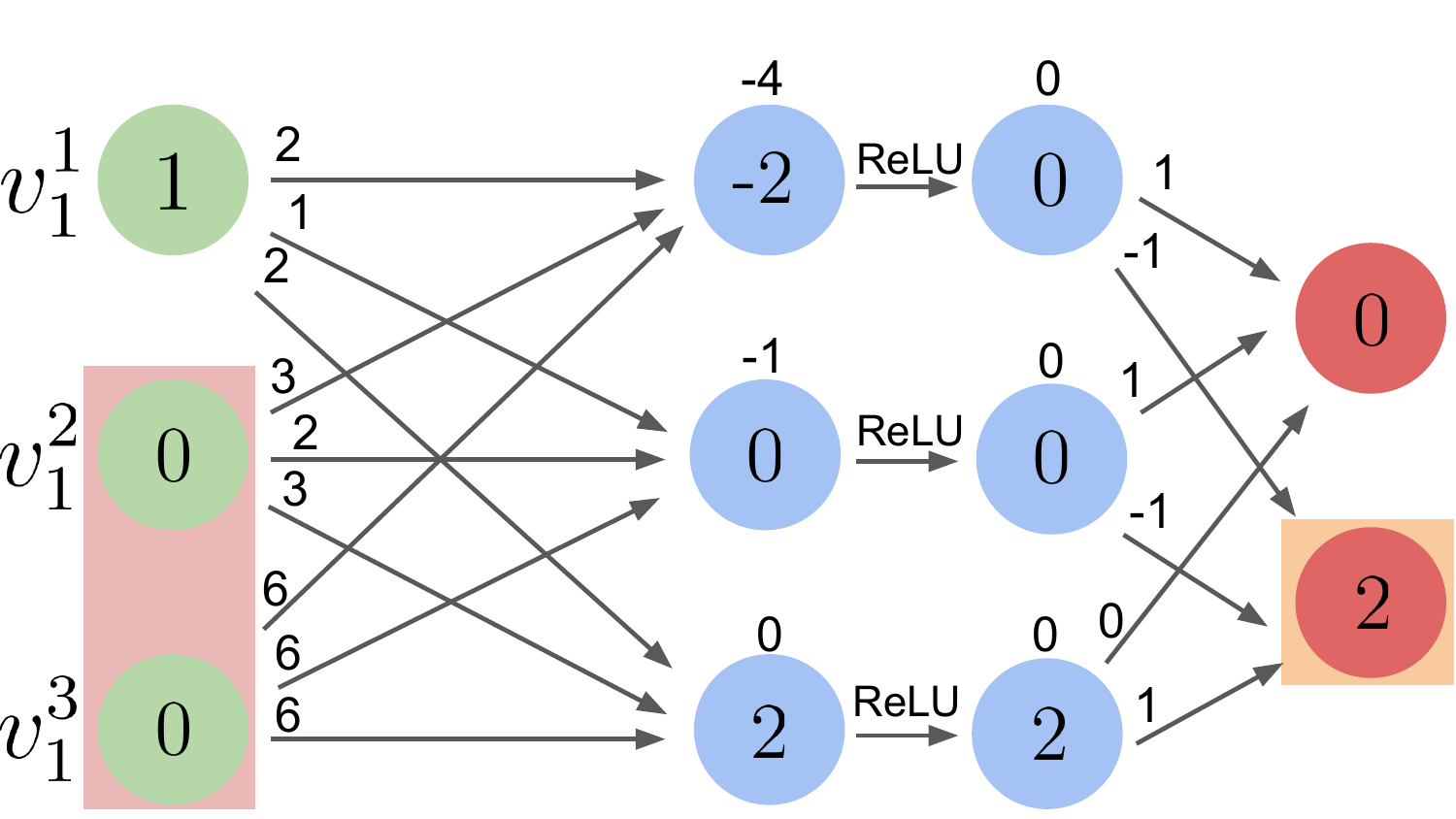}
    		\caption{$\{ v_1^2, v_1^3 \}$ is a contrastive example for $V_1=[1,1,1]^T$.}
    		\label{fig:neural_network_contrastive_example}
    	\end{center}
    \end{figure}

Checking whether $C$ is a contrastive example can be performed using
the query
$\langle P,N,Q\rangle = \langle(F\setminus C)=v,N,Q_{\neg c}\rangle$:
$C$ is contrastive iff the quest is \sat{}.
Any set
containing a contrastive example is contrastive, and so
we consider only  contrastive examples that are minimal, i.e.,
which do not contain any smaller contrastive examples.

Contrastive examples have an important property: every explanation
contains at least one element from every contrastive example~\cite{ignatiev2020contrastive, bassan2022towards}.  This can be
used for showing that a \emph{minimum hitting set} (MHS; see 
Sec.~\ref{sec:appendix:mhsDefinition} of the appendix) of all
contrastive examples is a minimum
explanation~\cite{ignatiev2016propositional, reiter1987theory}. In
addition, there exists a duality between contrastive examples and
explanations~\cite{ignatiev2020contrastive, ignatiev2015smallest}:
minimal hitting sets of all contrastive examples are minimal
explanations, and minimal hitting sets of all explanations are minimal
contrastive examples. This relation can be proved by reducing explanations and contrastive examples to minimal unsatisfiable sets and minimal correction sets, respectively, where this duality is known to hold~\cite{ignatiev2020contrastive}. Calculating an MHS is NP-hard, but can
be performed in practice using modern MaxSAT or
MILP solvers~\cite{li2021maxsat, ilog2018cplex}. The
duality is thus
useful since computing contrastive
examples and calculating their MHS is often more efficient than
directly computing minimum explanations~\cite{ignatiev2019abduction,
  ignatiev2020contrastive, bassan2022towards}.


%

  \section{K-Step Formal Explanations}
\label{formal-k-step-explanations-section}
A reactive system is a tuple $R=\langle S, A, I, T\rangle$, where $S$
is a set of states, $A$ is a set of actions, $I$ is a predicate over
the states of $S$ that indicates initial states, and
$T\subseteq S\times A\times S$ is a transition relation. In our
context, a reactive system has an associated neural network
$N:S\to A$. A $k$-step execution $\mathcal{E}$ of $R$ is a sequence
of $k$ states $(s_1,\ldots,s_k)$, such that $I(s_1)$ holds, and for
all $1\leq i\leq k-1$ it holds that $T(s_i,N(s_i),s_{i+1})$. We use
$\mathcal{E}_S=(s_1,\ldots,s_k)$ to denote the sequence of $k$ states
visited in $\mathcal{E}$, and $\mathcal{E}_A=(a_1,\ldots,a_k)$ to
denote the sequence of $k$ actions selected in these states. More
broadly, a reactive system can be considered as a deterministic, finite-state transducer Mealy automaton~\cite{shahbaz2009inferring}. 
Our goal
is to better understand $\mathcal{E}$, by finding abductive
explanations and contrastive examples that explain why $N$ selected
the actions in $\mathcal{E}_A$.

\mysubsection{K-Step Abductive Explanations.} Informally, we define an
explanation $E$ for a $k$-step execution $\mathcal{E}$ as a subset of
features of each of the visited states in $\mathcal{E}_S$, such that
fixing these features (while freeing all other features) is
sufficient for forcing the DNN to select the actions in
$\mathcal{E}_A$. More formally, $E=(E_1,\ldots,E_{k})$, such that
$\forall x_1, x_2,\ldots,x_k\in \mathbb{F}$,
\begin{equation}
\label{k-abductive-explanation-equation}
  \big
  (
  \bigwedge_{i=1}^{k-1}
    T(x_{i},N(x_{i}),x_{i+1})
    \wedge
    \bigwedge_{i=1}^{k}
    \bigwedge_{j\in E_i}(x^{j}_{i}=s^{j}_{i})
    \big)
\to
  \bigwedge_{i=1}^{k}
N(x_i)=a_{i}
  \end{equation}


\begin{figure}
  \begin{center}
    \includegraphics[width=0.3\textwidth]{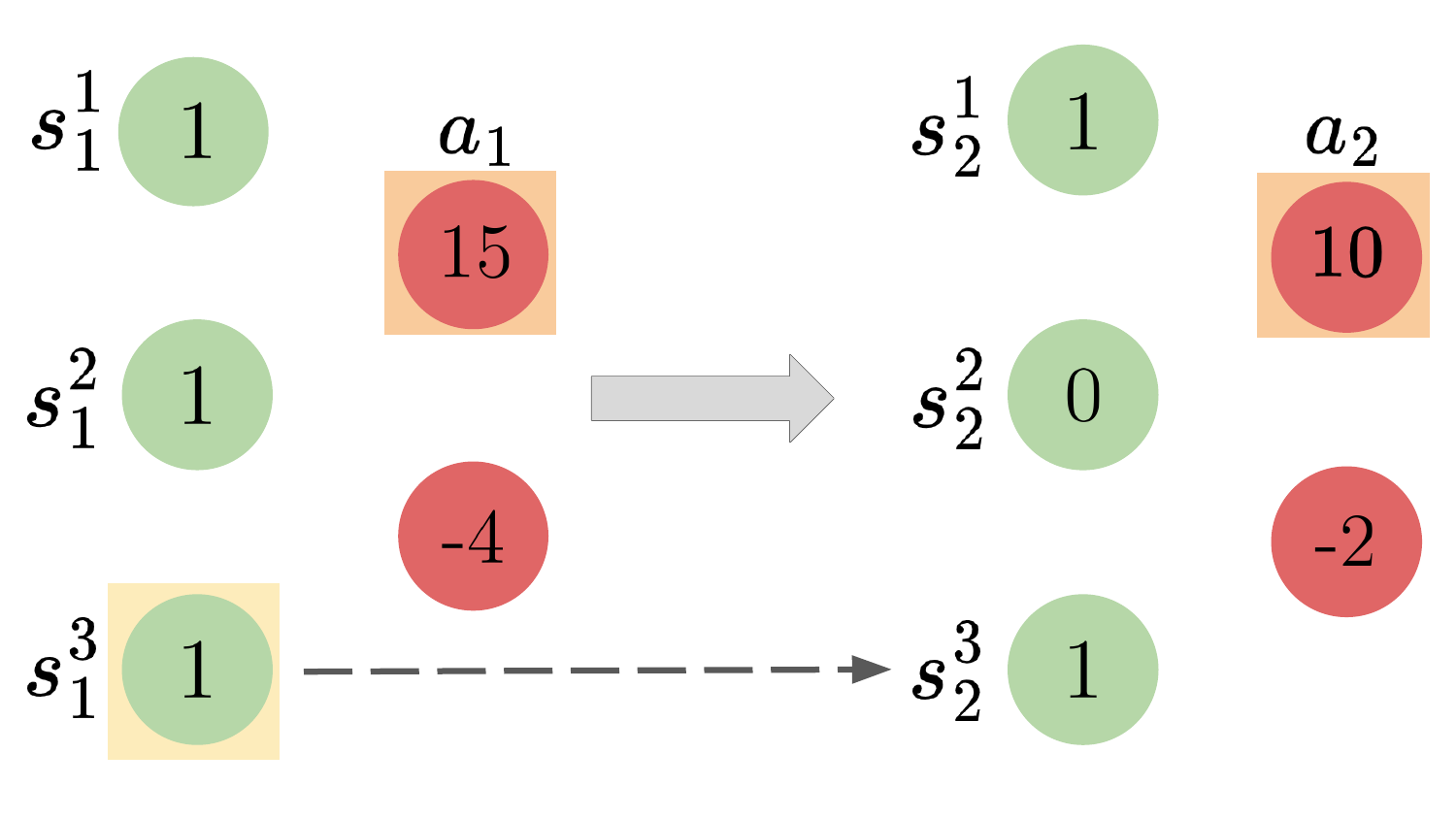}
    \caption{$(\{s^3\},\emptyset)$  is a (minimum) multi-step explanation for 
    $\mathcal{E}$.}
    \label{fig:k-explanation-example}
  \end{center}
\end{figure}
We continue with our running example.  Consider the transition
relation $T=\{(s,a,s')\ |\ s^3=s'^3\}$; i.e., we can transition from
state $s$ to state $s'$ provided that the third input neuron has the
same value in both states, regardless of the action selected in $s$.
Observe the $2$-step execution
$\mathcal{E}: s_1=(1,1,1)\stackrel{c_1}{\to}s_2=(1,0,1)\stackrel{c_1}{\to}$,
depicted in Fig.~\ref{fig:k-explanation-example} (for simplicity, we
omit the network's hidden neurons), and suppose we wish to explain
$\mathcal{E}_A=\{c_1,c_1\}$.  Because $\{s^3\}$ is an explanation for
the first step, and because fixing $s^3_1$  also fixes the value of $s_2^3$,
 it follows that fixing $s_1^3$ is sufficient to guarantee that action
 $c_1$ is selected twice --- i.e., $(\{s^3\},\emptyset)$
is a multi-step explanation for $\mathcal{E}$.

Given a candidate $k$-step explanation, we can check its validity by
encoding Eq.~\ref{k-abductive-explanation-equation} as a DNN
verification query. This is achieved by \emph{unrolling} the network
$N$ for $k$ subsequent steps; i.e., by encoding a network that is $k$
times larger than $N$, with input and output vectors that are $k$
times larger than the original. We must also encode the transition
relation $T$ as a set of constraints involving the input values, to mimic $k$ time-steps within a single feed-forward pass.  We
use $N_{[i]}$ to denote an unrolling of the neural network $N$ for $i$
steps, for $1\leq i\leq k$.

Using the unrolled network $N_{[k]}$, we encode the negation of
Eq.~\ref{k-abductive-explanation-equation} as the query
$\langle P,N,Q\rangle = \langle E=\mathcal{E}_S,N_{[k]},Q_{\neg \mathcal{E}_A}\rangle$, where
$E=\mathcal{E}_S$ means that we restrict the features in each subset $E_i\in E$ to
their corresponding values in $s_i$; and $Q_{\neg \mathcal{E}_A}$
indicates that in some step $i$, an action that is not $a_i$ was
selected by the DNN.
An \unsat{} result for this query indicates that $E$ is an explanation
for $\mathcal{E}$, because fixing $E$'s features to their values
forces the given sequence of actions to occur.

We can naturally define a \emph{minimal} $k$-step explanation as a 
$k$-step explanation that ceases to be a $k$-step explanation when we remove
any of its features. A \emph{minimum} $k$-step explanation is a minimal 
$k$-step explanation of the lowest possible cardinality; i.e., there does not 
exist a $k$-step explanation $E'=(E'_1, 
E'_2,\ldots,E'_k)$ such that $\sum_{i=1}^k|E'_i| < 
\sum_{i=1}^k|E_i|$.

\mysubsection{K-Step Contrastive Examples.} A contrastive example $C$ for
an execution $\mathcal{E}$ is a subset of features whose alteration can
cause the selection of an action not in
$\mathcal{E}_A$. A $k$-step contrastive example is
depicted in Fig.~\ref{fig:k-contrastive-example}: altering the
features $s_1^3$ and $s_2^3$ may cause
action $c_2$ to be chosen instead of $c_1$ in the second step.
Formally, $C$ is an ordered set of (possibly
empty) subsets $C=(C_1,C_2,\ldots,C_k)$, such that $C_i\subseteq F$, and
for which $\exists x_1,x_2,\ldots,x_k \in \mathbb{F}$ such that 
\begin{align}
  \begin{split}
  \label{eq:contrastive_examples}
  &\big(\bigwedge_{i=1}^{k-1} T(x_{i},N(x_{i}),x_{i+1}) \big)
  \wedge\\
  &\big(\bigwedge_{i=1}^{k}\bigwedge_{j\in F\setminus
    C_i}(x^{j}_{i}=s^{j}_{i}) \big)
  \wedge
  \big(\bigvee_{i=1}^k N(x_i)\neq a_i \big)
\end{split}
\end{align}
Similarly
to multi-step explanations, 
$C$ is a multi-step contrastive example iff the verification
query:
$\langle P,N,Q\rangle = \langle(F\setminus C_1, F\setminus
C_2,\ldots,F\setminus C_k)=\mathcal{E}_S,N_{[k]},Q_{\neg
  \mathcal{E}_A}\rangle$
is \sat{}.

  \begin{figure}
  \begin{center}
    \includegraphics[width=0.3\textwidth]{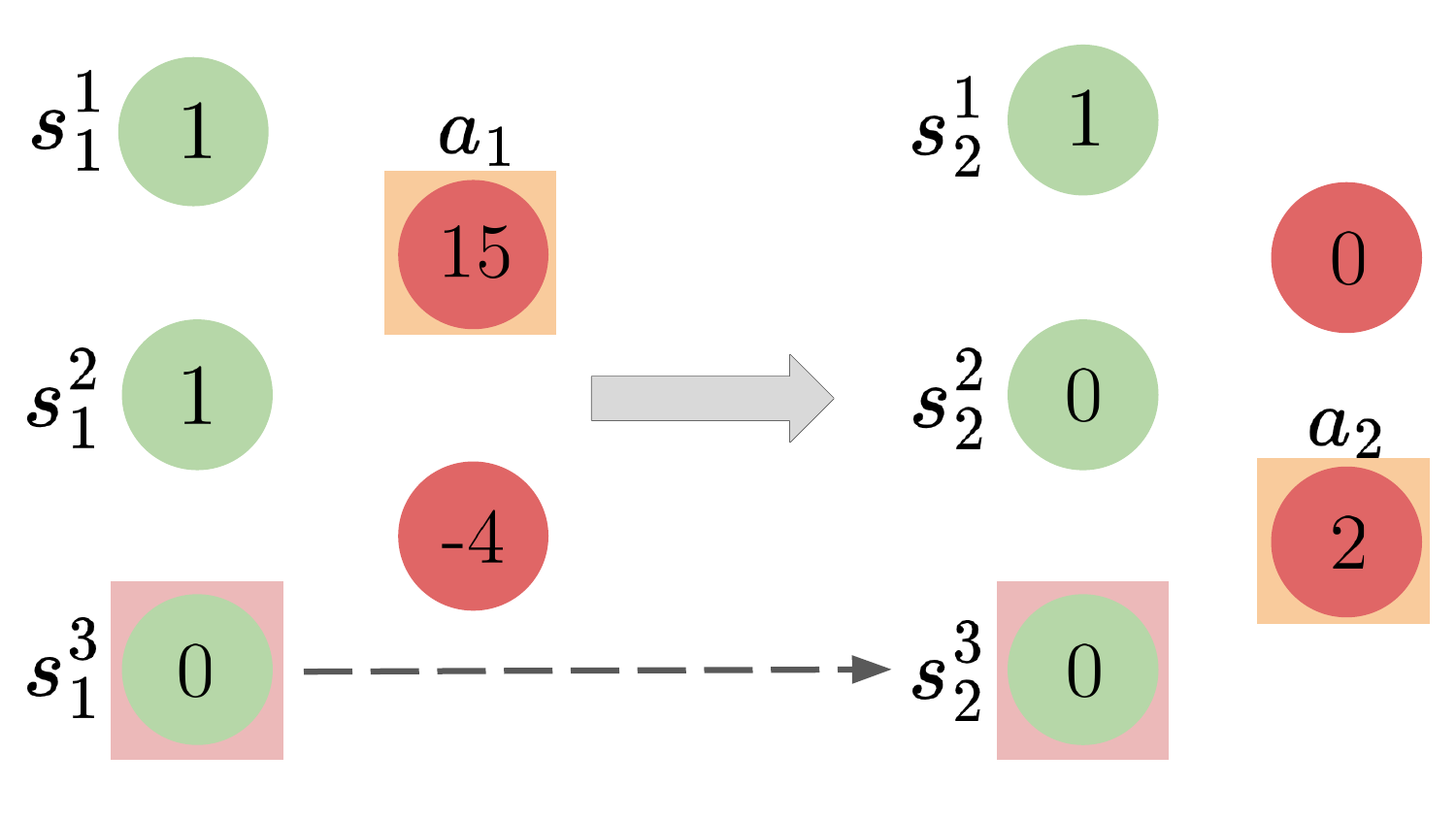}
    \caption{$(\{s^3\}, \{s^3\})$  is a multi-step contrastive example for 
$\mathcal{E}$.}
    \label{fig:k-contrastive-example}
  \end{center}
\end{figure}

\section{Computing Formal K-Step Explanations}
\label{sec:formal_k_explanations_computation}

We now propose four different methods for computing formal $k$-step
explanations, focusing on \emph{minimal} and \emph{minimum}
explanations.  All four methods use an underlying DNN verifier to
check candidate explanations, but differ in how they
enumerate different explanation candidates until ultimately
converging to an answer. We begin with the more straightforward
methods. 

\mysubsection{Method 1: A Single, K-Sized Step.}
\label{method-1-section}
The first method is to encode the negation of
Eq.~\ref{k-abductive-explanation-equation} by unrolling all $k$ steps
of the network, as described in
Sec.~\ref{formal-k-step-explanations-section}.  This transforms the
problem into explaining a non-reactive, single-step system (e.g., a
``one-shot'' classifier). We can then use any existing abductive explanation algorithm for
explaining the unrolled DNN (e.g.,~\cite{ignatiev2019abduction,
  bassan2022towards, ignatiev2020contrastive}).

This method is likely to produce small explanation sets but is
extremely inefficient. Encoding $N_{[k]}$ results in an input space
roughly $k$ times the size of any single-step encoding. Such an
unrolling for our running example is depicted in
Fig.~\ref{fig:method1Scheme}.  Due to the NP-completeness of DNN
verification, this may cause an exponential growth in the verification
time of each query.  Since finding minimal explanations
requires a linear number of queries (and for minimum
explanations --- a worst-case exponential number), this may cause a substantial
increase in runtime.
    \begin{figure}[ht]
	\begin{center}
		\includegraphics[width=0.3\textwidth]{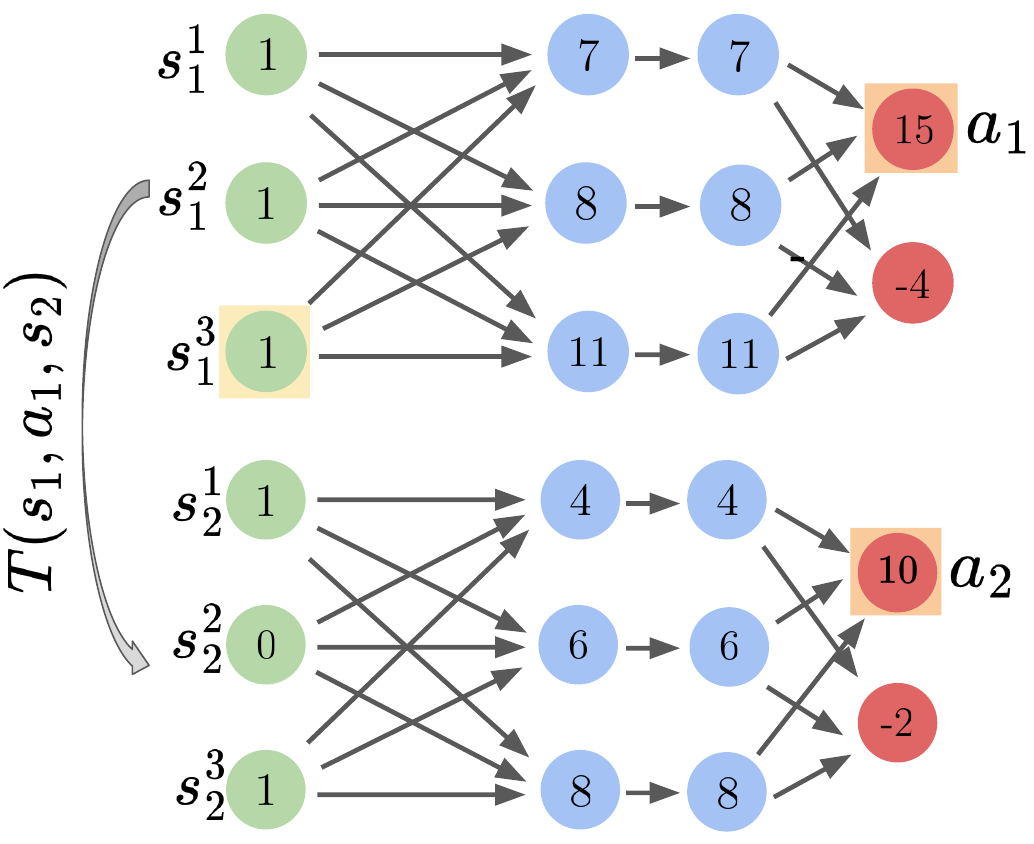}
	\end{center}
	\caption{Finding explanations using a 2-step unrolling.}
	\label{fig:method1Scheme}
      \end{figure}
      
\mysubsection{Method 2: Combining Independent, Single-Step Explanations.}
\label{section-method-2}
Here, we dismantle any $k$-step execution into $k$
individual steps. Then, we \emph{independently}
compute an explanation for each step, using any existing algorithm, and
without taking the transition relation into account. Finally, we
concatenate these explanations to form a multi-step explanation.  Fixing
the features of the explanation in each step ensures that the ensuing
action remains the same, guaranteeing the soundness of the
combined explanation.



\begin{figure}
	\centering
	\begin{center}
		\includegraphics[width=0.3\textwidth]{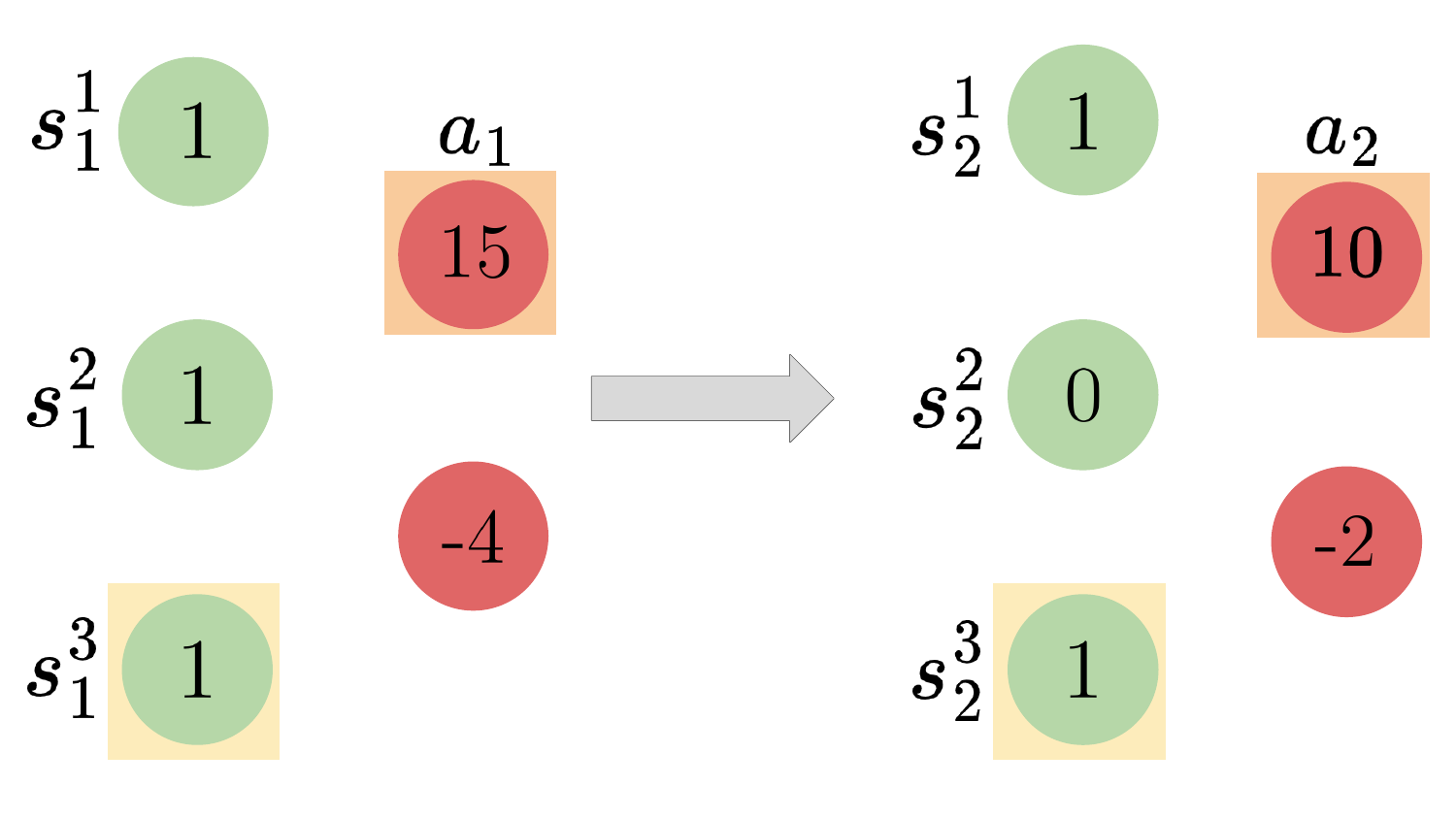}
	\end{center}
	\caption{Explaining each step individually.}
	\label{fig:method2Scheme}
\end{figure}
The downside of this method is that the resulting $E$ need not be
minimal or minimum, even if its constituent $E_i$ explanations are
minimal or minimum themselves; see
 Fig.~\ref{fig:method2Scheme}. In this instance, finding a
minimum explanation for each step results in the 2-step explanation
$(\{s^3\},\{s^3\})$, which is \emph{not minimal} --- even though its
components are minimum explanations for their respective steps. The
reason for this phenomenon is that this method ignores the transition
constraints and information flow across time-steps. This can
result in larger and less meaningful explanations, as we
later show in Sec.~\ref{sec:evaluation}.

\mysubsection{Method 3: Incremental Explanation Enumeration.} We
now suggest a scheme that takes into consideration the transition
constraints between steps (unlike Method 2), but which encodes the
verification queries for validating explanations in a more efficient
manner than Method 1.
The scheme relies on the following lemma:

\begin{lemma}
\label{method3_second_lemma}
Let $E=(E_1, E_2, \ldots, E_k)$ be a $k$-step explanation for
execution $\mathcal{E}$, and let $1\leq i \leq k$ such that
$\forall j>i$ it holds that $E_j=F$. Let $E'$ be the set obtained by
removing a set of features $F'\subseteq E_i$ from $E_i$, i.e.,
$E'=(E_1, \ldots, E_{i-1},E_i\setminus F', E_{i+1}, \ldots, E_k)$.
In this case,  fixing the  features in $E'$ prevents any changes in the first 
$i-1$ actions $(a_1,\ldots,a_{i-1})$; and if any of the last $k-i+1$ actions 
$(a_i,\ldots,a_k)$ change, then $a_i$ must also change.
\end{lemma}
A proof appears in Sec.~\ref{sec:appendix:additionalProofs} of the appendix.
The lemma states that ``breaking'' an explanation $E$ of $\mathcal{E}$
at some step $i$ (by removing features from the $i$'th step), given
that the features in steps $i+1,\ldots,k$ are fixed, causes $a_i$ to change before any other action. In this scenario, we can determine whether
$E$ explains $\mathcal{E}$ using a simplified verification query: we
can check whether $(E_1,\ldots,E_i)$ explains the first $i$ steps of
$\mathcal{E}$, regardless of steps $i+1,...,k$. If so, then $a_i$
cannot change; and from Lemma~\ref{method3_second_lemma}, no action in
$\mathcal{E}_A$ can change, and $(E_1,\ldots,E_k)$ is an explanation
for $\mathcal{E}$. Otherwise, $E$ allows an action in $\mathcal{E}_A$
to change, and it does not explain $\mathcal{E}$.  We can leverage
this property to efficiently enumerate candidates as part of a
search for a minimal/minimum explanation for $\mathcal{E}$, as
explained next.


\mysubsection{Finding Minimal Explanations with Method 3.}
A common approach for finding minimal explanations for a ``one-shot''
classification instance is via a greedy algorithm, which dispatches a
linear number of queries to the underlying
verifier~\cite{ignatiev2019abduction}. Such an algorithm can start
with the explanation set to be the entire feature space, and then
iteratively attempt to remove features. If removing a
feature allows misclassification, the algorithm keeps it as part of
the explanation; otherwise, it removes the feature and continues. A pseudo-code
for this approach appears in Alg.~\ref{alg:greedy-minimal}.

\begin{algorithm}
	\algnewcommand\algorithmicforeach{\textbf{for each}}
	\algdef{S}[FOR]{ForEach}[1]{\algorithmicforeach\ #1\ \algorithmicdo}
	\caption{\texttt{Greedy-Minimal-Explanation}\xspace}\label{alg:greedy-minimal}
	\textbf{Input} $N$ (DNN), $F$ ($N$'s features), $v$ (values), $c$ 
	(predicted class)
	\begin{algorithmic}[1]
		\State \explanation$\gets\ F$
		\ForEach {$f \in F$}
		\If{\verifyexplanation is \unsat{}}
		\State{\explanation $\gets$ \explanation$\setminus\{f\}$}
		\EndIf
		\EndFor
		\State \Return Explanation
	\end{algorithmic}
\end{algorithm}

We suggest performing a similar
process for explaining $\mathcal{E}$.  We start by fixing
all features in all states of $\mathcal{E}$ to their values; i.e., we
start with $E=(E_1,\ldots,E_k)$ where $E_i=F$ for all $i$, and then
perform the following steps:

First, we iteratively remove individual features from $E_1$, each time
  checking whether the modified $E$ remains an explanation for
  $\mathcal{E}$. Since all features in steps $2,\ldots,k$ are fixed,
  it follows from Lemma~\ref{method3_second_lemma} that checking
  whether the modified $E$ explains $\mathcal{E}$ is equivalent to
  checking whether the modified $E_1$ explains the selection of
  $a_1$. Thus, we perform a process that is identical to the one in
  the greedy Alg.~\ref{alg:greedy-minimal} for finding a minimal
  explanation for a ``one-shot'' classification DNN. At the end of
  this phase, we are left with $E=(E_1,\ldots,E_k)$ where $E_i=F$ for
  all $i>1$ and $E_1$ was reduced by removing features from it. We
  keep all current features in $E$ fixed for the following steps.

Second, we begin to
  iteratively remove features from $E_2$, each time checking whether
  the modified $E$ still explains $\mathcal{E}$. Since the features in
  steps $3,\ldots,k$ are entirely fixed, it suffices (from Lemma~\ref{method3_second_lemma}) to check whether the
  modified $(E_1,E_2)$ explains the selection of the first two actions 
  $(a_1,a_2)$ of $\mathcal{E}_A$. This is performed by
  checking whether
    \begin{equation}
    \label{eq:k_explanation-small-example}
    \begin{aligned}
    &(\forall x_1, x_2\in \mathbb{F}.\quad T(x_{1},a_1, x_{2}) \wedge
    \bigwedge_{j\in E_1}(x^{j}_{1}=s^{j}_{1})\wedge \\
    & \bigwedge_{j\in E_2}(x^{j}_{2}=s^{j}_{2}))\to N(x_2)=a_2
      \end{aligned}
    \end{equation}
    We do not need to require that $N(x_1)=a_1$ (as in Method 1) ---
    this is guaranteed by Lemma~\ref{method3_second_lemma}. This is
    significant, because it exempts us from encoding the neural
    network twice as part of the verification query. We denote the
    negation of Eq.~\ref{eq:k_explanation-small-example} for
    validating $(E_1,E_2)$ as:
    $\langle P,N,Q\rangle = \langle
    (E_1,E_2)=\mathcal{E}_{S_{[2]}},N,Q_{\neg a_2}\rangle$.
    
  Third, we continue this iterative process for all $k$ steps of
    $\mathcal{E}$, and find the minimal explanation for each step
    separately. In step $i$, for each query we encode $i$
    transitions and check whether the modified $E$ still explains
    the first $i$ steps of $\mathcal{E}$ (by encoding $\langle
    (E_1,\ldots,E_i)=\mathcal{E}_{S_{[i]}},N,Q_{\neg a_i}\rangle$),
    which \emph{does not} require encoding the DNN
    $i$ times. The
    correctness of each step follows directly from
    Lemma~\ref{method3_second_lemma}.
    
  The pseudo-code for this process appears in
  Alg.~\ref{alg:minimal_multi_step_explanation_alg}. 
  The minimality
  of the resulting explanation holds because removing any
  feature from this explanation would allow the action in that step to
  change (since minimality is maintained in each step of the
  algorithm).  An example of the first two iterations of this process
  on our running example appears in Fig.~\ref{fig:method3Scheme}:
  in the first iteration, we attempt to remove features
  from the first step, until converging to an explanation $E_1$. In
  the second iteration, while the features in $E_1$ remain fixed to
  their values, we encode the constraints of the transition relation
  $T(s_1,a_1,s_2)$ between the first two steps, and dispatch queries
  to verify candidate explanations for the second step --- until
  converging to a minimal explanation $(E_1,E_2)$. In this case, 
  $E_2=\emptyset$, and $(\{s^3\},\emptyset)$ is a valid explanation for the 
  2-step execution, since fixing the value of 
  $s_1^3$ determines the value
  of $s_2^3$ --- which forces the selection of $a_2$ in the second
  step.

\begin{figure}
\centering
\subfloat[First iteration]{\includegraphics[width=6cm]{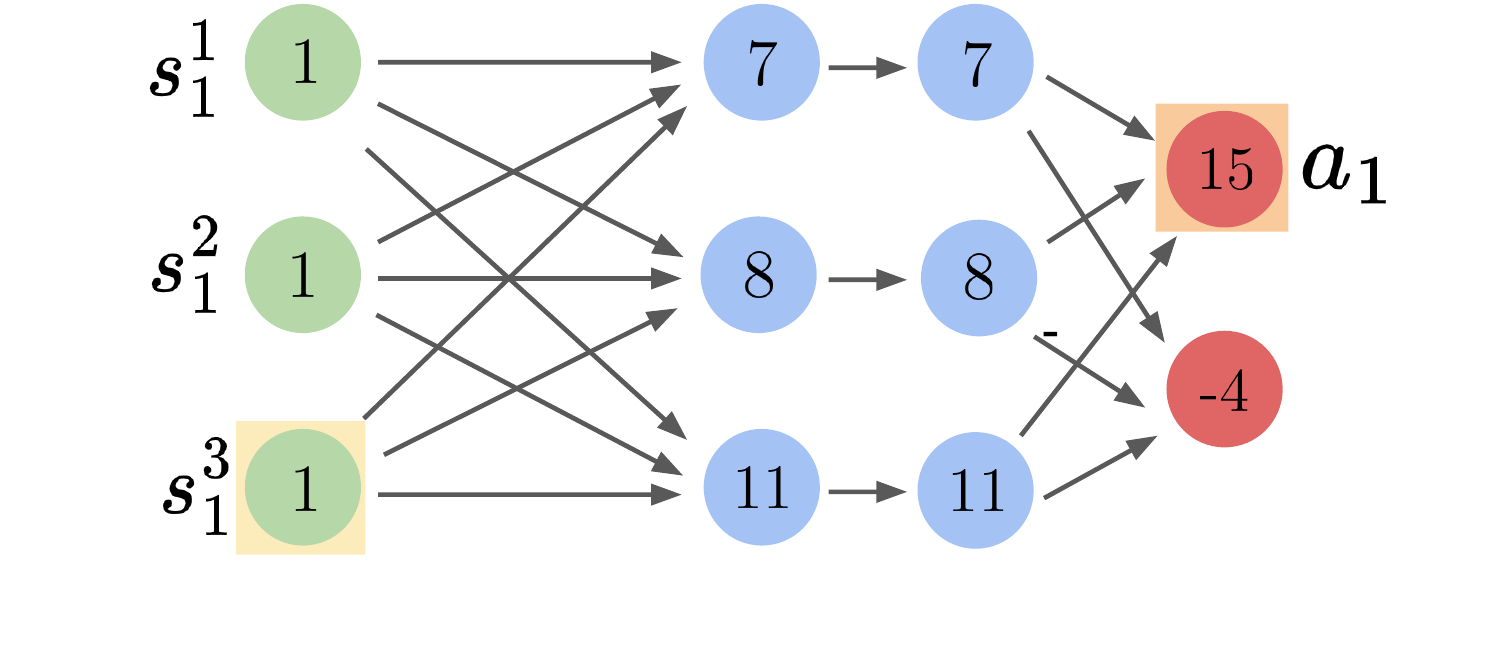}\label{fig:Ex_Im}}\quad
\subfloat[Second iteration]{\includegraphics[width=6cm]{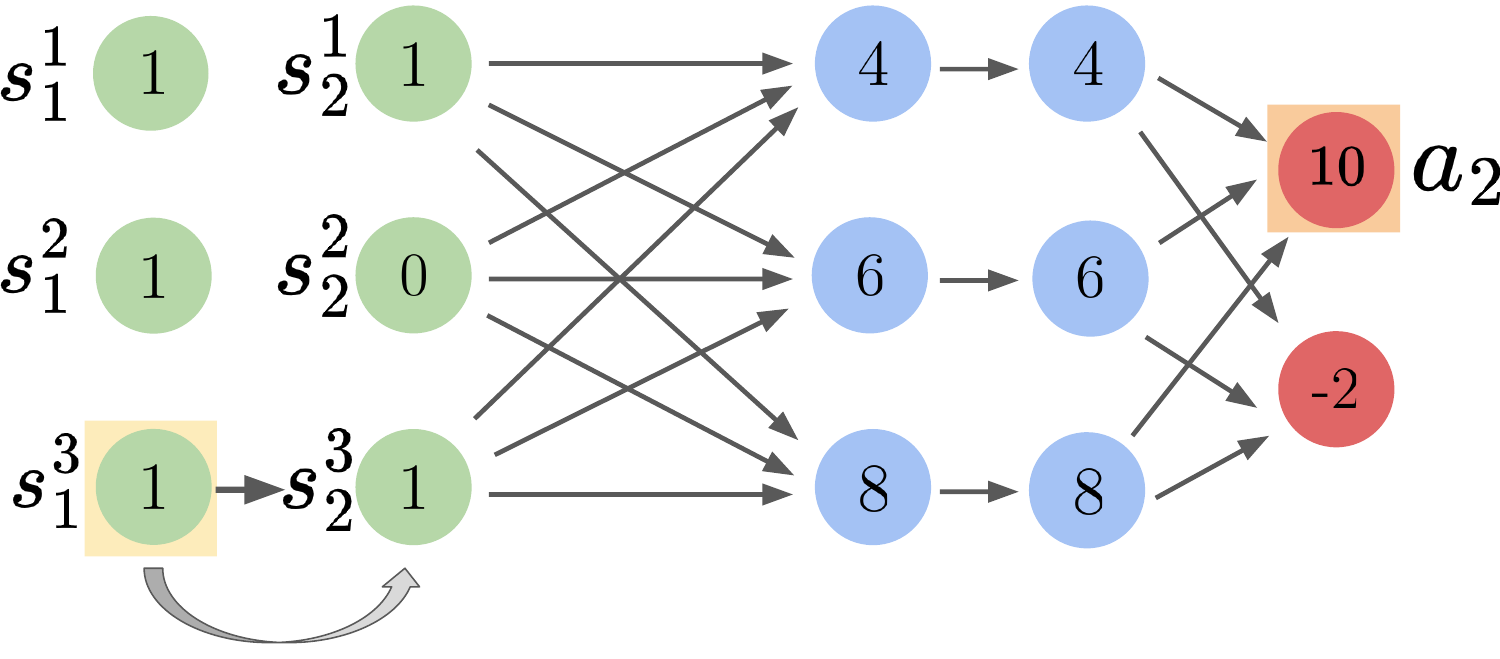}\label{fig:Ex_Im2}}
\caption{Running Method 3 for finding minimal explanations, for two
  iterations.}
	\label{fig:method3Scheme}
\end{figure}


We emphasize that incrementally enumerating candidate explanations for
a $k$-step execution in this way is preferable to simply finding a
minimal explanation by encoding verification queries that encompass
all $k$-steps, \`a la Method 1:
\begin{inparaenum}[(i)]
\item in each iteration, we dispatch a verification query involving
  only a single 
invocation of the DNN, thus circumventing the linear growth in the network's
size --- which causes an exponential worst-case
increase in verification times; and
\item in each iteration, we do not need to encode the entire set of
  $k$ disjuncts (from the negation of
  Eq.~\ref{k-abductive-explanation-equation}), since we only need to
  validate $a_i$ on the $i$'th iteration, and not all actions of
  $\mathcal{E}_A$.
\end{inparaenum}


\begin{algorithm}
		\algnewcommand\algorithmicforeach{\textbf{for each}}
		\algdef{S}[FOR]{ForEach}[1]{\algorithmicforeach\ #1\ \algorithmicdo}
    	\textbf{Input} $N$ (DNN), $F$ ($N$'s features), $\mathcal{E}$ 
    	(execution of length $k$ to explain)
		\caption{\texttt{Incremental-Minimal-Explanation\\-Enumeration}}
		\label{alg:minimal_multi_step_explanation_alg}
		\begin{algorithmic}[1]
			\State{\explanation $\gets (E_1,\ldots,E_k)$ where $E_i=F$ for all $1\leq i \leq k$}
   		\ForEach {$i \in \{1,...,k\}$ \text{and} $f \in E_i$}\label{lst:line:startregularupper}
			\If{\verifymultistepexplanation is \unsat{}}\label{lst:line:simplifiedlineminimal}
			\State{$E_i \gets E_i\setminus f$}
			\EndIf
			\EndFor\label{lst:line:endregularupper}
            \State \Return Explanation
		\end{algorithmic}
\end{algorithm}



\mysubsection{Finding Minimum Explanations with Method 3.}
We can also use our proposed enumeration to efficiently find 
\emph{minimum} explanations, using a recursive approach. In each
step $i=1,\ldots,k$, we iterate over all the possible explanations, each
time considering a candidate explanation and recursively invoking the
procedure for step $i+1$. In this way, we iterate over all the 
possible multi-step explanation candidates and can return the
smallest one that we find. This process is described in
Alg.~\ref{alg:minimum_multi_step_explanation_alg}.



Finding a minimum explanation in this manner is superior to using
Method 1, for the same reasons noted before. In addition, the
exponential blowup here is in the number of explanations in each step,
and not in the entire number of features in each step --- which is
substantially smaller in many cases. Nevertheless, as the method
advances through steps, it is expected to be significantly harder to
iterate over all the candidate explanations. We discuss more efficient
ways for finding global minimum explanations in Method 4.

\begin{algorithm}[ht]
	\algnewcommand\algorithmicforeach{\textbf{for each}}
	\algdef{S}[FOR]{ForEach}[1]{\algorithmicforeach\ #1\ \algorithmicdo}
	\caption{\texttt{Incremental-Minimum-Explanation-\\Enumeration}}
	\label{alg:minimum_multi_step_explanation_alg}
	\textbf{Input} $N$ (DNN), $F$ ($N$'s features), $\mathcal{E}$ (execution to 
	explain) \Comment{\textcolor{blue}{Global Variables}}
	\begin{algorithmic}[1]
		\State{AllExplanations $\gets$ 
			\Call{\allExplanationRecursiveSearchWithLineBreak}{$\emptyset$, 1}}
		\State \Return $E\in$ AllExplanations such that $E$ is with minimum 
		cardinality
		
	\end{algorithmic}
\end{algorithm}

\begin{algorithm}[ht]
	\algnewcommand\algorithmicforeach{\textbf{for each}}
	\algdef{S}[FOR]{ForEach}[1]{\algorithmicforeach\ #1\ \algorithmicdo}
	\caption{\texttt{All-Explanation-Recursive-Search}}
	\label{alg:upper-thread}
	\textbf{Input} E (explanation), i (step number)
	\begin{algorithmic}[1]
		
		\If{i = k}
		\State \Return E
		\EndIf
		\State{AllExplanations $\gets \emptyset$}
		\ForEach {subset $F'$ of $F$}
		\If{\verifymultistepminimumexplanation is
			\unsat{}}\label{lst:line:simplifiedlineminimum}
		\State{Explanations $\gets$ 
			\allExplanationRecursiveSearchWithLineBreak(E $\cdot\ 
			(F')$, i+1)}
		\State{AllExplanations $\gets$ AllExplanations $\cup$ Explanations}
		\EndIf
		\EndFor
		\State \Return AllExplanations
		
	\end{algorithmic}
\end{algorithm}

\mysubsection{Method 4: Multi-Step Contrastive Example Enumeration.}
As mentioned earlier, a common approach for finding minimum
explanations is to find all contrastive examples, and then calculate
their minimum hitting set (MHS). Because DNNs tend to be sensitive to
small input perturbations~\cite{su2019one}, small contrastive examples
are often easy to find, and this can expedite the process
significantly~\cite{bassan2022towards}. When performing this procedure
on a multi-step execution $\mathcal{E}$, we show that it is possible
to enumerate contrastive example candidates in a more efficient manner
than simply using the encoding from Method 1.

\begin{lemma}
\label{method4_second_lemma}
Let $\mathcal{E}$ be a $k$-step execution, and let
$C=(C_1,\ldots,C_k) $ be a minimal contrastive example for
$\mathcal{E}$; i.e., altering the features in $C$ can cause at
least one action in $\mathcal{E}_A$ to change. Let
$1\leq i \leq k$ denote the index of the first action
$a_i$ that can be changed by features in $C$. It holds that: $C_i \neq \emptyset$; $C_j=\emptyset$ for
all $j>i$; and if there exists some $l<i$ such that
$C_l\neq \emptyset$, then all sets $\{C_{l}, C_{l+1},\ldots,C_i\}$
are not empty.
\end{lemma}

The lemma gives rise to the following scheme.  We  examine
some contrastive example $C'$ of \emph{a set of subsequent steps of
  $\mathcal{E}$}. For simplicity, we discuss the case where $C'=
(C'_i)$
involves only a single step $i$; but the technique generalizes to
subsets of steps, as well. Such a $C'_i$ can be found using a
``one-shot'' verification query on step $i$, without encoding the transition
relation or unrolling the network. Our goal is to use $C'$ to find
many contrastive examples for $\mathcal{E}$, and use them in computing
the MHS.  We observe that there are three possible cases:
\begin{enumerate}
\item $C=(\emptyset, \ldots, \emptyset, C'_i, \emptyset, \ldots, \emptyset)$ already constitutes a 
  contrastive example for $\mathcal{E}$. In this case, we say that
  $C'=(C'_i)$ is an \emph{independent contrastive example}.
\item The features in $C'_i$ can cause a skew from $\mathcal{E}$ only
  when features from preceding steps $l,\ldots, i-1$ (for some $l<i$)
  are also altered.  In this case, we say that $C'$ is a
  \emph{dependent contrastive example}, and that it depends on steps
  $l,\ldots,i-1$; and together, the features from all these steps form
  the contrastive example
  $C=(\emptyset,\ldots,\emptyset, C_l, \ldots, C_{i-1},C'_i,\emptyset,\ldots,\emptyset)$ for
  $\mathcal{E}$.

\item $C'$ is a \emph{spurious contrastive example}: the first $i-1$
  steps in $\mathcal{E}$, and the constraints that the transition
  relation imposes, prevent the features freed by $C'_i$ from causing
  any action besides $a_i$ to be selected in step $i$.

 \end{enumerate}

 Fig.~\ref{fig:step_cxp_option_examples} illustrates the three
 cases. The first case is identical to the one from
 Fig.~\ref{fig:k-contrastive-example}, where $(\{s^3\})$ is a dependent
 contrastive example of the second step, which depends on the
 previous step and is part of a larger contrastive example:
 $(\{s^3\},\{s^3\})$. In the second case, assume that $T$ requires
 that $s_3^1+s_3^2\neq1$ for any feasible transition.  Thus, the
 assignment for $s_2^3$ which may cause the second action in the
 sequence to change is not reachable from the previous step, and hence
 $(\{s^3\})$ is a spurious contrastive example of the second step. In the third case, 
 assume that $T$ allows all transitions, and hence $(\{s^3\})$ is an
 independent contrastive example for the second step; and so
 $(\emptyset,\{s^3\})$ is a contrastive example of the entire
 execution.

\begin{figure*}[ht] 
	\centering
\centering
\subfloat[Dependent]{\includegraphics[width=4.4cm]{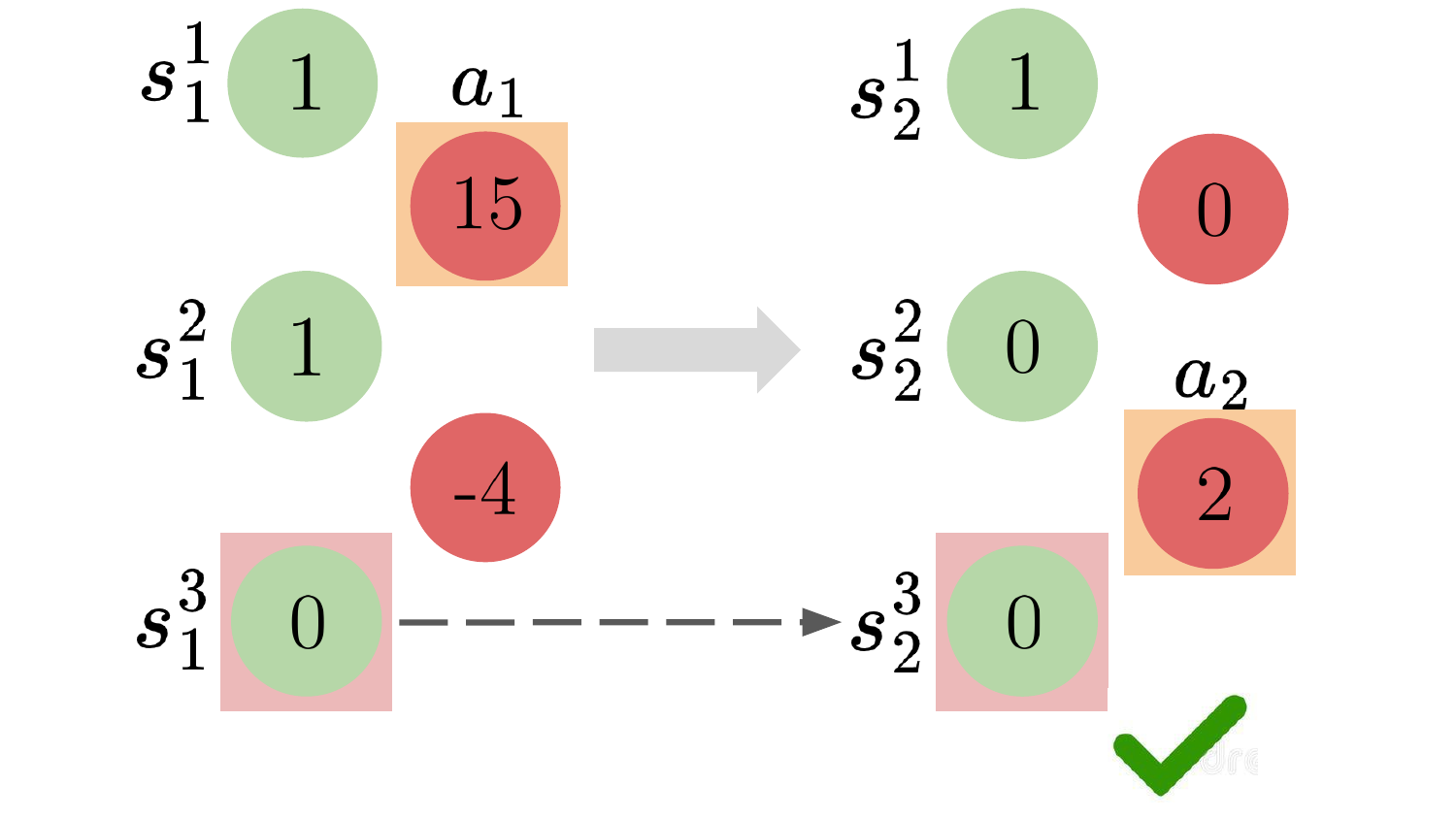}\label{fig:dependent_cxp}}
\subfloat[Spurious]{\includegraphics[width=4.4cm]{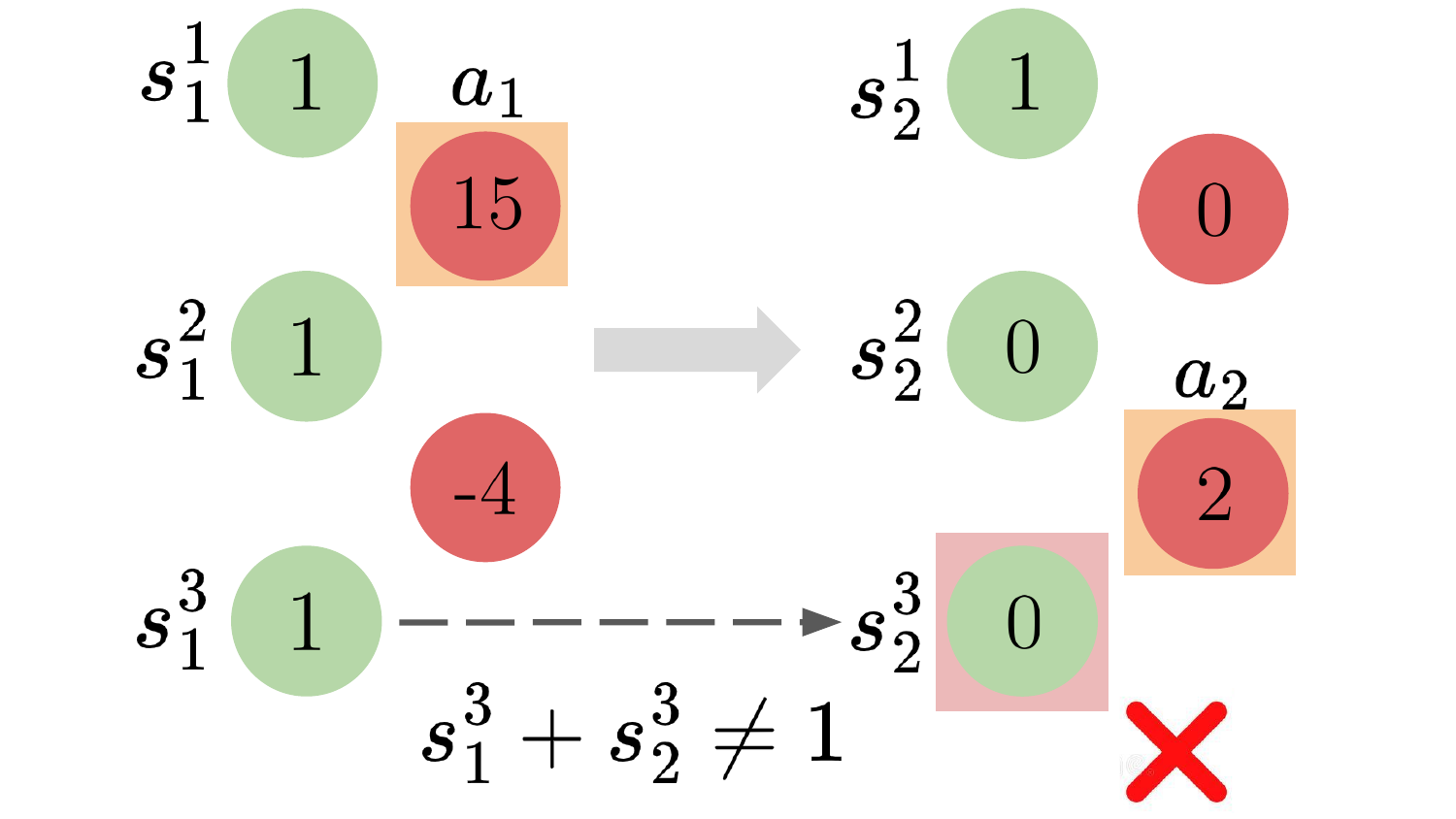}\label{fig:spurius_cxp}}
\subfloat[Independent]{\includegraphics[width=4.4cm]{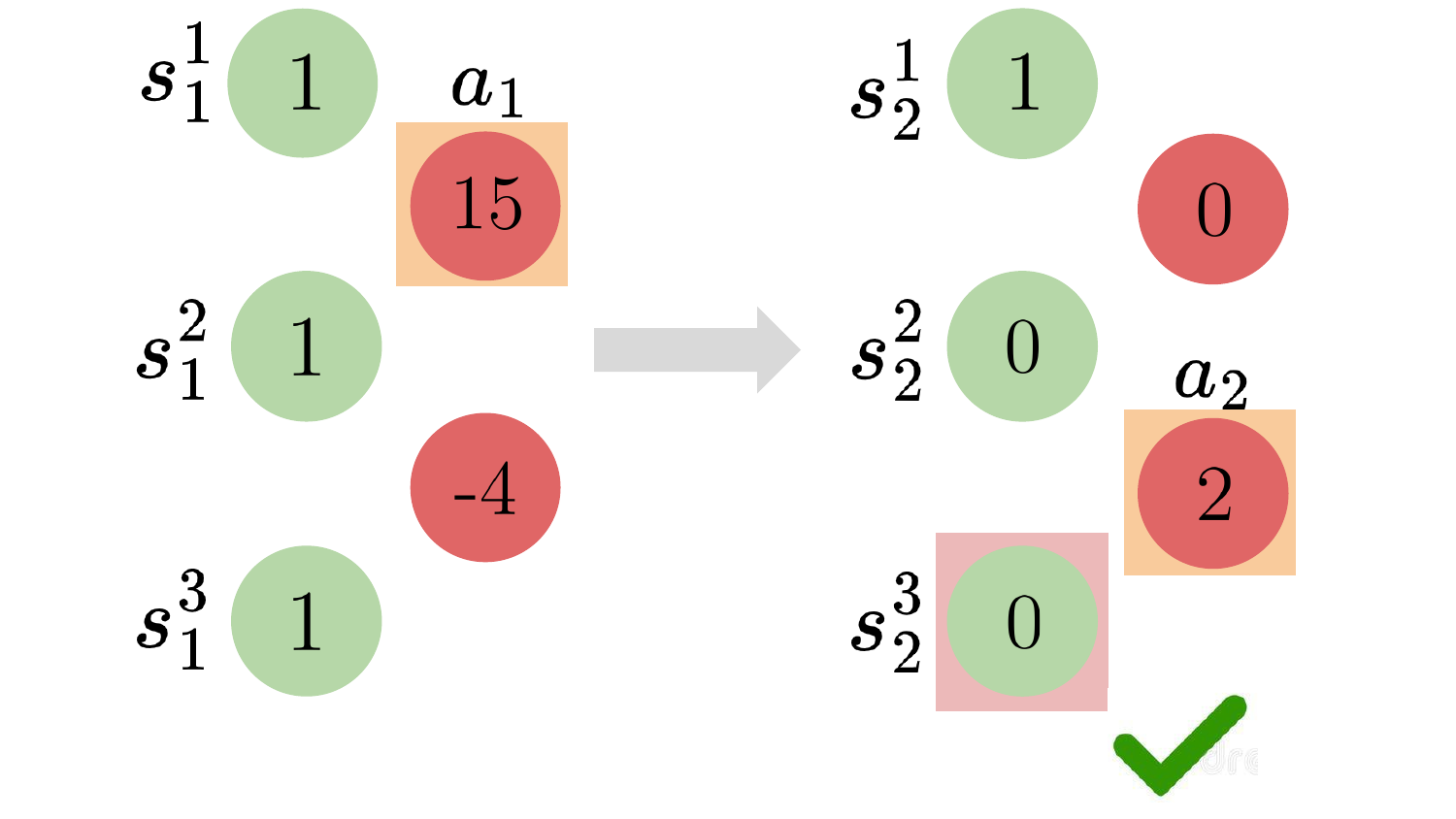}\label{fig:independent_cxp}}
\caption{$(\{s^3\})$ as a dependent, spurious and independent contrastive 
example.}
	\label{fig:step_cxp_option_examples}
\end{figure*}

It follows from Lemma~\ref{method4_second_lemma} that one of
these three cases must always apply. We next explain how
verification can be used to classify each contrastive example of a
subset of steps into one of these three categories. If $C'$ is
independent, it can be used as-is in computing the MHS; and if it is
spurious, it should be ignored. In the case where $C'$ is dependent,
our goal is to find all multi-step contrastive examples that contain
it, for the purpose of computing the MHS. We next describe a recursive
algorithm, termed \emph{reverse incremental enumeration} (RIE), that
achieves this.

\mysubsection{Reverse Incremental Enumeration.}
%
%
%
Given a contrastive example $C'$ containing features from a set of
subsequent steps of $\mathcal{E}$, we propose to classify it into one
of the three categories by iteratively dispatching queries
that check the reachability of $C'$ from the previous steps of the
sequence. We execute this procedure by recursively enumerating
contrastive examples in previous steps. For simplicity, we assume again 
that  $C'=(C'_i)$ is a single-step contrastive example
of step $i$.
\begin{enumerate}
\item For checking whether $C'$ is an independent contrastive
  example of $\mathcal{E}$, we set
  $C_{i-1}=\emptyset$ and $C_i=C'_i$, and check whether
  $C=(C_{i-1},C_i)$ is a contrastive example for steps
  $i-1$ and $i$. This is achieved by dispatching the following query:
    $\exists x_{i-1}, x_i \in \mathbb{F}$ such that:
    \begin{equation}
\begin{aligned}
    \label{eq:contrastive_examples_for_second_step_case_1}
    &T(x_{i-1},N(x_{i-1}),x_i)
    \wedge\\
    &\big(\bigwedge_{l=i-1}^{i}\bigwedge_{j\in F\setminus C_l}(x^{j}_{l}=s^{j}_{l})\big)
  \wedge(N(x_i)\neq a_i)
\end{aligned}
\end{equation}
If the verifier returns \sat{},  $C'_i$ is independent of step  $i-1$, and hence independent of all steps
$1,\ldots,i-1$. Hence, $C'$ is an independent contrastive
example of $\mathcal{E}$.

\item If the query from
  Eq.~\ref{eq:contrastive_examples_for_second_step_case_1} returns
  \unsat{}, we now attempt to decide whether $C'$ is dependent. We
  achieve this through additional verification queries, again setting
  $C_i=C'_i$, but now setting $C_{i-1}$ to a \emph{non empty} set of
  features --- once for every possible set of features, separately. We
  again generate a query using the encoding from
  Eq.~\ref{eq:contrastive_examples_for_second_step_case_1}, and if the
  verifier returns \sat{} it follows that $C'$ is dependent on step
  $i-1$, and that $C''=(C_{i-1},C_i)$ is a contrastive example for
  steps $i-1$ and $i$.  We recursively continue with this enumeration
  process, to determine whether $C''$ is independent,
  dependent of step $i-2$, or a spurious contrastive example.
  
\item In case the previous phases determine that $C'$ is neither
  independent nor part of a larger contrastive example, we conclude
  that it is spurious.
  
\end{enumerate}
An example of a single reverse incremental enumeration step on a
contrastive example $C'$ in our running example is depicted in
Fig.~\ref{fig:reverse_incremental_batching_figure}, and its recursive
call is shown in Alg.~\ref{alg:rie-algorithm} (Cxps denotes the set of all
multi-step contrastive examples containing the initial $C'$).

\begin{figure*}[ht]
\centering

\subfloat[First iteration]{\includegraphics[width=6cm]{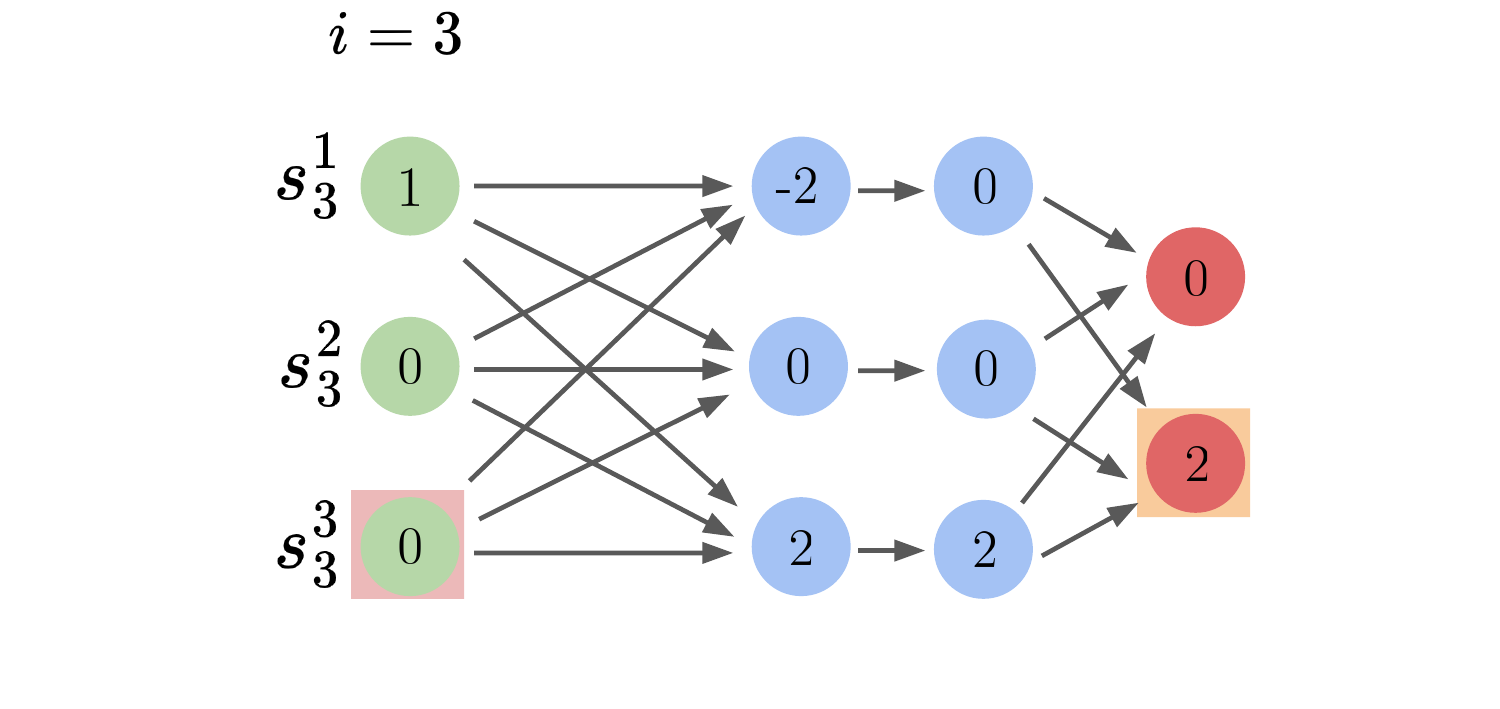}\label{fig:Ex_Im}}
\subfloat[Second iteration]{\includegraphics[width=6cm]{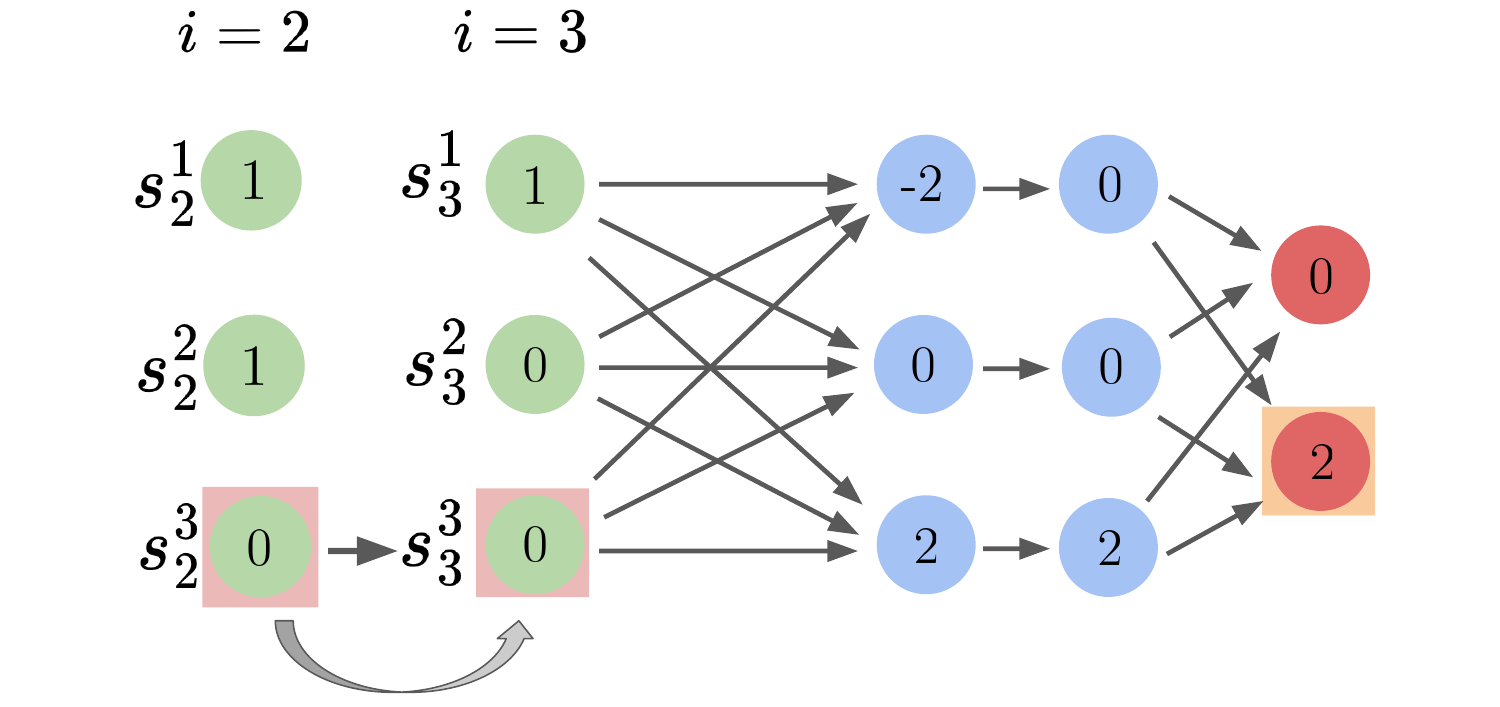}\label{fig:Ex_Im2}}\quad
\subfloat[Third iteration]{\includegraphics[width=6cm]{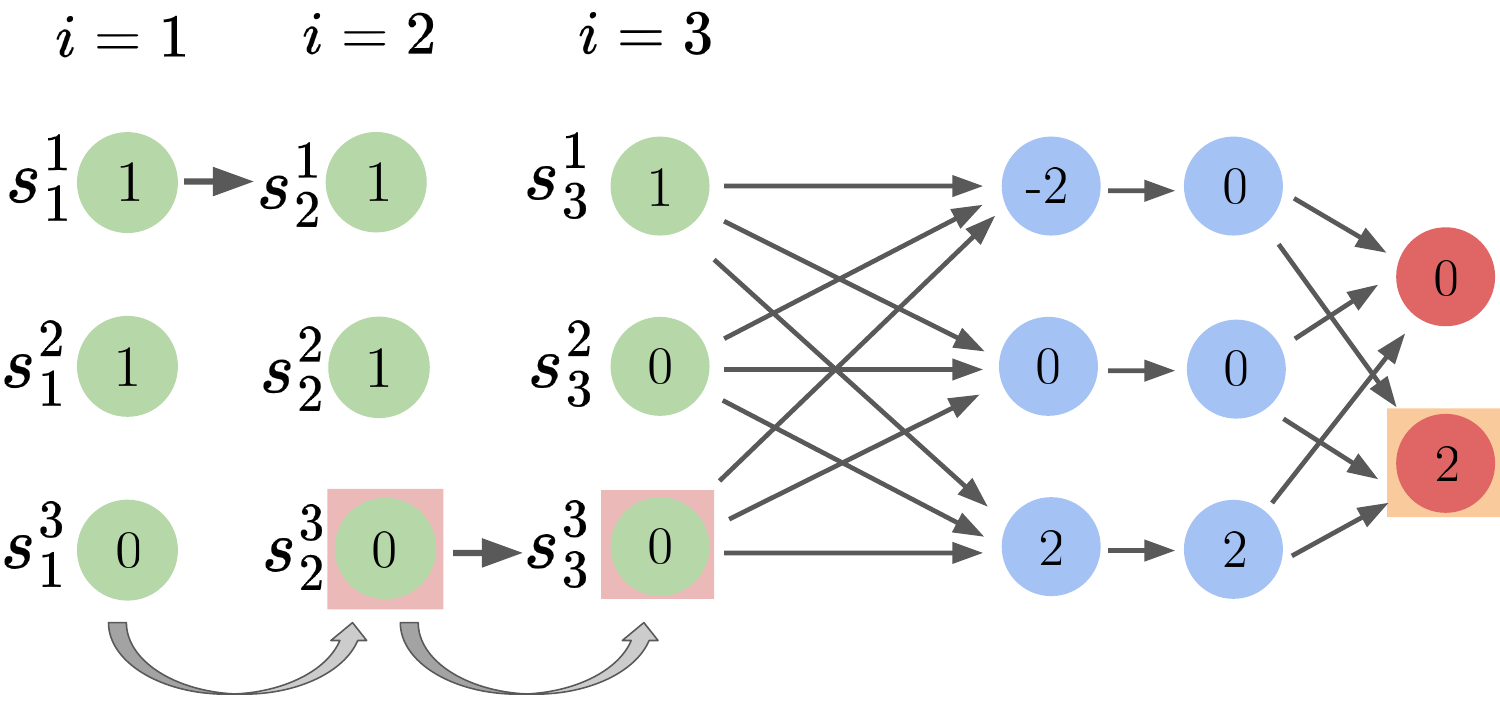}\label{fig:Ex_Im2}}
\caption{An illustration of reverse incremental enumeration. We start with a single-step contrastive example,
  $C'_3=\{s^3\}$ for the third step of the execution.  In
  the second iteration, we find that $(C'_3)$ is dependent on the
  previous step, and that $(\{s^3\},\{s^3\})$ constitutes a contrastive
  example for steps 2 and 3. In the third iteration, $(\{s^3\},\{s^3\})$
  is found to be independent of the first step, and hence $(\emptyset,\{s^3\},\{s^3\})$ is a
  contrastive example for $\mathcal{E}$.}
  \label{fig:reverse_incremental_batching_figure}

\end{figure*}

 \begin{algorithm}[ht]
		\algnewcommand\algorithmicforeach{\textbf{for each}}
		\algdef{S}[FOR]{ForEach}[1]{\algorithmicforeach\ #1\ \algorithmicdo}
		\caption{\texttt{Reverse Incremental Enumeration (\RIE)}}
		\label{alg:rie-algorithm}
    	\textbf{Input} i (starting index), j (reversed index), $C'=(C'_j,\ldots,C'_i)$
		\begin{algorithmic}[1]
            \If{j=1} 
            \State \Return C' \Comment{\textcolor{blue}{C' is trivially 
            independent of steps $j<1$}}
            \EndIf
            \If{$(\emptyset,C'_j,\ldots,C'_i)$ is a contrastive example of steps $j-1\ldots i$}
            \State \Return $(C_l\ |\ \forall 1\leq l\leq j-1, \ C_l=\emptyset)\cdot C'$ 
            \Comment{\textcolor{blue}{C' is independent of step $j-1$}}
			\EndIf
            \State{Cxps $\gets \emptyset$}
			\ForEach {subset $C_f$ of F} 
            \If{$(C_f,C'_j,\ldots,C'_i)$ is a contrastive example of steps $j-1\ldots i$}
            \State{Cxps $\gets$ Cxps $\cup$ 
            \Call{\RIE}{$i,j-1,C_f$}}\Comment{\textcolor{blue}{C' is dependent 
            of step $j-1$}}
			\EndIf
            \EndFor
            \State \Return Cxps \Comment{\textcolor{blue}{if Cxps is empty, C' 
            is spurious}}

		\end{algorithmic}
\end{algorithm}

Using  reverse incremental enumeration, we can find all
multi-step contrastive examples of
$\mathcal{E}$:
\begin{enumerate}
\item First, we find all contrastive examples for the first step of
  $\mathcal{E}$. This is again the same as finding contrastive
  examples of a ``one-shot'' classification problem, and can be
  performed efficiently~\cite{bassan2022towards}, via 
  Alg.~\ref{alg:all-cxps-single}. We first enumerate all
  contrastive examples of size $1$ (i.e., contrastive
  \emph{singletons}); then all contrastive examples of size $2$
  that do not contain contrastive singletons within them; and then
  continue this process for all $1\leq i \leq |F|$ (``skipping'' all
  non-minimal cases).

\item Next, we search for all contrastive examples for the second step
  of $\mathcal{E}$, in the same manner. We perform a reverse
  incremental enumeration on each contrastive example found,
  obtaining all contrastive examples for steps 1 and 2.

\item We continue iteratively, each time visiting a new step $i$ and
  reversely enumerating all contrastive examples that affect steps
  $1,\ldots,i$. We stop when we reach the final step, $i=k$.
\end{enumerate}
The full enumeration process for finding all contrastive examples of
$\mathcal{E}$ is described fully in Alg.~\ref{alg:all-cxps-multi}, which 
invokes Alg.~\ref{alg:all-cxps-single}.

\begin{algorithm}
	\algnewcommand\algorithmicforeach{\textbf{for each}}
	\algdef{S}[FOR]{ForEach}[1]{\algorithmicforeach\ #1\ \algorithmicdo}
	\caption{\texttt{Enumerate-All-Cxps}\xspace}\label{alg:all-cxps-multi}
	\textbf{Input} $N$ (DNN), $F$ ($N$'s features), $\mathcal{E}$ 
	(execution to explain) \Comment{\textcolor{blue}{Global 
			Variables}}		\begin{algorithmic}[1]
		\State{Cxps $\gets \emptyset$}
		\ForEach {$i \in \{1,...,k\}$}
		\State{CxpCandidates $\gets$ 
			\Call{\enumeratecxpssinglestepWithLineBreak}{i}}
		\ForEach {Cxp $\in$ CxpCandidates}
		\State{Cxps $\gets$ Cxps $\cup$ \Call{\RIE}{(Cxp), i, i}}
		\EndFor
		\EndFor
		\State \Return Cxps
		
	\end{algorithmic}
\end{algorithm}

\begin{algorithm}[ht]
	\algnewcommand\algorithmicforeach{\textbf{for each}}
	\algdef{S}[FOR]{ForEach}[1]{\algorithmicforeach\ #1\ \algorithmicdo}
	\caption{\texttt{Enumerate-All-Cxps-In-Single-Step}\xspace}\label{alg:all-cxps-single}
	\textbf{Input} $N$ (DNN), $F$ ($N$'s features), $\mathcal{E}$ (execution to 
	explain), i (step number)
	\begin{algorithmic}[1]
		\State{Cxps $\gets \emptyset$} \Comment{\textcolor{blue}{denotes the
				set of all contrastive examples}}
		\ForEach{$1\leq i \leq |F|$}
		\ForEach {subset $c$ of $F$ of length $i$ not containing sets from Cxps}
		\If{\verifycontrastiveexamplesingle is \sat{}}
		\State{Cxps $\gets$ Cxps $\cup$ c}
		\EndIf
		\EndFor
		\EndFor
		\State \Return Cxps
		
	\end{algorithmic}
\end{algorithm}




We also make the following observation: we can further expedite the enumeration 
process by discarding sets that contain contrastive examples within them 
since 
we are specifically searching for minimal contrastive examples. For instance, 
in the given example in Fig.~\ref{fig:reverse_incremental_batching_figure}, if 
we find $(\emptyset,{s^1},\emptyset)$ as a contrastive example for the entire 
multi-step instance, we no longer need to consider sets in step $2$ that 
contain $s^1$ when iterating in reverse from step $3$ to step $2$. Our 
evaluation shows that this approach can significantly improve performance as 
the increasing number of contrastive examples found in previous steps greatly 
reduces the search space.


Of course, our approach's worst-case complexity is still exponential in the 
number of steps, $k$, because each dependent contrastive example requires a 
recursive call that potentially enumerates all contrastive examples for the 
previous step. However, the number of recursive iterations is limited by the 
dependency between steps. For instance, if contrastive examples in step $i$ are 
only dependent on step $i-1$ and not on step $i-2$, the recursive iterations 
will be limited to $2$. Additionally, skipping the verification of candidates 
containing contrastive examples found in previous steps can also significantly 
reduce runtime.



\section{Evaluation} 
\label{sec:evaluation}

\mysubsection{Implementation and Setup.}  We created a
proof-of-concept implementation of all aforementioned approaches and 
benchmarks~\cite{ArtifactRepository}. To
search for explanations, our tool~\cite{ArtifactRepository}
dispatches verification queries using a backend DNN verifier (we use 
\marabou~\cite{KaHuIbJuLaLiShThWuZeDiKoBa19}, previously employed in additional 
studies~\cite{AmFrKaMaRe23, 
	AmWuBaKa21,AmScKa21,AmZeKaSc22,AmCoYeMaHaFaKa23,AmMaZeKaSc23,CaKoDaKoKaAmRe22,
	ReKa22},
although other
engines may also be used). The queries encode the architecture of the
DNN in question, the transition constraints between consecutive steps
of the reactive system, and the candidate explanation or contrastive
example being checked. Calculating the MHS, when relevant, was done 
using \rcTwo, a \maxSat-based tool of the \pysat 
toolkit~\cite{ignatiev2018pysat}.



\mysubsection{Benchmarks.} We trained DRL agents for two well-known reactive 
system benchmarks: GridWorld~\cite{SuBa18} and TurtleBot~\cite{TaPaLi17} (see 
Fig.~\ref{fig:benchmarks}). 
GridWorld involves an agent moving in a 2D grid, while TurtleBot is a 
real-world robotic navigation platform. These benchmarks have been extensively 
studied in the DRL literature. GridWorld has 8 input features per state, 
including agent coordinates, target coordinates, and sensor readings for 
obstacle 
detection. The agent can take $4$ possible actions: \up, \down, \leftOutput, or 
\rightOutput. TurtleBot has $9$ input features per state, including lidar 
sensor readings, target distance, and target angle. The agent has $3$ 
possible actions: \leftOutput, \rightOutput, or \forwardOutput. We trained our 
DRL agents with the state-of-the-art PPO algorithm~\cite{ShWoDh17}.
Additional details 
appear in 
Sec.~\ref{sec:appendix:Training} and~\ref{sec:appendix:TransitionEncoding} of 
the 
appendix.

\begin{figure}[H]
	\centering
	\captionsetup{justification=centering}
	\begin{center}
		\includegraphics[width=1.0\linewidth]{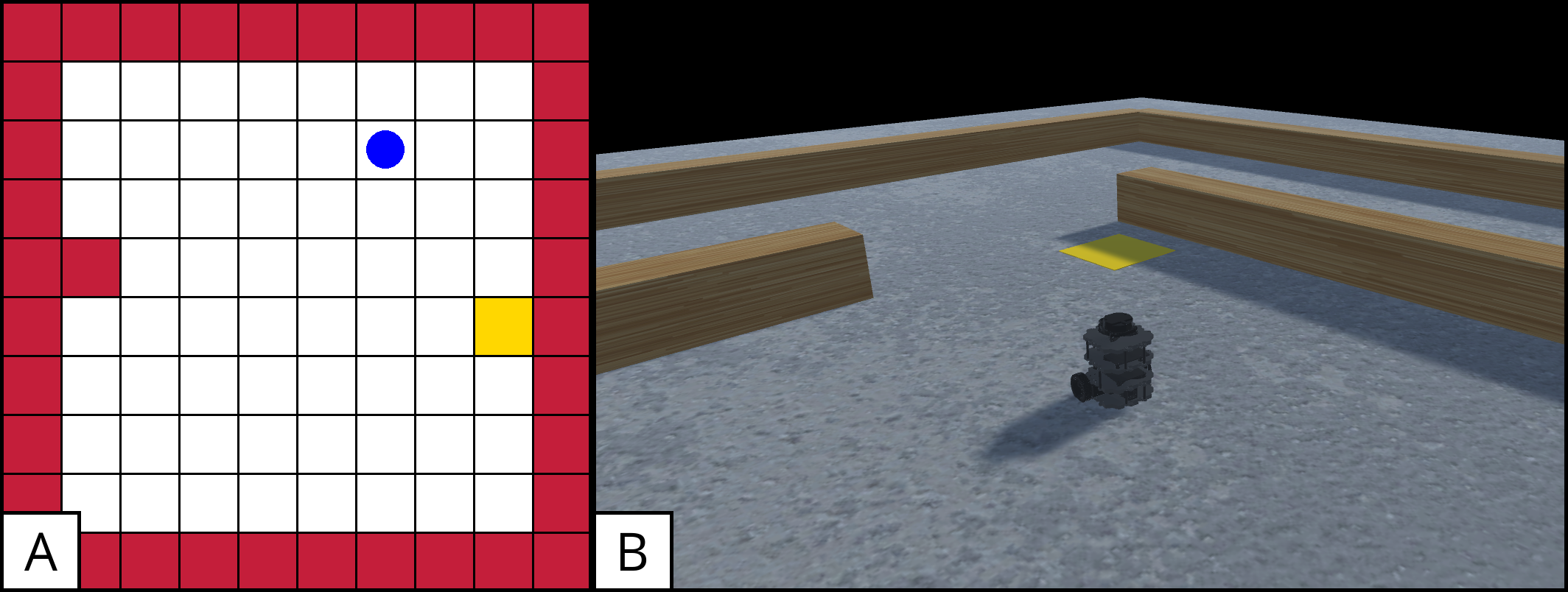}
	\end{center}
	\caption{Benchmarks: (A) GridWorld; and (B) TurtleBot.}
	\label{fig:benchmarks}
\end{figure}




\mysubsection{Generating Executions.}  We generated 200 unique
multi-step executions of our two benchmarks: $100$ GridWorld
executions (using $10$ agents, each producing $10$ unique executions
of lengths $6 \leq k \leq 14$), and $100$ TurtleBot executions (using
$100$ agents, each producing a single execution of length
$6 \leq k \leq 8$).  Next, from each $k$-step execution, we generated
$k$ unique sub-executions, each representing the first $i$ steps of
the original execution ($1\leq i\leq k$). This resulted in a total of
$931$ GridWorld executions and $647$ Turtlebot executions.  We used
these executions to assess the different methods for finding minimal
and minimum explanations. Each experiment ran with a timeout value of
$3\cdot i$ hours, where $i$ is the execution's length.

\mysubsection{Experiments.} We begin by 
comparing the performance of the  four methods discussed in 
Sec.~\ref{sec:formal_k_explanations_computation}:
\begin{inparaenum}[(i)]
\item encoding the entire network as a ``one-shot'' query;
\item computing individual explanations for each step;
\item incrementally enumerating explanations; and
\item reversely enumerating contrastive examples and calculating their
  MHS.
\end{inparaenum}
We note that we use Methods 1--3 to generate both minimal and
minimum explanations, whereas Method $4$ is only used to generate
minimum explanations. To generate minimum explanations using the
``one-shot'' encodings of Methods $1$ and $2$, we use the
state-of-the-art approach of Ignatiev et
al.~\cite{ignatiev2019abduction}. We use two common criteria for
comparison~\cite{ignatiev2019abduction,
  bassan2022towards, ignatiev2020contrastive}: the \emph{size} of
the generated explanations (small explanations tend to be more
meaningful), and the overall runtime and timeout ratios.



\mysubsection{Results.}  Results for the GridWorld benchmark are
presented in Table~\ref{table:gridWorld}. These results clearly indicate that 
Method 2 (generating
explanations in independent steps) was significantly faster in all
experiments, but generated drastically larger explanations --- about
two times larger when searching for a \emph{minimal} explanation, and
about five times larger for a \emph{minimum} explanation, on
average. This is not surprising; as noted earlier, the explanations
produced by such an approach do not take the transition constraints
into account, and hence, may be quite large. In addition, we note again
that this approach does not guarantee the minimality of the
combined explanation, even when combining minimal/minimum explanations for each 
step. 
The corresponding results for TurtleBot
appear in Sec.~\ref{sec:appendix:supplementaryResults} of the appendix, and 
also demonstrate similar outcomes.

\begin{table}
  \centering
    \caption{\textit{GridWorld}: columns from left to right: experiment type, 
	method name (and number), time and size of returned explanation (out 
	of 
	experiments that 
	terminated), and the percent of solved instances (the rest timed 
	out). The bold highlighting indicates the method that generated the 
	explanation with the optimal size.}
	
	\centering
        \resizebox{\columnwidth}{!}{%
	\begin{tabular}{c||c||c||ccc||c} 
		\hline
	\multirow{2}{*}{\textbf{setting}}                                       
		   & \multirow{2}{*}{\textbf{experiment}} 
		& \textbf{time (s)} & \multicolumn{3}{c||}{\textbf{size}}         & 
		\multirow{2}{*}{\begin{tabular}[c]{@{}c@{}}\textbf{solved}\\\textbf{(\%)}\end{tabular}}

		  \\
		&                                      &                              
		\textbf{avg.}     & \textbf{min} & \textbf{avg.} & \textbf{max} 
		&                                                                       
		                                                                        
		                            \\ 
		\hline
		\multirow{3}{*}{\begin{tabular}[c]{@{}c@{}}minimal\\(local)\end{tabular}}
		
		& one-shot (1)                             & 
	                  
		304               & 5            & 33            & 112          & 
		98

		\\
		
		& independent (2)                                                      
		& 1                 & 5            & 34            & 
		97           & 
		99.9                                                                    
                                                           
		\\
		
		& \textbf{incremental}  (3)                          & 
		 
		\textbf{1}                 & \textbf{5}            & \textbf{18}            & 
		\textbf{78}           & 
		\textbf{99.7}                                                                    
                                                           
		\\
		
		\hline
		\multirow{4}{*}{\begin{tabular}[c]{@{}c@{}}minimum\\(global)\end{tabular}}
		
		& one-shot   (1)                                                    & 
		405               & 5            & 14            & 32           & 
		29.8                                                                    
                                                           
		\\
		
		& independent   (2)                         & 
		 
		4               & 5            & 35            & 98         & 
		98.3

		\\
		
		& incremental   (3)                                                    
		& 
		1,396             & 5            & 7             & 
		9            & 
		17.9                                                                    
                                                         
		                         \\
		& \textbf{reversed}  (4)                               & 
		                        
		\textbf{39}                & \textbf{5}            & \textbf{7}             & 
		\textbf{16}           & 
		\textbf{99.7}                                                           
                                                            
        \\
		\hline
	\end{tabular}
        }
    \label{table:gridWorld}
\end{table}

When comparing the three approaches that can guarantee minimal
explanations, the incremental enumeration approach (Method 3) is
clearly more efficient than the ``one-shot'' scheme (running for about $1$ 
second compared to above $5$ minutes, on average, across all solved instances), 
as depicted in Fig.~\ref{fig:cumulative_time_minimal_explanation}. 
For the
minimum explanation comparison, the results show that the
reversed-enumeration-based strategy (Method 4) ran significantly faster
than all other methods that can find guaranteed minimum explanations:
on average, it ran for $39$ seconds, while the other methods ran for
more than $6$ and $23$ minutes. In addition, out of all methods
guaranteed to produce a minimum explanation, experiments that ran with
the ``reversed'' strategy hit significantly fewer timeouts.
The 
``reversed'' strategy outperforms the competing
methods significantly, on both benchmarks (see 
Fig.~\ref{fig:cumulative_time_minimum_explanation}).

Next, we analyzed the strategies at a higher resolution --- focusing 
on a \emph{step-wise} level comparison, i.e., on analyzing how the length of 
the execution 
affected runtime. 
The results (see 
Figs.~\ref{fig:res_gridworld_local}-~\ref{fig:res_turtlebot_full} of the 
appendix) 
demonstrate the drastic 
performance gain of our ``reversed'' strategy as $k$ increases: this strategy can efficiently find explanations for longer executions, 
while the competing ``one-shot'' strategy fails. This again is not surprising: 
when dealing with large numbers of steps, the transition function,
the unrolling of the network, and the underlying enumeration scheme become 
more taxing on the underlying verifier. A full analysis of both benchmarks, 
and all explanation types, appears in 
Sec.~\ref{sec:appendix:supplementaryResults} of the 
appendix. 
%

\begin{figure}[t]
	\centering
	{\includegraphics[width=0.24\textwidth]{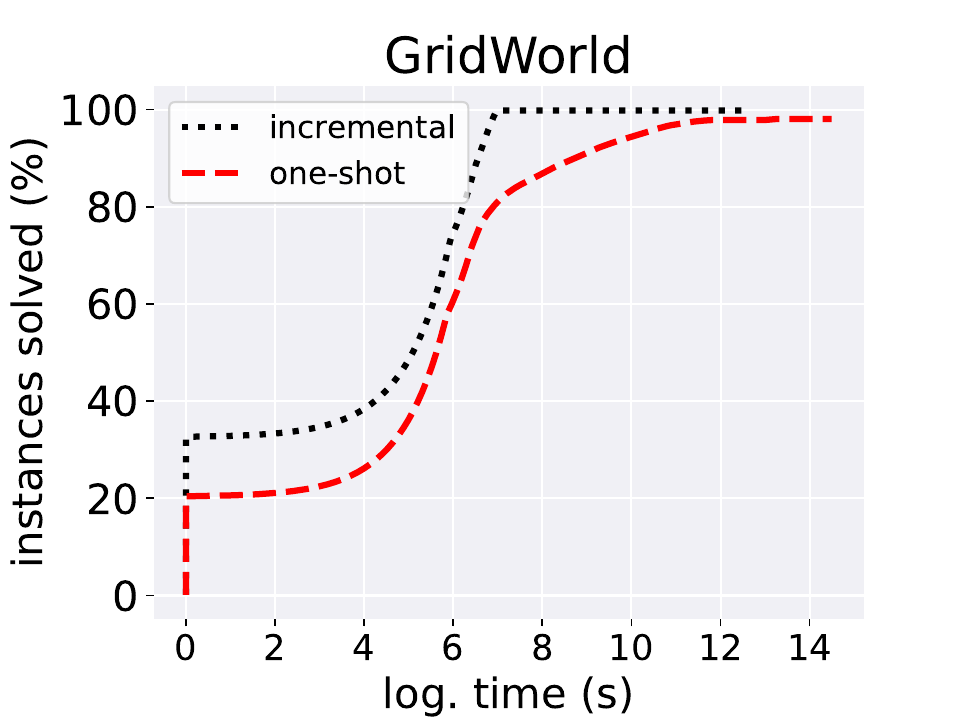}}
	{\includegraphics[width=0.24\textwidth]{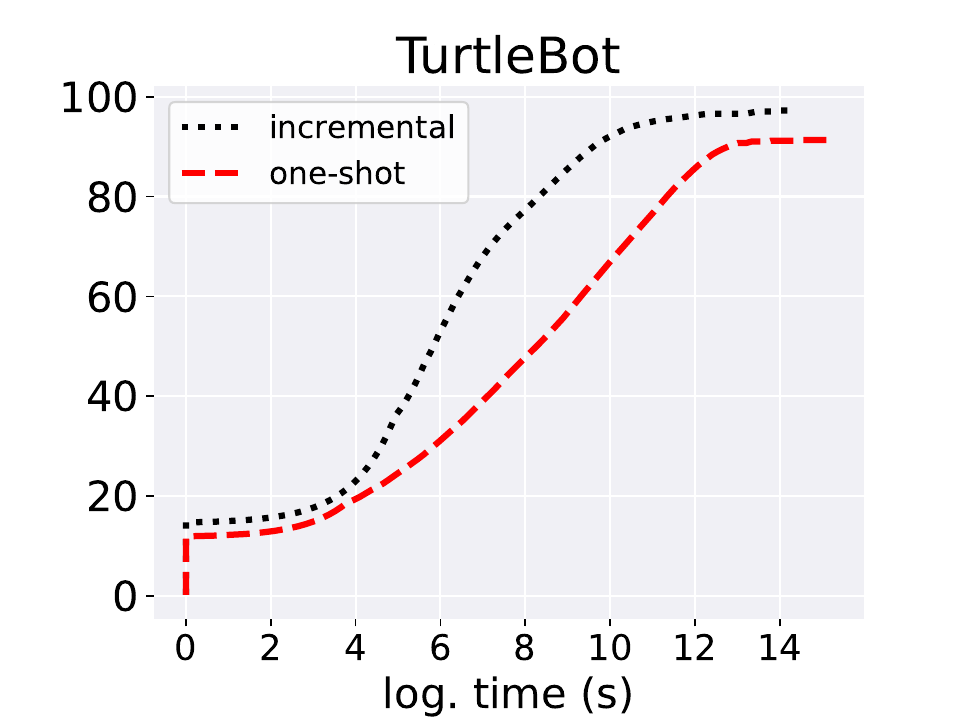}}
	\caption{\textit{Minimal explanation}: number of solved instances depending 
	on (accumulative) time, for the methods that guarantee minimality.}
	\label{fig:cumulative_time_minimal_explanation}
\end{figure}

\begin{figure}[t]
	\centering
	{\includegraphics[width=0.24\textwidth]{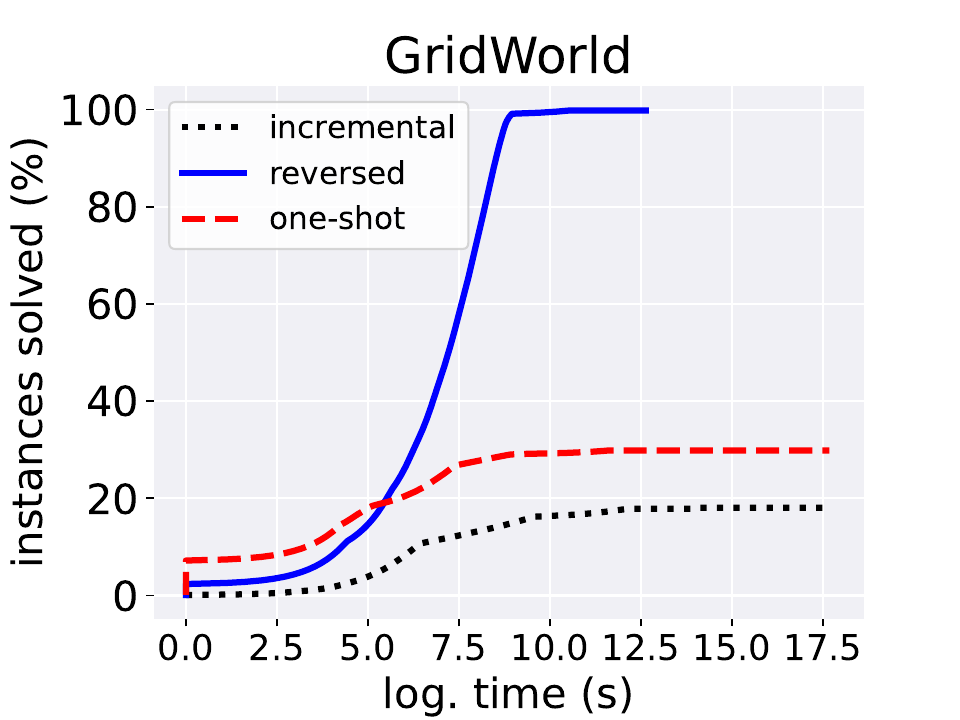}}
	{\includegraphics[width=0.24\textwidth]{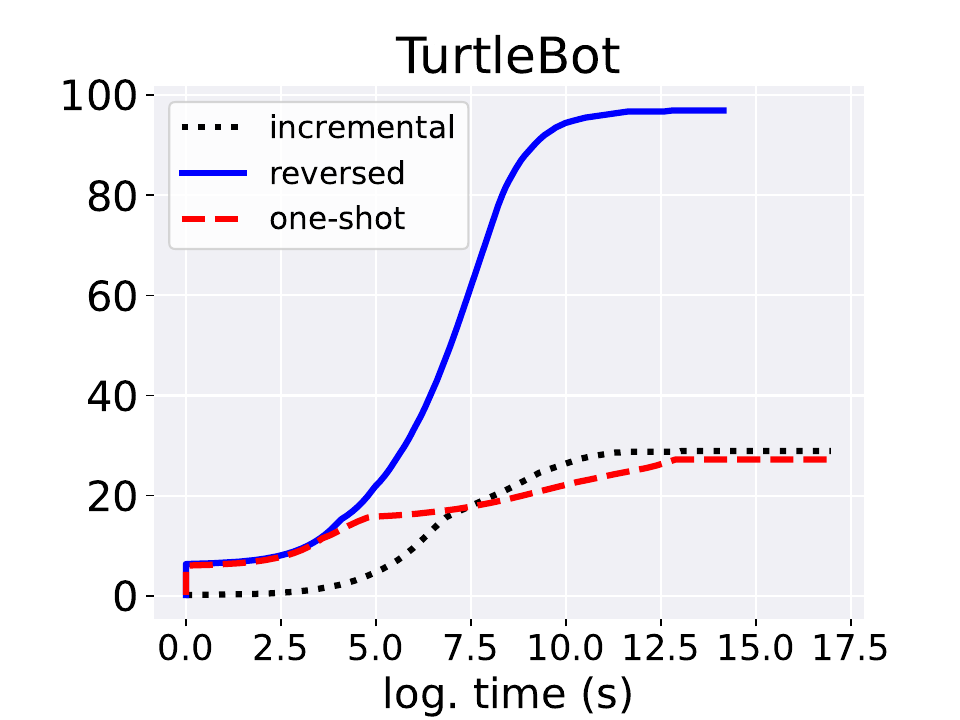}}
	\caption{\textit{Minimum explanation}: number of solved instances depending 
	on (accumulative) time, for the methods that guarantee minimality.}
	\label{fig:cumulative_time_minimum_explanation}
\end{figure}

\mysubsection{Explanation Example.}
We provide a visual example
of an instance from our GridWorld experiment identified as a
minimum explanation. The results (depicted in Fig.~\ref{fig:visual_example}) 
include a minimum explanation for an execution of $8$ steps. 
They show the following meaningful insight: fixing part of the agent's 
location sensors at the initial step, 
and a single sensor in the sixth step, is sufficient
for forcing the agent to move along the original path, regardless of
any other sensor reading.  

\begin{figure}[t]
	\centering
	\begin{center}		
		\includegraphics[width=1.0\linewidth]{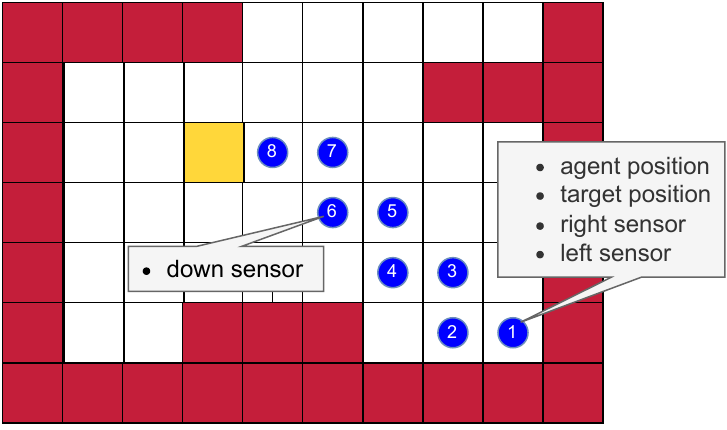}
	\end{center}
	\caption{\textit{GridWorld}: a $5$-sized explanation for an $8$-step 
		execution.
		The steps are numbered (in blue circles), the target
		is the yellow square, and the
		obstacles are depicted in red.}
	\label{fig:visual_example}
\end{figure}

\mysubsection{Comparison to Heuristic XAI Methods.} We
also compared our results to popular, non-verification-based,
heuristic XAI methods. Although these methods proved scalable, they often
returned unsound explanations when compared to our approach. For additional
details, see Section~\ref{sec:appendix:heuristicMethods} 
of the appendix.

\section{Related Work}

\label{sec:RelatedWork}
This work joins recent efforts on utilizing formal verification
to explain the decisions of ML
models~\cite{ignatiev2019abduction, shih2018symbolic,
	shi2020tractable, bassan2022towards, wu2022verix, fel2022don,
	la2021guaranteed }. 
Prior studies primarily focused on formally explaining \emph{classification} over various domains~\cite{ignatiev2019abduction,bassan2022towards,wu2022verix,ignatiev2019abduction,
	ignatiev2019validating,la2021guaranteed} or specific
model 
types~\cite{HuIzIgMa21,marques2020explaining,izza2020explaining,ignatiev2021sat,ignatiev2018sat}.
 while others explored alternative ways of defining 
explanations over classification tasks~\cite{wu2022verix, 
la2021guaranteed,izza2021efficient, 
	waeldchen2021computational, 
	AnPaDiCh19,PrAf20,mcmillan2020bayesian,HuCoMoPlMa23}.

Methods closer to our own have focused on formally
explaining DNNs~\cite{ignatiev2019abduction, bassan2022towards,
	wu2022verix, la2021guaranteed,HuMa23}, where the problem is known to be
complex~\cite{liberatore2005redundancy, ignatiev2019abduction}. This
work relies on the rapid development of DNN
verification~\cite{katz2017reluplex, ZhShGuGuLeNa20,
	GeLeXuWaGuSi22,
	AvBlChHeKoPr19,
	KoLoJaBl20,
	BaShShMeSa19,Al21,GuPuTa21}.  
There has also been ample work on heuristic 
XAI~\cite{ribeiro2016should,ribeiro2018anchors,lundberg2017unified,
	gunning2019xai, selvaraju2017grad}, including approaches for
explaining the decisions of reinforcement-learning-based reactive systems 
 (XRL)~\cite{puiutta2020explainable,
	madumal2020explainable, heuillet2021explainability,
	juozapaitis2019explainable}. However, these methods do not
provide formal guarantees. 


\section{Conclusion}
\label{sec:Conclusion}
Although DNNs are used extensively within reactive systems, they remain 
``black-box'' models, uninterpretable to humans. We seek to mitigate this 
concern by producing formal explanations for executions of
reactive systems. As far as we are aware, we are the first to provide a 
formal 
basis of explanations in 
this context, and to suggest methods for efficiently producing such
explanations --- 
significantly outperforming the competing approaches.
We also note that our approach is agnostic to the type of reactive system, and 
can be generalized beyond DRL systems, to any k-step reactive DNN system 
(including RNNs, LSTMs, GRUs, etc.).
Moving forward, a main extension could be scaling our method, beyond the simple 
DRLs evaluated here, to larger systems of higher complexity. Another 
interesting extension could include evaluating the attribution of the 
hidden-layer features, rather than just the input features.  

%

%
%

\mysubsection{Acknowledgments.}  The work of Bassan, Amir, Refaeli, and Katz was
partially supported by the Israel Science Foundation (grant number
683/18). The work of Amir was supported by a scholarship from the Clore Israel 
Foundation. 
The work of Corsi was partially supported by the ``Dipartimenti di Eccellenza 
2018-2022'' project, funded by the Italian Ministry of Education, Universities, 
and Research (MIUR).


{
	\bibliographystyle{abbrv}
	\bibliography{references}
}

\newpage
\onecolumn

\setcounter{section}{0}

{\huge{Appendix}}

\section{Examples of Minimal and Minimum Explanations}
\label{sec:appendix:minimalandMinimumExplanationExampleFigure}

We present figures depicting a minimal explanation, and a minimum explanation, 
for the toy DNN (depicted in Fig.~\ref{fig:neural_network_example_1}), and 
the input $V_1=[1,1,1]^T$.

\vspace{10pt}
\begin{figure}[H]
	\centering
	\includegraphics[width=0.32\textwidth]{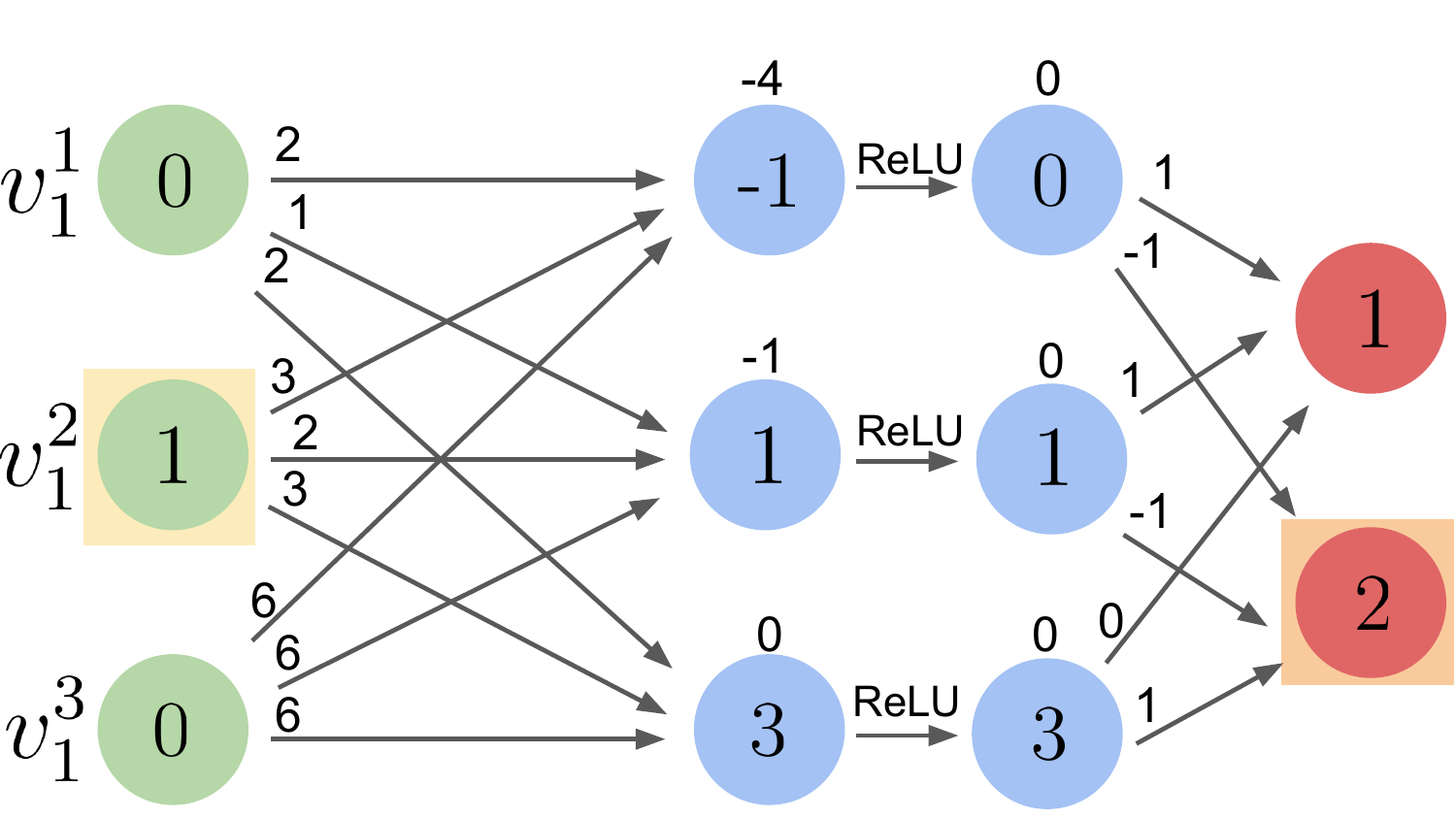}
	\hspace{40pt}
	\includegraphics[width=0.32\textwidth]{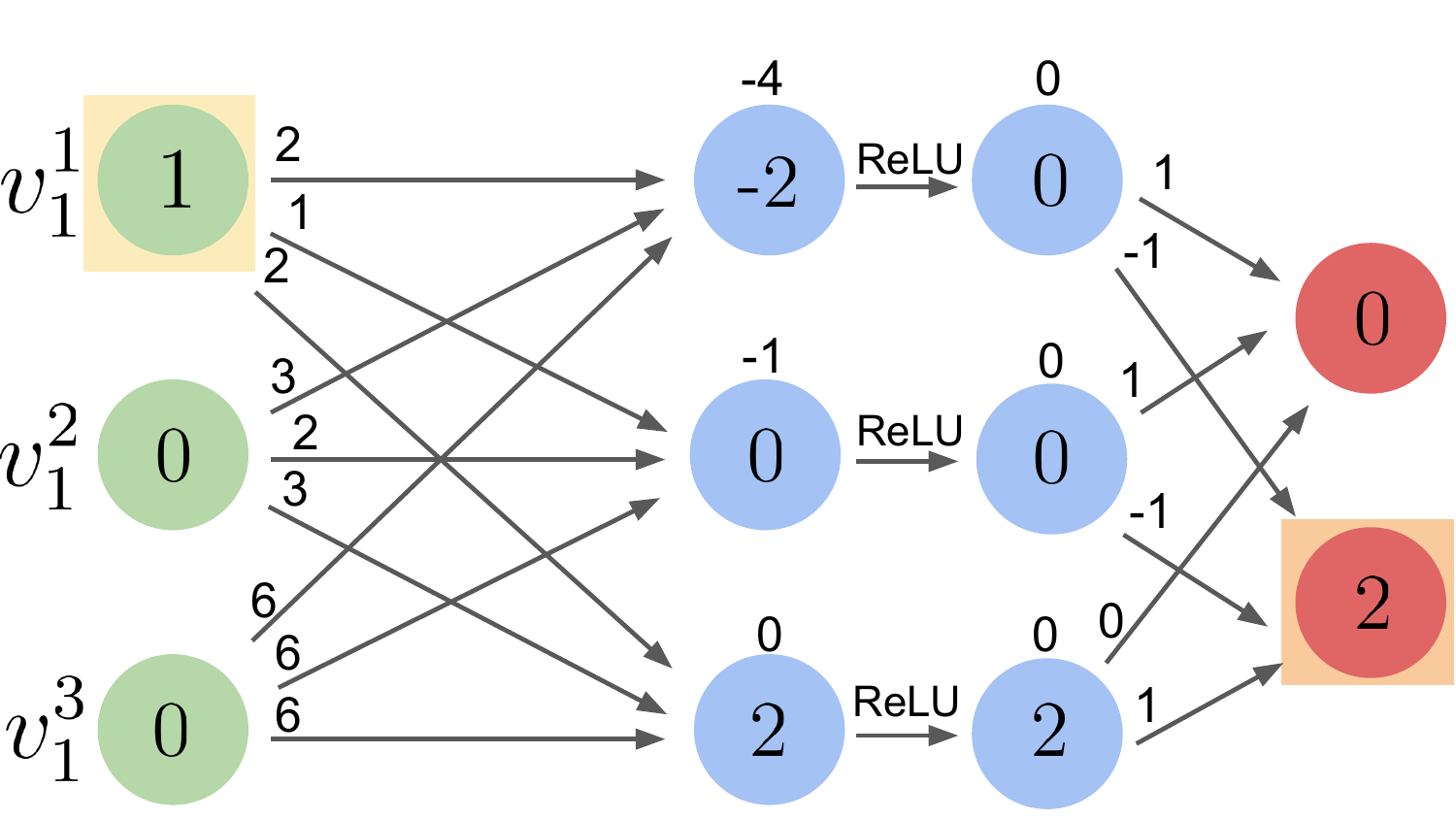}
	\caption{$\{ v_1^1, v_1^2 \}$ is a minimal explanation for input 
	$V_1=[1,1,1]^T$.}
	\label{fig:neural_network_minimal_explanation}
\end{figure}

\vspace{20pt}
\begin{figure}[ht]
	\centering
	{\includegraphics[width=0.32\textwidth]{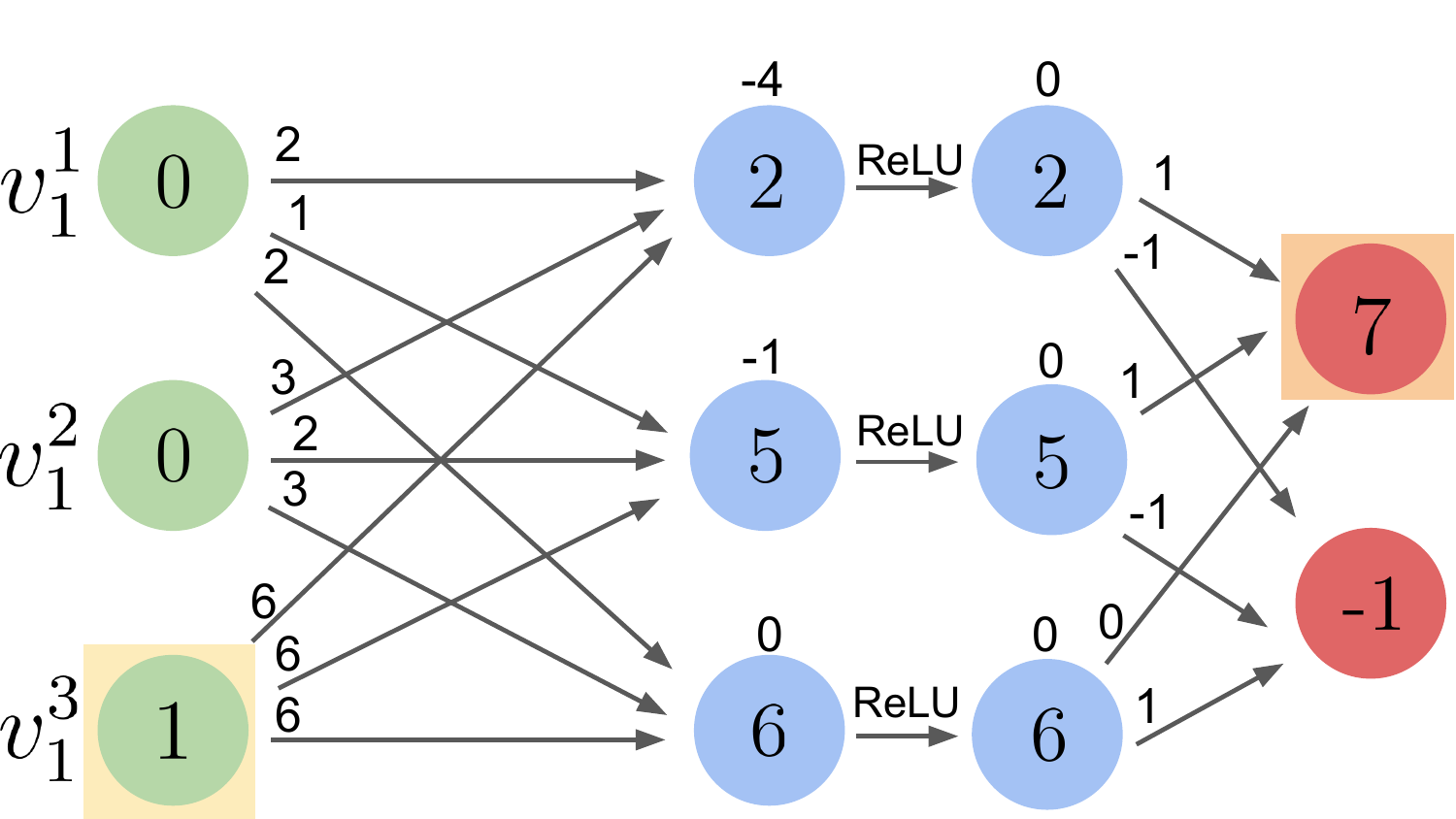}}
	\hfill
	{\includegraphics[width=0.32\textwidth]{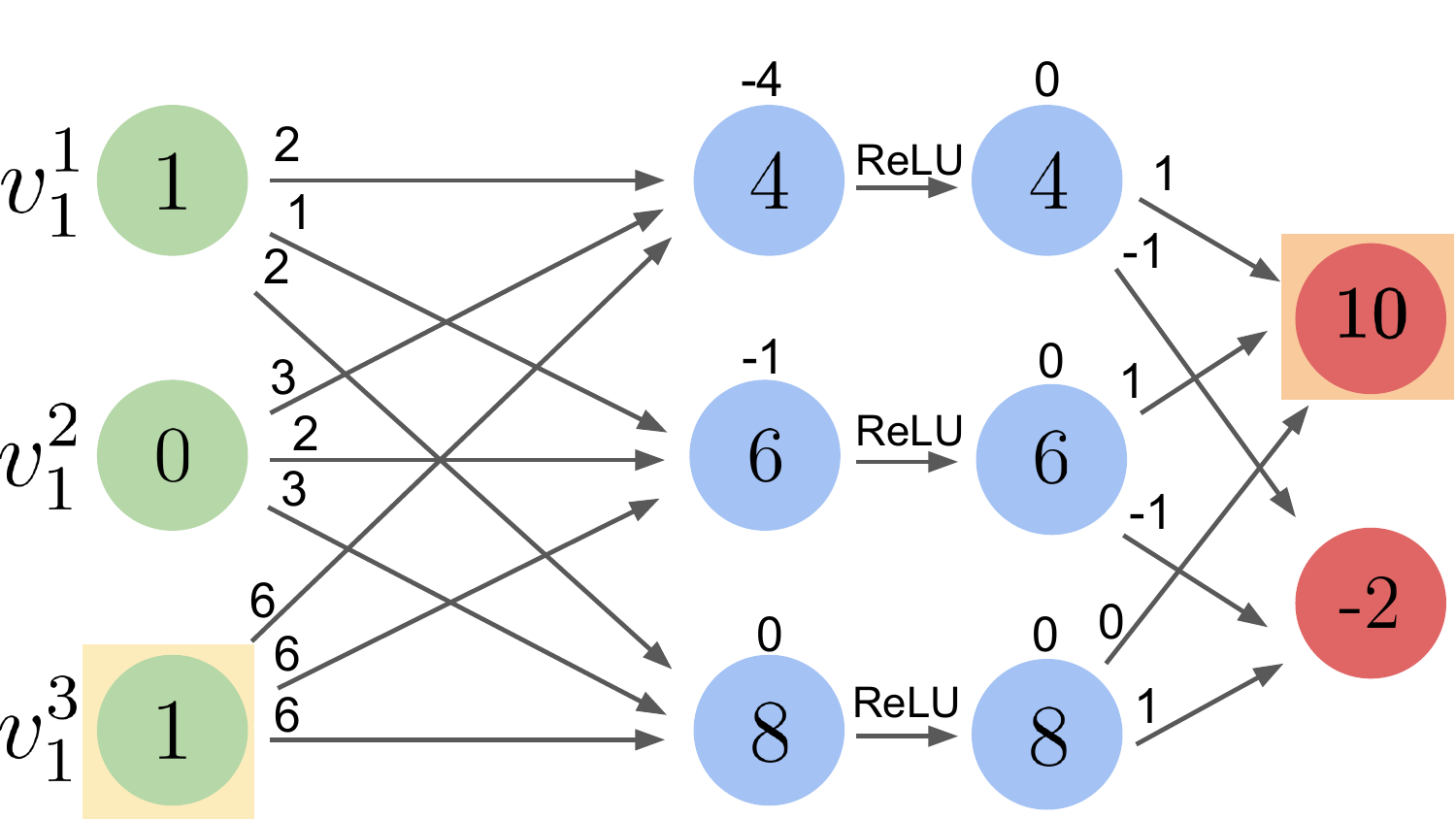}}
	\hfill
	{\includegraphics[width=0.32\textwidth]{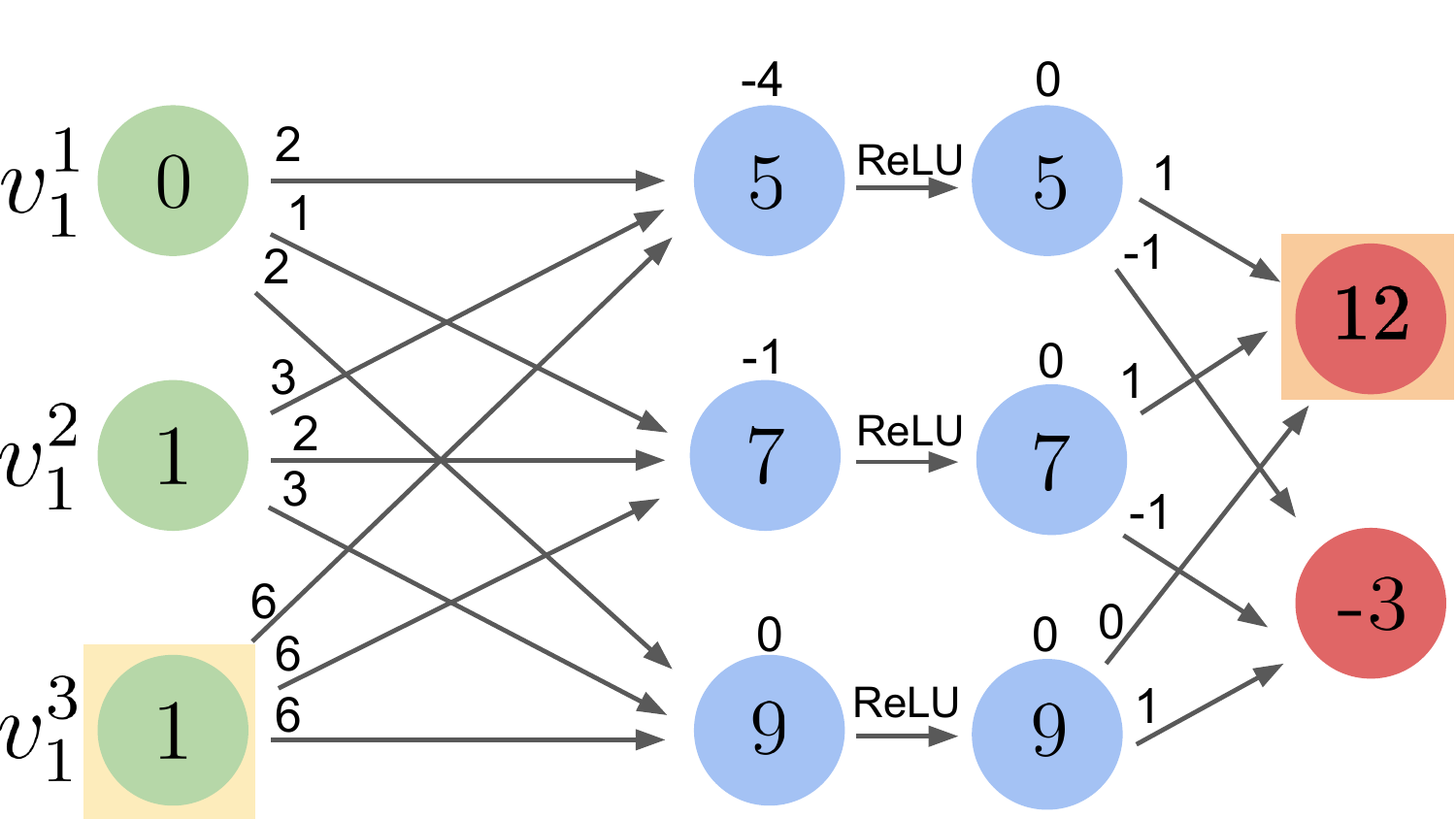}}
	\caption{$\{ v_1^3 \} $ is a minimum explanation for input $V_1=[1,1,1]^T$.}
	\label{fig:neural_network_minimum_explanation}
\end{figure}

\vspace{20pt}

\section{Minimum Hitting Set (MHS)}
\label{sec:appendix:mhsDefinition}

Given a collection $\mathbb{S}$ of sets from a universe U, a hitting set $h$ for
$\mathbb{S}$ is a set such that
$\forall S \in \mathbb{S}, h\cap S \neq \emptyset$. A hitting set $h$
is said to be \emph{minimal} if none of its subsets is a hitting set,
and \emph{minimum} if it has the smallest possible cardinality among
all existing hitting sets.

\section{Additional Proofs}
\label{sec:appendix:additionalProofs}

We present the full proofs for the three lemmas mentioned in this work.

\setcounter{lemma}{0}
  \begin{lemma}
  \label{method3_first_lemma_for_appendix}
  Let $E=\{E_1, E_2, \ldots E_k\}$ be a $k$-step explanation for
execution $\mathcal{E}$, and let $1\leq i \leq k$. Let $E'$ be the set
obtained by removing a set of features $F' \subseteq E_i$ from $E_i$,
i.e.,
$E'=(E_1, \ldots, E_{i-1},E_i\setminus F', E_{i+1}, \ldots, E_k)$. In
this case, fixing the features in $E'$ prevents the 
 first $i-1$ actions, $(a_{1}, a_{2}, \ldots a_{i-1})$, from changing.
  \end{lemma}

  \begin{proof}
  After removing $F'$ from $E$, when validating if $E'$ is an explanation for 
  $\mathcal{E}$, then all features in $(E_1,\ldots,E_{i-1})$ are still fixed to 
  their corresponding values. Assume by contradiction that one of the actions 
  $(a_{1}, a_{2}, \ldots a_{i-1})$ changed, then $(E_1,\ldots,E_{i-1})$ is not 
  an explanation for the first $i-1$ steps of $\mathcal{E}$, contradicting the 
  assumption.  
  Hence, actions $(a_{1}, a_{2}, \ldots a_{i-1})$ were selected.
\end{proof}

\begin{lemma}
\label{method3_second_lemma_for_appendix}
Let $E=(E_1, E_2, \ldots, E_k)$ be a $k$-step explanation for
execution $\mathcal{E}$, and let $1\leq i \leq k$ such that
$\forall j>i$ it holds that $E_j=F$. Let $E'$ be the set obtained by
removing a set of features $F'\subseteq E_i$ from $E_i$, i.e.,
$E'=(E_1, \ldots, E_{i-1},E_i\setminus F', E_{i+1}, \ldots, E_k)$.
In this case,  fixing the  features in $E'$ prevents any changes in the first $i-1$ actions $(a_1,\ldots,a_{i-1})$, and if at least one of the last $k-i+1$ actions $(a_i,\ldots,a_k)$ changed, then $a_i$ must have changed.
\end{lemma}
\begin{proof}
When fixing the features of $E'$ to their corresponding values, it holds from 
Lemma~\ref{method3_first_lemma_for_appendix} that actions 
$(a_1,\ldots,a_{i-1})$ were selected. Assume by contradiction that some $a_j$ 
such that $j>i$ changed, and that $a_i$ did not. If $a_i$ did not change, hence 
$(E_1,\ldots,E_i)$ is an explanation for the first $i$ steps of $\mathcal{E}$. 
More formally, $\forall x_1, x_2,\ldots,x_i \in \mathbb{F}$:

\begin{equation}
  \big
  (
  \bigwedge_{l=1}^{i}
    T(x_{l},N(x_{l}),x_{l+1})
    \wedge
    \bigwedge_{l=1}^{i}
    \bigwedge_{r\in E_{l}}(x^{r}_{l}=s^{r}_{l})
    \big)
\to
  \bigwedge_{l=1}^{i}
N(x_l)=a_{l}
  \end{equation}
  Since we know that $a_i$ occurred then we know that 
  $T(x_{i},N(x_{i}),x_{i+1})$ holds, and since all features in steps 
  $i+1,\ldots,k$ are fixed to their original values then fixing them clearly 
  determines $N(x_l)=a_l$ for all $l\geq i$ and that the transitions 
  $T(x_{l},N(x_{l}),x_{l+1})$ also hold. Overall we get that $\forall x_1, 
  x_2,\ldots,x_k \in \mathbb{F}$:
  \begin{equation}
  \big
  (
  \bigwedge_{l=1}^{k}
    T(x_{l},N(x_{l}),x_{l+1})
    \wedge
    \bigwedge_{l=1}^{k}
    \bigwedge_{r\in E_{l}}(x^{r}_{l}=s^{r}_{l})
    \big)
\to
  \bigwedge_{l=1}^{k}
N(x_l)=a_{l}
  \end{equation}
  meaning that $(E_1,\ldots,E_k)$ is an explanation for $\mathcal{E}$. Hence, 
  it is not possible for $a_j$ to be altered, contradicting the 
  assumption.  
\end{proof}

\begin{lemma}
\label{method4_second_lemma_for_appendix}
Let $\mathcal{E}$ be a $k$-step execution, and let
$C=\{C_1,\ldots,C_k\} $ be a minimal contrastive example for
$\mathcal{E}$; i.e., altering the features in $C$ can cause at
least one action in $\mathcal{E}_A$ to change. Let
$1\leq i \leq k$ denote the index of the first action
$a_i$ that can be changed by features in $C$. It holds that: $C_i \neq \emptyset$; $C_j=\emptyset$ for
all $j>i$; and if there exists some $l<i$ such that
$C_l\neq \emptyset$, then all sets $\{C_{l}, C_{l+1},\ldots,C_i\}$
are not empty.
\end{lemma}
\begin{proof}
Since $a_i$ denotes the first action that can be potentially changed by 
altering the values of $C$, then all actions 
$(a_1,...,a_{i-1})$ were selected, and thus $(F\setminus C_1,\ldots,F\setminus C_{i-1})$ is 
an explanation for the first $i-1$ steps of $\mathcal{E}$. It also holds that 
$(C_1,\ldots,C_{i})$ is a contrastive example for the first $i$ steps of 
$\mathcal{E}$, since altering its values can cause $a_i$ to change.

First, assume by contradiction that there exists some $C_j\neq \emptyset$ for 
some $j>i$. Since $(C_1,\ldots,C_i)$ is a contrastive example for the first $i$ 
steps of $\mathcal{E}$, then there exists some contrastive example 
$C'=(C_1,\ldots,C_i, \emptyset \ldots, \emptyset)$ for $\mathcal{E}$. Since 
$|C'|<|C|$, it thus holds that $C=(C_1,\ldots,C_k)$ is not minimal. Hence, 
$C_j=\emptyset$ for all $j>i$. 

Second, assume by contradiction that $C_i=\emptyset$. Since we proved that 
$C_j=\emptyset$ for all $j>i$ then 
$C=\{C_1,\ldots,C_{i-1},\emptyset,\ldots,\emptyset\}$. Let there be some 
$E=\{E_1,\ldots,E_k\}$ such that for all $1\leq i \leq k$ it holds that 
$E_i=F\setminus C_i$. Since $C_j=\emptyset$ for all $j>i$ then it holds that 
for all $j>i$, $E_j=F$. Since we also know that $(F\setminus 
C_1,\ldots,F\setminus C_{i-1})$ is an explanation for the first $i-1$ steps 
then it follows from Lemma~\ref{method3_second_lemma_for_appendix} that when 
fixing the 
values of $E$, and allowing the values of $C$ to alternate freely, then it 
holds that if some 
$a_l$ changed such that $l>i-1$ then $a_{i-1}$ must also change. But we know 
that $a_i$ was changed and that $a_{i-1}$ was selected, contradicting the 
assumption.  
Hence, $C_i\neq \emptyset$.  

Third, assume that there exists some $l<i$ such that $C_l\neq\emptyset$. Assume 
by contradiction that not all sets $\{C_{l}, C_{l+1},\ldots,C_i\}$ are not 
empty, i.e, there exists some $C_d=\emptyset$ such that $l\leq d\leq i$. Since 
$C_l\neq\emptyset$ and $C_i\neq\emptyset$, it follows that $l<d<i$. Let there 
be some $E=(F\setminus C_1, \ldots, F\setminus C_k)$. Since $(F\setminus 
C_1,\ldots,F\setminus C_{i-1})$ is an explanation for the first $i-1$ steps of 
$\mathcal{E}$, and $d\leq i-1$, then $(F\setminus C_1,\ldots,F\setminus C_{d})$ 
is an explanation for the first $d$ steps. Thus, fixing the features in $E$ to 
their corresponding values (and allowing the features in $C$ to alternate 
arbitrarily)  forces the first $d$ actions to occur. 
Since $C_d=\emptyset$ then $E_d=F$, meaning it is entirely fixed, and thus 
alternating the values of any one of the sets: $(C_1,\ldots,C_{d-1})$ clearly 
cannot affect any of the actions $(a_d,\ldots,a_k)$. Particularly, since $l<d$, 
alternating the values of $C_l$ cannot cause actions $(a_d,\ldots,a_k)$ to 
change. Hence, there exists some 
$C'=(C_1,\ldots,C_{l-1},\emptyset,C_{l+1},\ldots,C_k)$, which is also a 
contrastive example for $\mathcal{E}$. $|C'|<|C|$, and hence, it again holds 
that $C$ is not minimal, contradicting the 
assumption. 
  
\end{proof}

%



\section{Deep Reinforcment Learning}
\label{appendix-deep-reinforcment-learning}

Deep reinforcement learning (DRL)~\cite{Li17} is a specific paradigm in 
machine
learning that seeks to learn models that will be deployed within complex
and reactive environments. In DRL, a DNN \textit{agent} is trained to
learn a \emph{policy} $\pi$, that maps an observed \emph{state} $s$ to
an \emph{action} $a$. The policy can be either deterministic or
stochastic, depending on the chosen setting and the various learning
algorithms.
During training, a \textit{reward} $r_t$ is assigned to the agent at each 
time-step $t\in{0,1,2...}$, based on the action $a_t$ performed at time-step 
$t$. Various DRL training algorithms leverage the reward 
differently~\cite{SuBa18, SuMcSi99, ShWoDh17}. However, the final goal is to 
find the optimal policy $\pi$ that maximizes the \textit{expected cumulative 
	discounted reward}. 
In recent years, DRL-trained agents have demonstrated promising results in a 
large variety of tasks, 
from game playing~\cite{MnKaSi13} to robotic navigation~\cite{MaCoFa21b}, and 
more. Since DRL-based agents are deployed 
within reactive systems --- various DRL verification tools \emph{unroll} the 
DRL agent 
for a finite number of steps, before verifying the property of interest among 
these encoded time-steps~\cite{AmScKa21, ElKaKaSc21}.

\section{Training the DRL Models}
\label{sec:appendix:Training}
In the following section, we go into further detail about the hyperparameters 
applied during training and the specific implementation methods used. The 
training process was executed using the \textit{BasicRL} 
baselines\footnote{\url{https://github.com/d-corsi/BasicRL}}.


\mysubsection{General parameters and algorithmic implementation.}
For the training, we exploited the Proximal Policy Optimization (PPO) algorithm 
based on an Actor-Critic structure. The strategy for the critic's training is a 
pure Monte Carlo approach without temporal difference rollouts. The 
actor network is updated periodically after a sequence of data collection 
episodes. 
The actor update rule follows the original implementation of~\cite{ShWoDh17}. 
For reproducibility, we set the same random seed for the Random,
NumPy and TensorFlow Python modules.

\mysubsection{Parameters for the GridWorld environment.} 
\begin{itemize}
	\item \textit{memory limit}: None
	\item \textit{gamma}: 0.99
	\item \textit{trajectory update frequency:} 10
	\item \textit{trajectory reduction strategy:} sum
	\item \textit{actor-network size:} 2 layers of 8 neurons each
	\item \textit{critic batch size:} 128
	\item \textit{critic epochs:} 60
	\item \textit{critic-network size:} same as actor
	\item \textit{PPO clip:} 0.2
	\item \textit{reward}: $+1$ for reaching the target and $0$ 
	otherwise
	\item \textit{random seeds:} $[207, 700]$
\end{itemize}

\mysubsection{Parameters for the TurtleBot environment.} 
\begin{itemize}
	\item \textit{memory limit}: None
	\item \textit{gamma}: 0.99
	\item \textit{trajectory update frequency:} 10
	\item \textit{trajectory reduction strategy:} sum
	\item \textit{actor-network size:} 2 layers of 32 neurons each
	\item \textit{critic batch size:} 128
	\item \textit{critic epochs:} 60
	\item \textit{critic-network size:} same as actor
	\item \textit{PPO clip:} 0.2
	\item \textit{reward}: same as~\cite{AmCoYeMaHaFaKa23}
	\item \textit{random seeds:} $[49, 80, 99, 211, 233]$
\end{itemize}

All original agents can be found in our publicly-available artifact 
accompanying this paper~\cite{ArtifactRepository}. 

\section{Property Constraints \& Transition Functions}
\label{sec:appendix:TransitionEncoding}

Next, we provide details regarding the transition functions of both benchmarks. 
This, in turn, defined the queries which we dispatched to our backend verifier 
(we used \marabou~\cite{KaHuIbJuLaLiShThWuZeDiKoBa19}, which was previously 
used in additional 
settings~\cite{AmWuBaKa21,AmZeKaSc22,AmFrKaMaRe23,AmMaZeKaSc23, 
CaKoDaKoKaAmRe22,CoYeAmFaHaKa22, ReKa22,OsBaKa22, LaKa21, JaBaKa20}). We also 
note 
that in 
order to speed verification for the GridWorld queries, we also configured Marabou to incorporate the 
\textit{Gurobi} LP solver\footnote{\url{https://www.gurobi.com/}}.

\subsection{GirdWorld}

\mysubsection{Inputs.}
The DRL-based agent has $8$ inputs in total. These represent the location of 
the agent and the target, as well as discrete sensor reading values indicating 
the 
closest obstacle in each direction. More specifically, the DRL-based agent 
receives: 

\begin{itemize}
	\item $2$ inputs ($x_0$,$x_1$) representing the discrete 2D coordinates of 
	the 
	\emph{agent}.
	
	\item $2$ inputs ($x_2$,$x_3$) representing the discrete 2D coordinates of 
	the 
	\emph{target}.
	
	\item $4$ input \emph{sensor readings} ($x_4$,$x_5$,$x_6$,$x_7$) indicating 
	if the 
	agent 
	senses an obstacle in one of four directions: \up, \down, \leftOutput, or 
	\rightOutput. 
	
\end{itemize}

\mysubsection{Outputs.}
The agent has $4$ outputs, each representing one of four possible actions to 
move in the current step: \up, \down, \leftOutput, or \rightOutput.

\mysubsection{Trivial Bounds.} 
\begin{itemize}
	
	\item 
	all the DNN's inputs are normalized to the range $[0, 1]$.
	
	\item 
	the \emph{location} inputs (i.e., $x_i$ for $i\in[0,1,2,3]$) have a value 
	$v\in\{0.1, 0.2, \ldots, 1\}$. Each of these values 
	represents a separate location on one of the axes of the $10 X 10$ grid.
	
	\item 
	the \emph{sensor reading} inputs (i.e., $x_i$ for $i\in[4,5,6,7]$) have a 
	value $v\in\{0, \frac{1}{2}, 
	1\}$ indicating if, and how far, an obstacle is in the relevant direction. 
	For 
	example, if for the \rightOutput input sensor reading, the value is $1$, 
	then if the 
	agent will decide to move \rightOutput, it will collide; if the sensor 
	reading 
	value is $\frac{1}{2}$, then there is an obstacle two steps to the right 
	(and 
	hence two subsequent \rightOutput actions will result in a collision). 
	If the sensor reading is zero, then the closest obstacle on the right 
	direction is 
	at least three steps away from the current state.
	
\end{itemize}

\mysubsection{Transitions.} We will elaborate on the transitions for moving 
in a given direction $d\in\{\leftOutput,\rightOutput, \up, \down \}$  
(the transition function is symmetric along all axes and so it encodes all 
possible transitions). For a movement in direction $d$ at some time-step 
$t$:

	\begin{itemize}

		\item 
		\textit{agent's location on the axis matching the direction d: } 
		$x_{d-axis}^{t} = x_{d-axis}^{t+1} \pm 0.1$ (the sign depends on d)
		
		\item 
		\textit{agent's location on the axis orthogonal to the direction d: } 
		$x_{orthogonal-d-axis}^{t} = x_{orthogonal-d-axis}^{t+1}$
		
		\item 
		\textit{target's location does not change: } 
		$x_{2}^{t} = x_{2}^{t+1}$, 	$x_{3}^{t} = x_{3}^{t+1}$
		
		\item 
		\textit{obstacle sensor reading in the direction of movement: } \\
		$x_{sensor-d}^{t} \leq x_{sensor-d}^{t+1} \leq x_{sensor-d}^{t} + 
		\frac{1}{2}$
		
		\item 
		\textit{obstacle sensor reading in the direction of movement: } \\
		$x_{sensor-d}^{t} + x_{sensor-d}^{t+1} \in \{0,\frac{1}{2}, 1\}$
		
		\item 
		\textit{obstacle sensor reading in the opposite direction of movement: 
		} \\
		$x_{sensor-opposite-d}^{t} - \frac{1}{2} \leq 
		x_{sensor-opposite-d}^{t+1} 
		\leq x_{sensor-opposite-d}^{t} $
		
		\item 
		\textit{obstacle sensor reading in the opposite direction of movement: 
		} \\
		$x_{sensor-opposite-d}^{t} + x_{sensor-opposite-d}^{t+1} \in 
		\{0,\frac{1}{2}, 1\}$

	\end{itemize}

\subsection{TurtleBot}

\mysubsection{Inputs.}
The DRL-based agent has $9$ inputs in total:
\begin{itemize}
	\item $7$ inputs: ($x_0, x_1,\ldots, x_6$) representing the \emph{lidar 
	sensors}. Each set of subsequent inputs represents lidar sensors aimed at 
	$30\degree$ 
	between one another.

	\item $1$ input ($x_7$) indicating the \emph{angle} between the agent 
	and 
	the target. 

	\item $1$ input ($x_8$) indicating the \emph{distance} between the agent 
	and 
	the target.
\end{itemize}

\mysubsection{Outputs.}
The agent also has $3$ outputs: $<y_0, y_1, y_2>$, that correspond to the 
actions $<\forwardOutput, \leftOutput, \rightOutput>$.

\mysubsection{Trivial Bounds.} 
all the DNN's inputs are normalized 
to the range $[0, 1]$. 

\mysubsection{Transitions.} 
For simplicity, we focused on properties in which each one of the steps 
(except, perhaps, the last) is either \rightOutput or \leftOutput 
(see~\cite{AmCoYeMaHaFaKa23}):
\begin{itemize}
	\item \rightOutput action (output at time-step $t$):  
	\begin{itemize}
		\item \textit{bounds}:
		$x_i^{t} \in [0.2, 1] 
		\text{ 
			for } i = 
		[0, 1, 2, 3, 4, 5, 6, 8] \text{ and }\land 
		x_7^{t} \in 
		[0, 1]$
		\item \textit{lidar ``sliding 
		window''}: 		
		$\text{ 
			for } i = 
		[1, 2, 3, 4, 5, 6]$: $x_i^{t} = x_{i-1}^{t+1}$
		
		\item \textit{turn $30\degree$ to the right}: 
		$x_7^{t+1}=x_7^{t}-\frac{1}{12}$
		
		\item \textit{distance to target does not change}:
		$x_8^{t}=x_8^{t+1}$
	\end{itemize}

	\item \leftOutput action (output at time-step $t$):  
\begin{itemize}
	\item \textit{bounds}:
	$x_i^{t} \in [0.2, 1] 
	\text{ 
		for } i = 
	[0, 1, 2, 3, 4, 5, 6, 8] \text{ and }\land 
	x_7^{t} \in 
	[0, 1]$
	\item \textit{lidar ``sliding 
		window''}: 		
	$\text{ 
		for } i = 
	[1, 2, 3, 4, 5, 6]$: $x_{i-1}^{t} = x_{i}^{t+1}$
	
	\item \textit{turn $30\degree$ to the left}: 
	$x_7^{t+1}=x_7^{t}+\frac{1}{12}$
	
	\item \textit{distance to target does not change}:
	$x_8^{t}=x_8^{t+1}$
\end{itemize}

\end{itemize}

\section{Supplementary Results}
\label{sec:appendix:supplementaryResults}
\mysubsection{TurtleBot Results.} Table.~\ref{table:turtleBot} presents the full results of the four aforementioned approaches on the TurtleBot benchmark.

\begin{table} [H]
	\centering
\caption{\textit{TurtleBot}: columns from left to right: experiment type, 
	method name, 
	method number, time and size of the returned explanation (out of 
	experiments 
	that 
	terminated), the percent of solved instances (the rest timed out), and a 
	column 
	indicating whether the 
	explanation is guaranteed to be minimal. The bold highlighting indicates 
	the 
	method that generated the 
	explanation with the optimal size.}
\label{table:turtleBot}

	\centering
 	\scalebox{1.0}{

	\begin{tabular}{c||c||c||c||ccc||c||c} 
		\hline
		\multirow{2}{*}{\textbf{setting}}                                       
		& \multirow{2}{*}{\textbf{experiment}} & \multirow{2}{*}{\textbf{M}} 
		& \textbf{time (s)} & \multicolumn{3}{c||}{\textbf{size}}         & 
		\multirow{2}{*}{\begin{tabular}[c]{@{}c@{}}\textbf{solved}\\\textbf{(\%)}\end{tabular}}
		& 
		\multirow{2}{*}{\begin{tabular}[c]{@{}c@{}}\textbf{guaranteed}\\\textbf{minimality}\end{tabular}}
		\\
		&                                      &                             & 
		\textbf{avg.}     & \textbf{min} & \textbf{avg.} & \textbf{max} 
		&                                                                       
		&
		
		\\ 
		\hline
		\multirow{3}{*}{\begin{tabular}[c]{@{}c@{}}minimal\\(local)\end{tabular}}
		
		& one-shot                             & 1                           & 
		1,084             & 2            & 6             & 
		12           & 
		91.2                                                                    
		&
		
		\greencheck                                                             
		\\
		
		& independent                          & 
		2                          
		& 1                 & 4            & 24            & 
		54           & 
		100                                                                     
		&
		
		\redxmark                                                               
		\\
		
		& \textbf{incremental}                          & 
		\textbf{3}                           & 
		\textbf{764}               & \textbf{2}            & 
		\textbf{6}             & 
		\textbf{10}           & 
		\textbf{97.1}                                                           
		         
		&
		
		\greencheck                                                             
		\\
		
		\hline
		\multirow{4}{*}{\begin{tabular}[c]{@{}c@{}}minimum\\(global)\end{tabular}}
		
		& one-shot                             & 1                           & 
		2,228             & 2            & 6             & 
		10           & 
		27.1                                                                    
		&
		
		\greencheck                                                             
		\\
		
		& independent                          & 
		2                           & 
		77                & 2            & 17            & 
		37           & 
		100                                                                     
		&
		
		\redxmark                                                               
		\\
		
		& incremental                          & 3                           & 
		637               & 2            & 6             & 
		11           & 
		28.8                                                                    
		&
		
		\greencheck                                                             
		\\
		& \textbf{reversed}                              & 
		\textbf{4}                           & 
		\textbf{267}               & \textbf{2}            & 
		\textbf{5}             & 
		\textbf{10}           & 
		\textbf{96.8}                                                           
		         
		&
		
		\greencheck                                                             
		\\
		\hline
  
	\end{tabular}
 }
\end{table}

\mysubsection{Full Evaluation by Execution Size.} We present here the full 
analysis of the results, evaluated under different execution sizes (number of 
steps) both for the minimal and minimum explanation settings, for the two 
benchmarks.  Fig.~\ref{fig:res_gridworld_local} presents the results for 
GridWorld under the minimal explanation setting, and 
Fig.~\ref{fig:res_gridworld_full} for the minimum explanation setting.  
Fig.~\ref{fig:res_turtlebot_local} presents the results for TurtleBot under the 
minimal explanation setting, and Fig.~\ref{fig:res_turtlebot_full} for the 
minimum explanation setting

\begin{figure}[H]
	\centering
	\captionsetup{justification=centering}
	\begin{center}
		\includegraphics[width=0.35\linewidth]{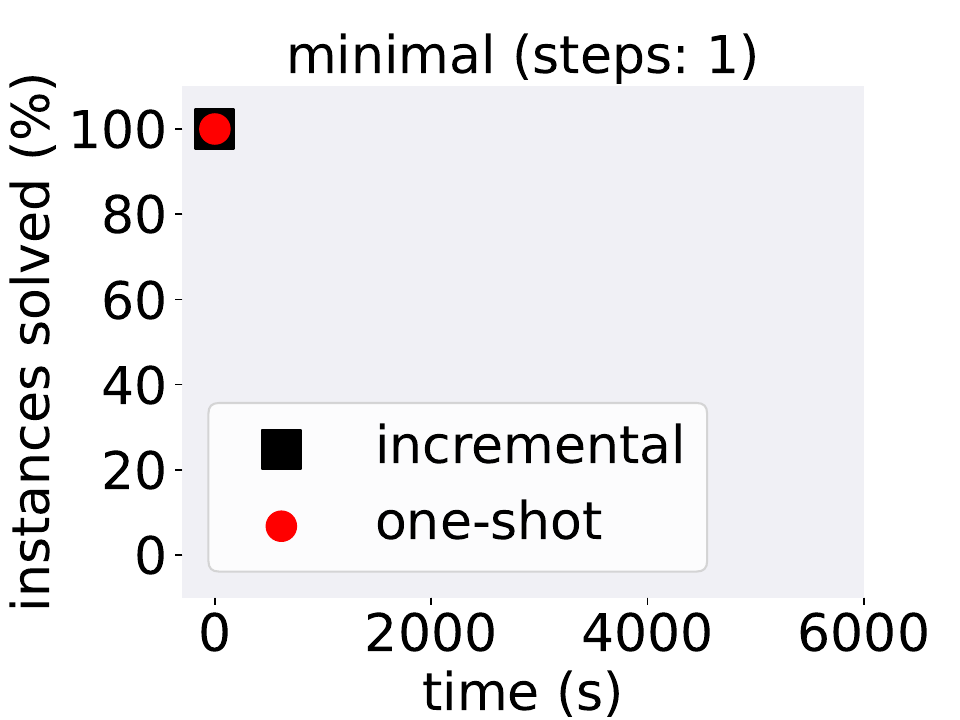}
		\includegraphics[width=0.35\linewidth]{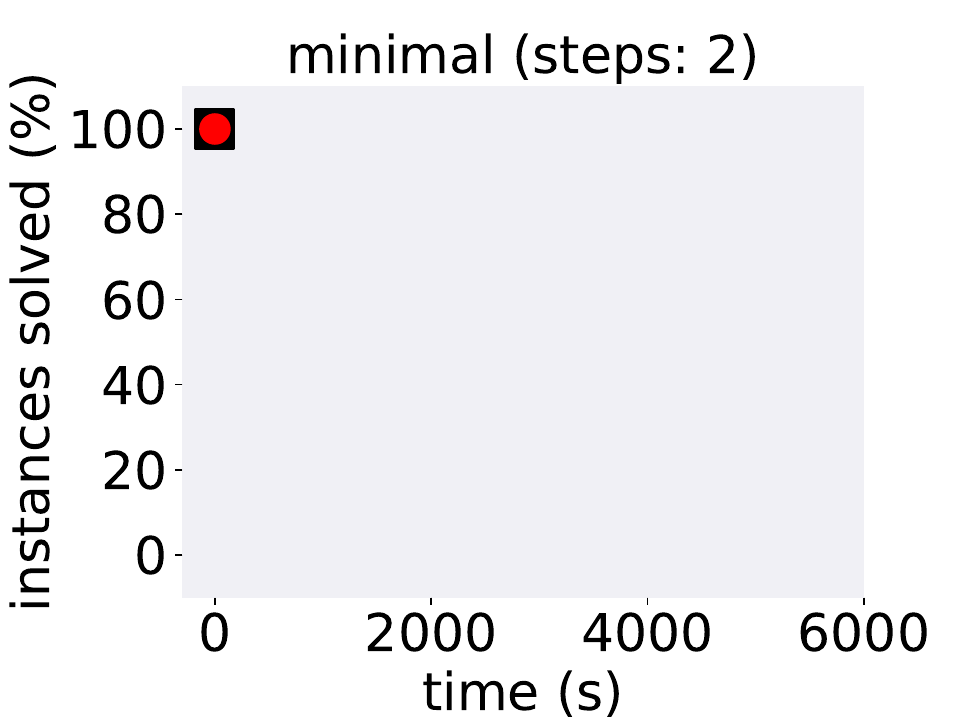}
		\includegraphics[width=0.35\linewidth]{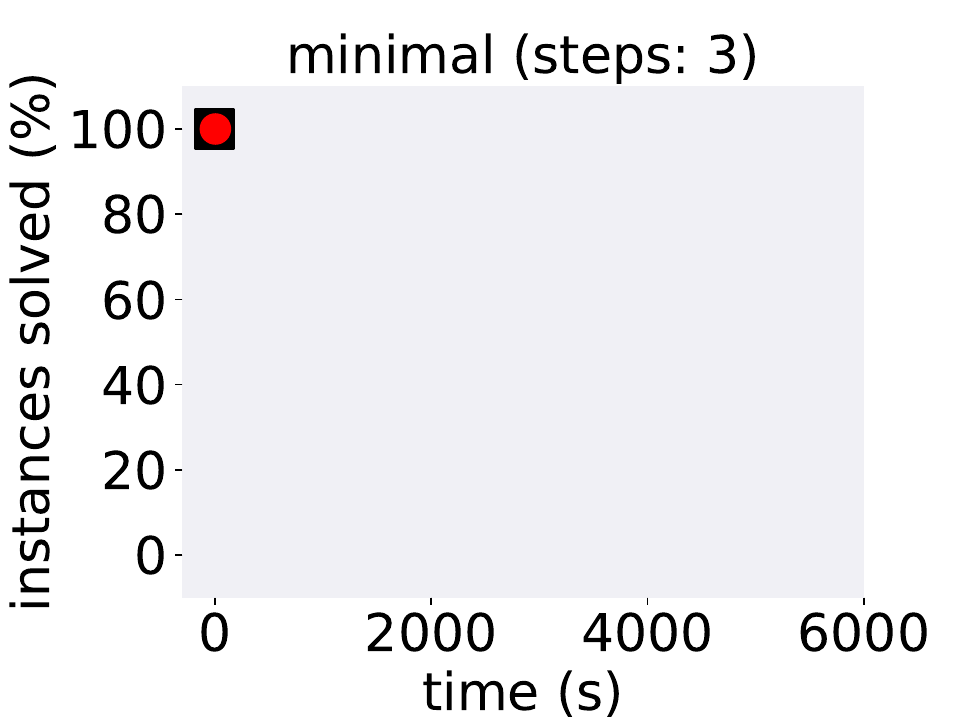}
		\includegraphics[width=0.35\linewidth]{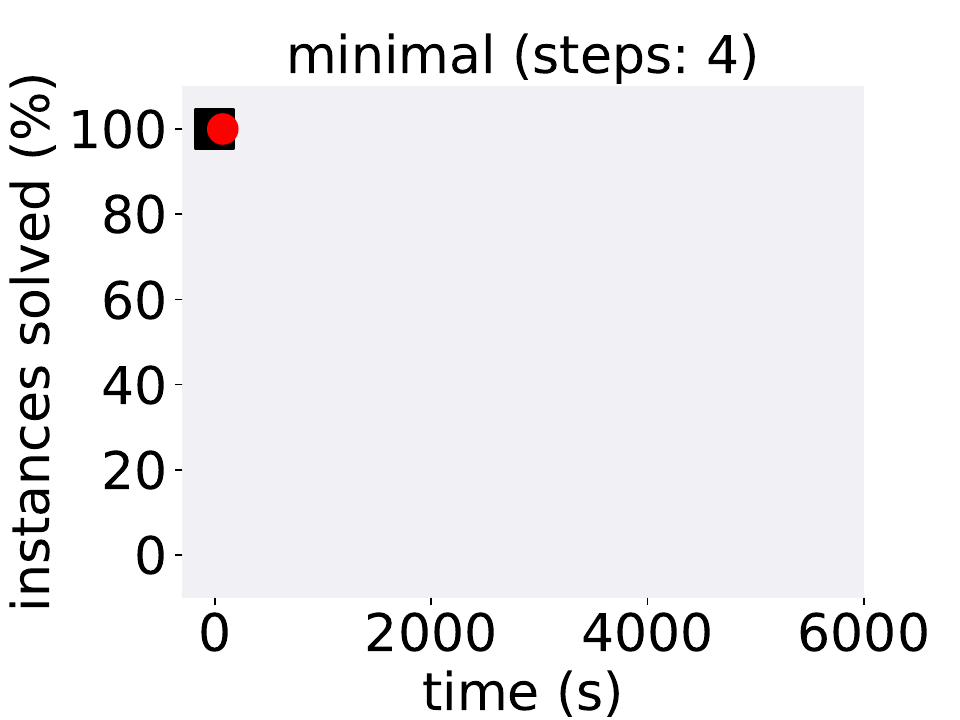}
		\includegraphics[width=0.35\linewidth]{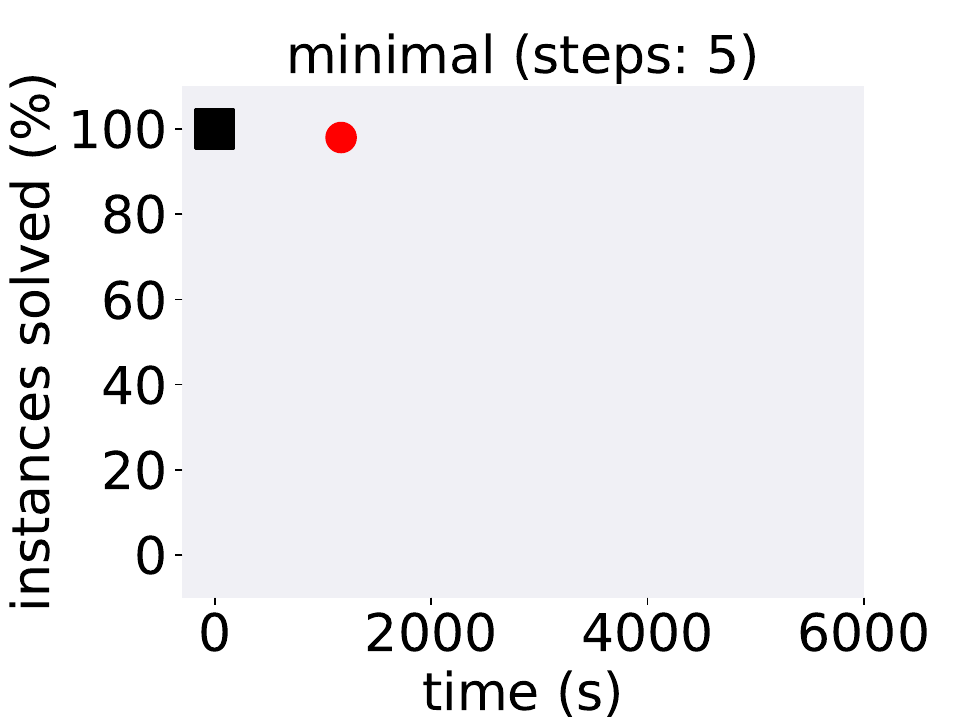}
		\includegraphics[width=0.35\linewidth]{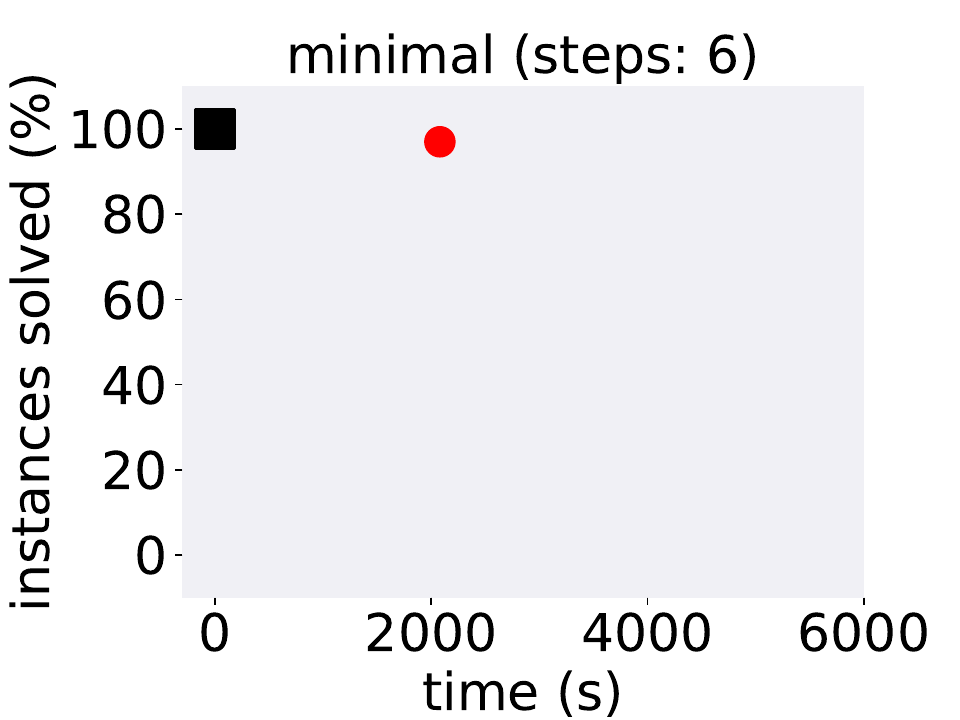}
		\includegraphics[width=0.35\linewidth]{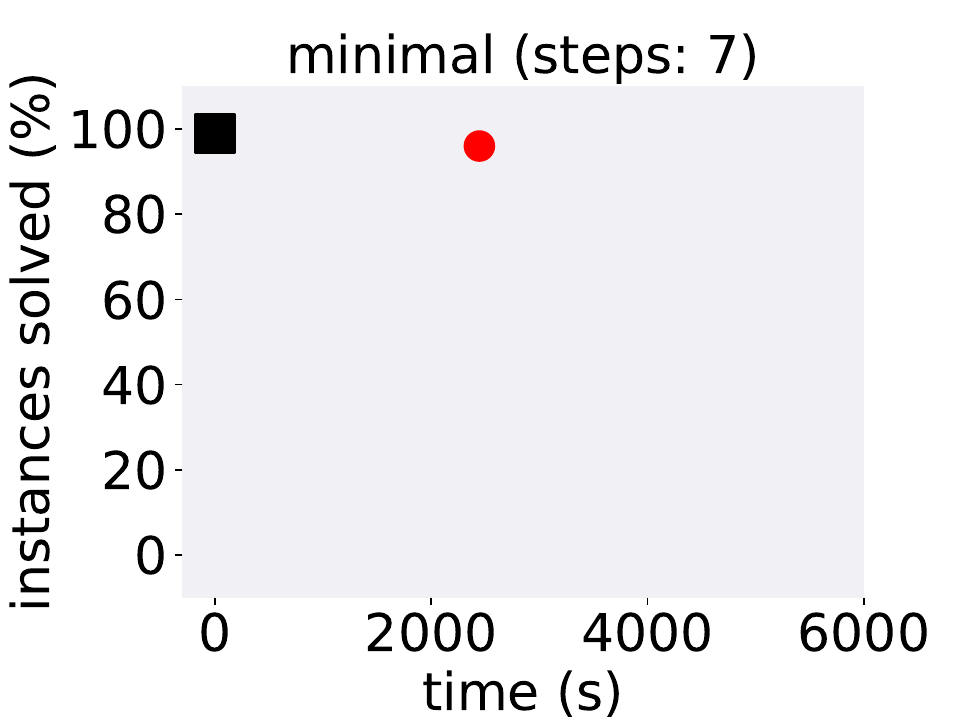}
		\includegraphics[width=0.35\linewidth]{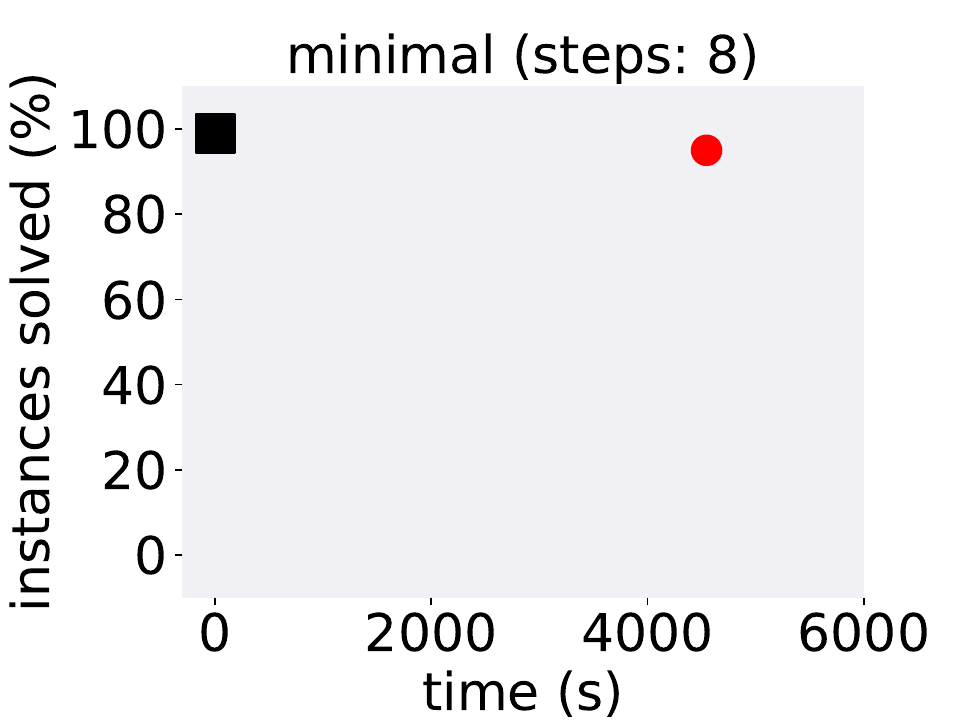}
	\end{center}
		\caption{\textit{GridWorld}: solved instances of minimal explanation 
		search, by (accumulative) time, and $1\leq k \leq 8$ steps.}
	\label{fig:res_gridworld_local}
\end{figure}

\begin{figure}[H]
	\centering
	\captionsetup{justification=centering}
	\begin{center}
		\includegraphics[width=0.35\linewidth]{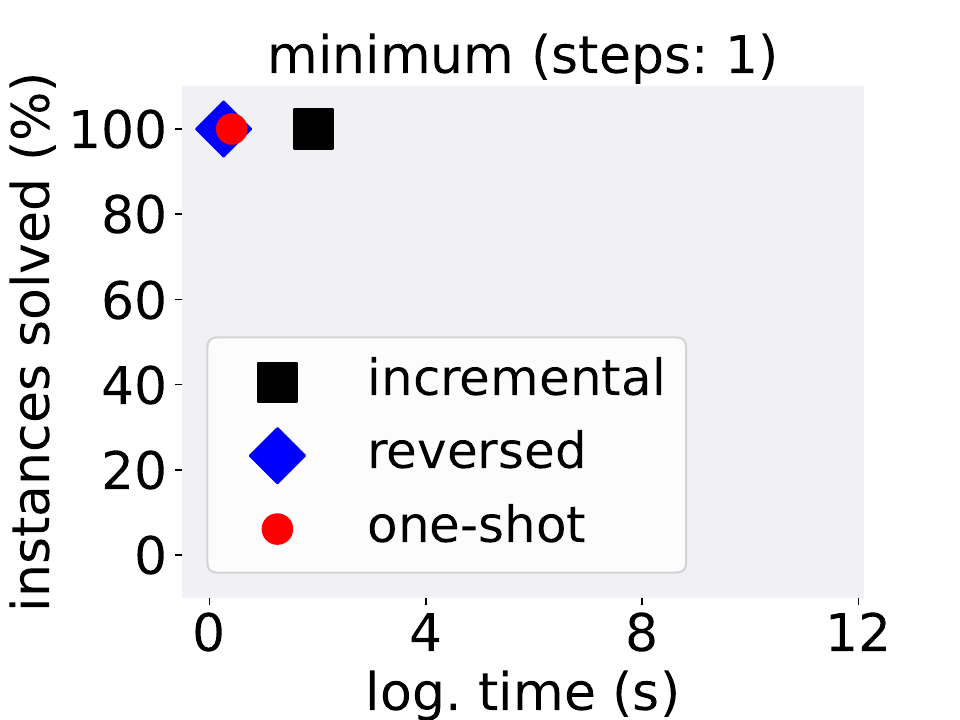}
		\includegraphics[width=0.35\linewidth]{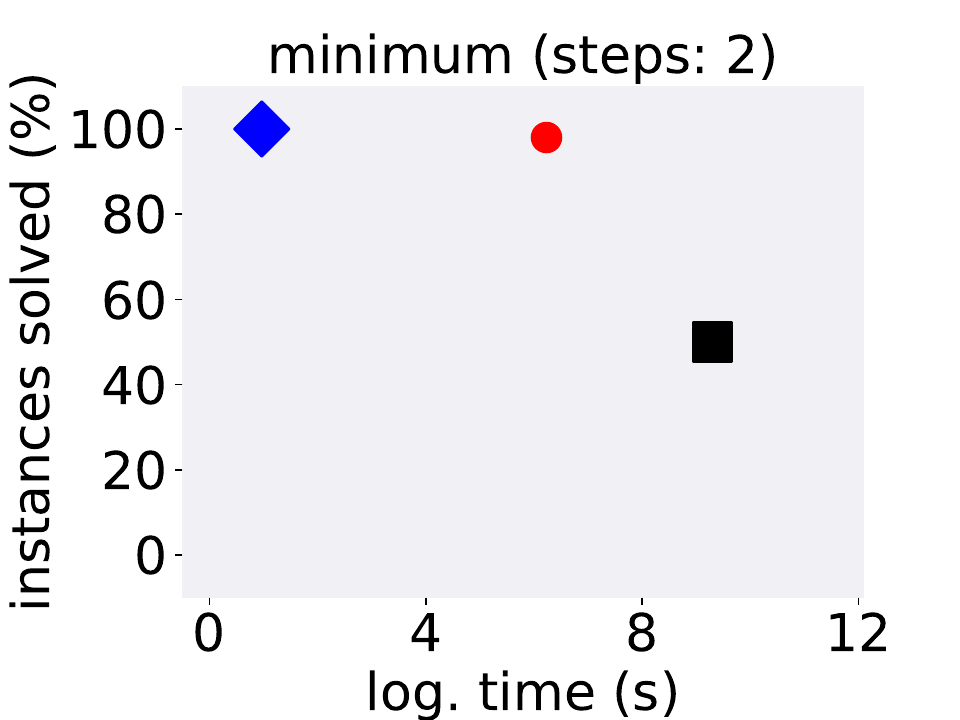}
		\includegraphics[width=0.35\linewidth]{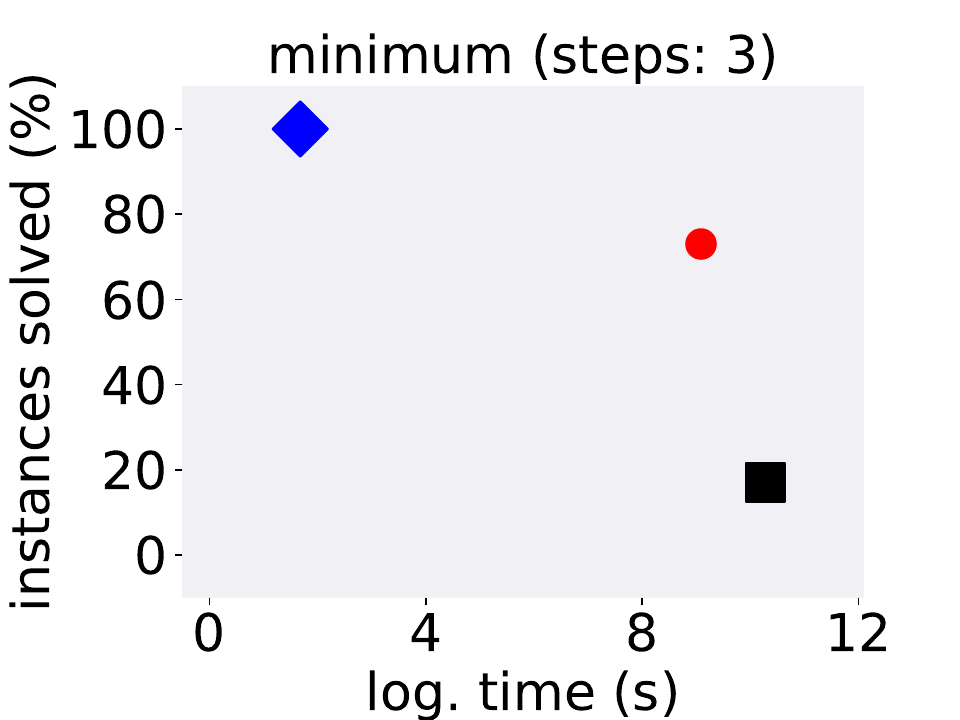}
		\includegraphics[width=0.35\linewidth]{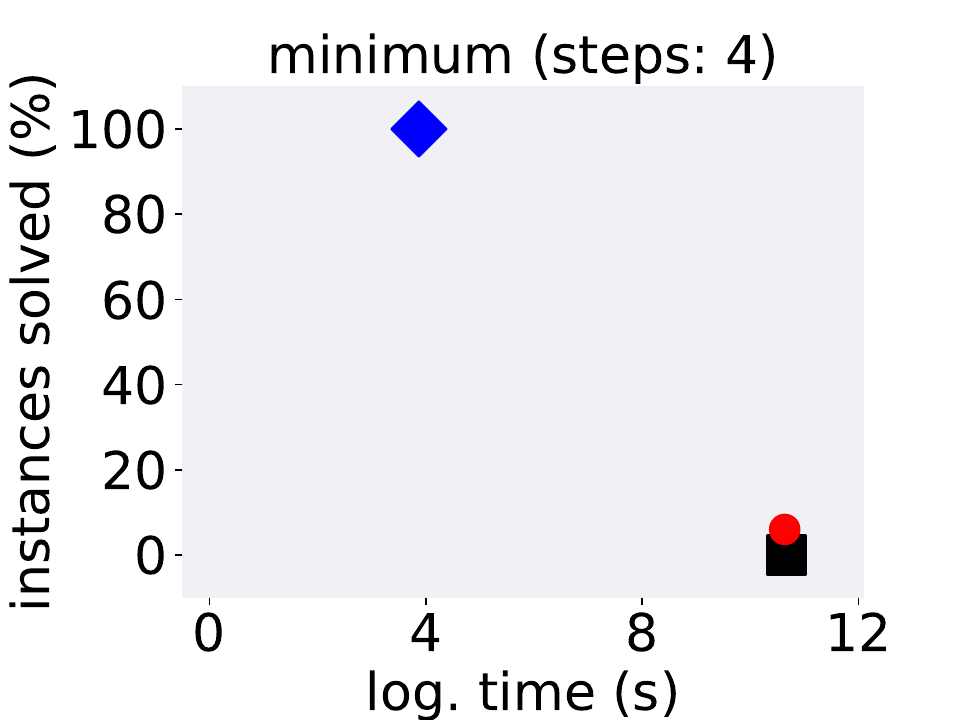}
		\includegraphics[width=0.35\linewidth]{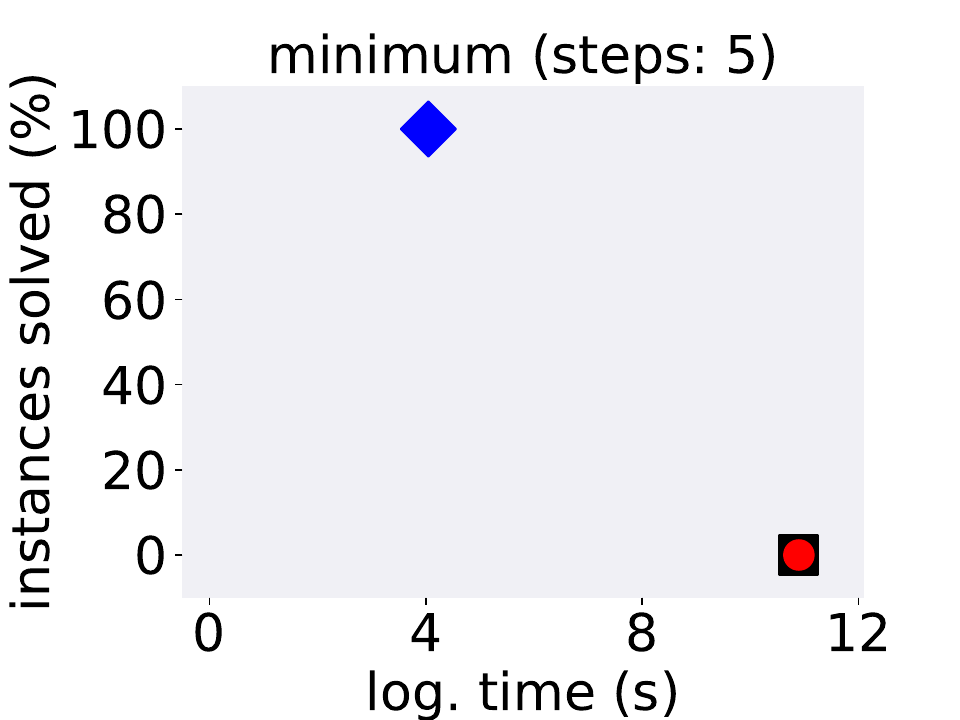}
		\includegraphics[width=0.35\linewidth]{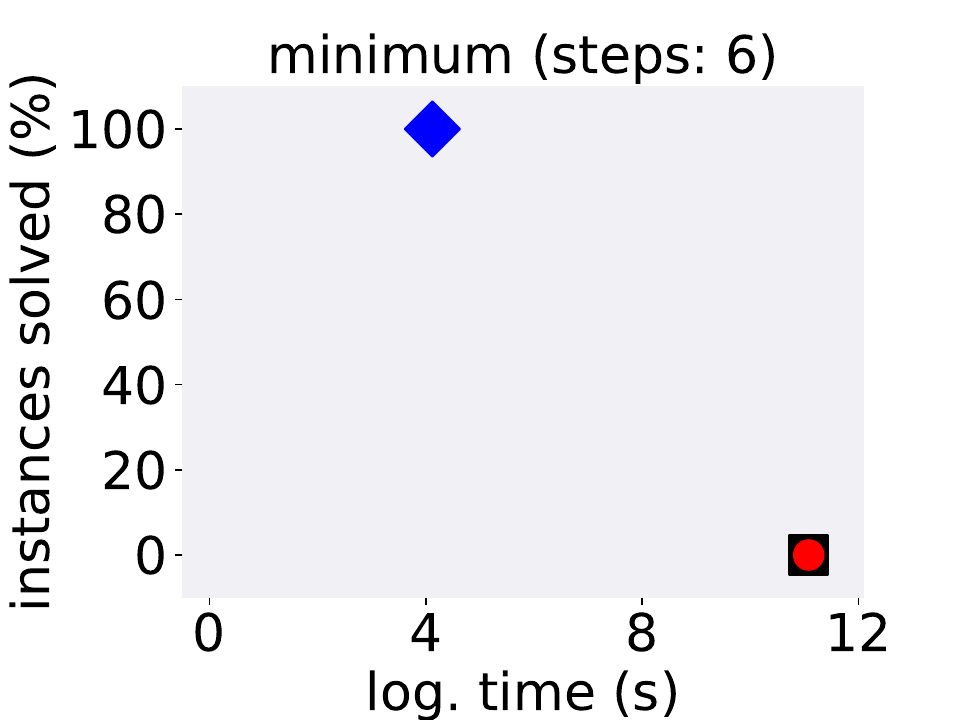}
		\includegraphics[width=0.35\linewidth]{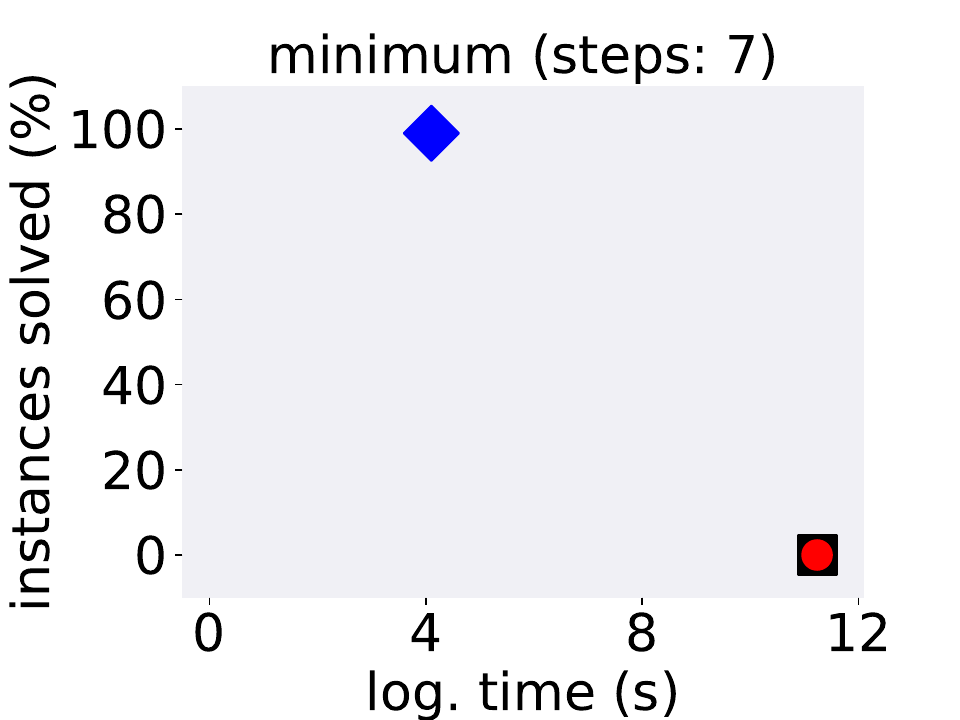}
		\includegraphics[width=0.35\linewidth]{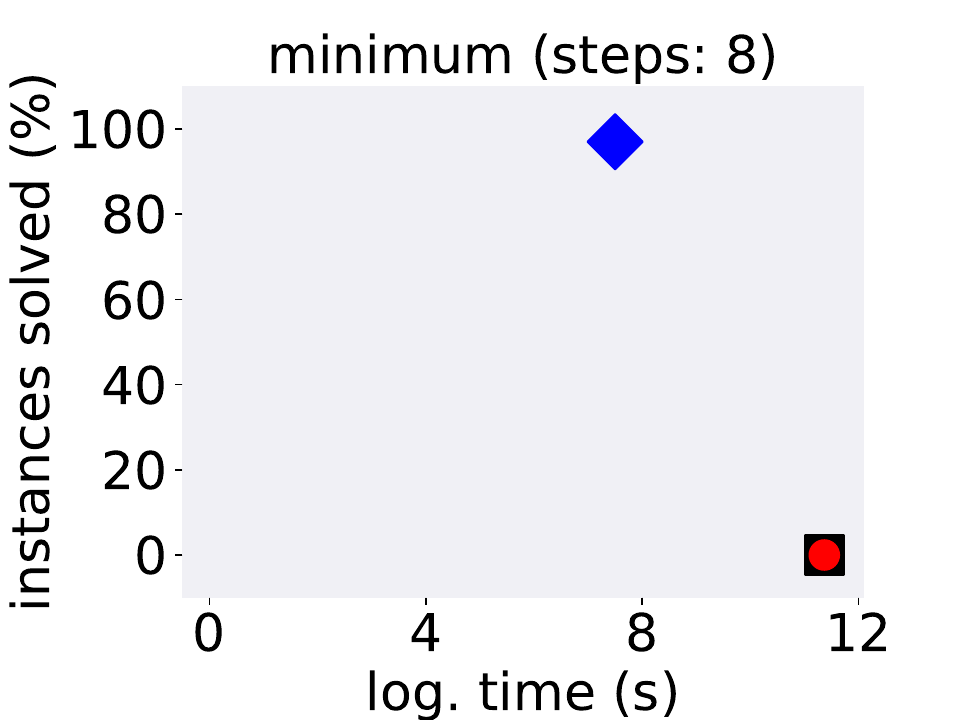}
	\end{center}
			\caption{\textit{GridWorld}: solved instances of minimum 
			explanation 
		search, by (accumulative) time, and $1\leq k \leq 8$ steps.}
	\label{fig:res_gridworld_full}
\end{figure}

\begin{figure}[H]
	\centering
	\captionsetup{justification=centering}
	\begin{center}
		\includegraphics[width=0.35\linewidth]{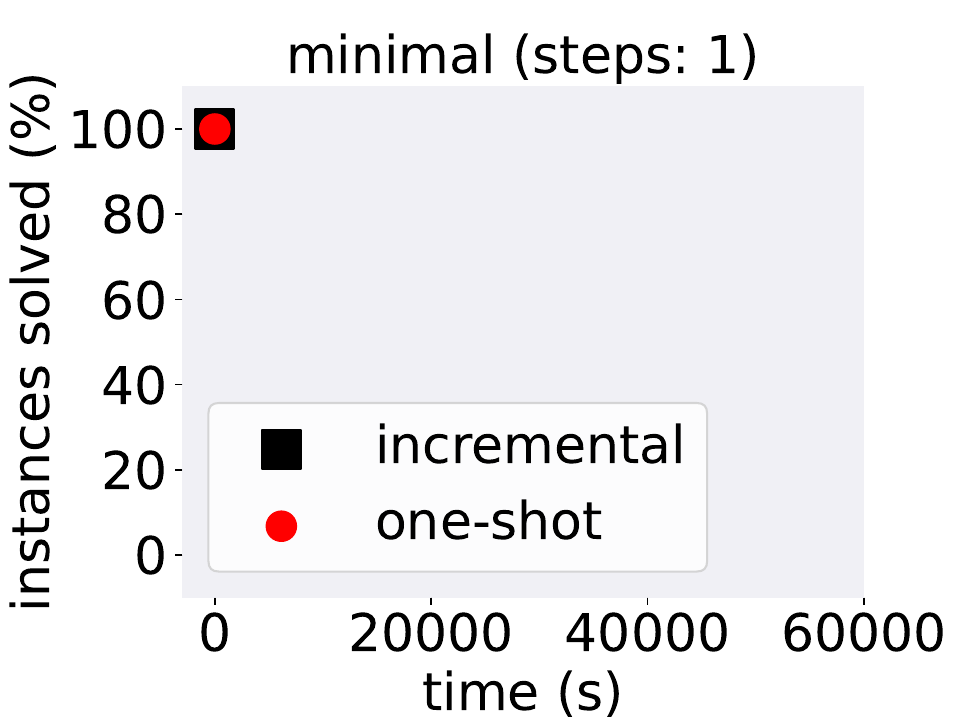}
		\includegraphics[width=0.35\linewidth]{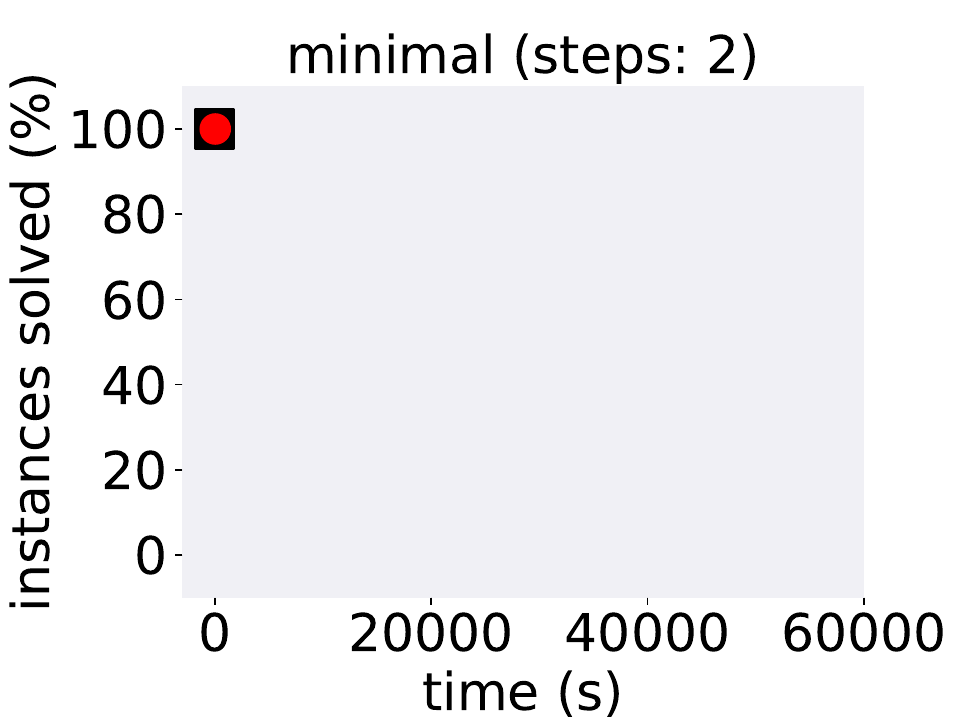}
		\includegraphics[width=0.35\linewidth]{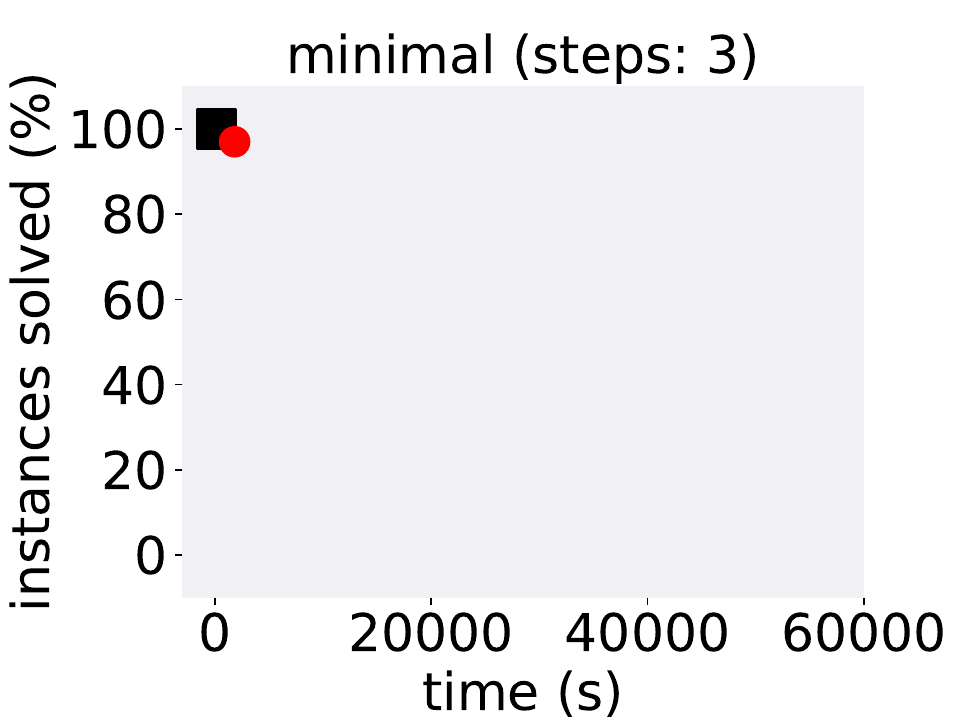}
		\includegraphics[width=0.35\linewidth]{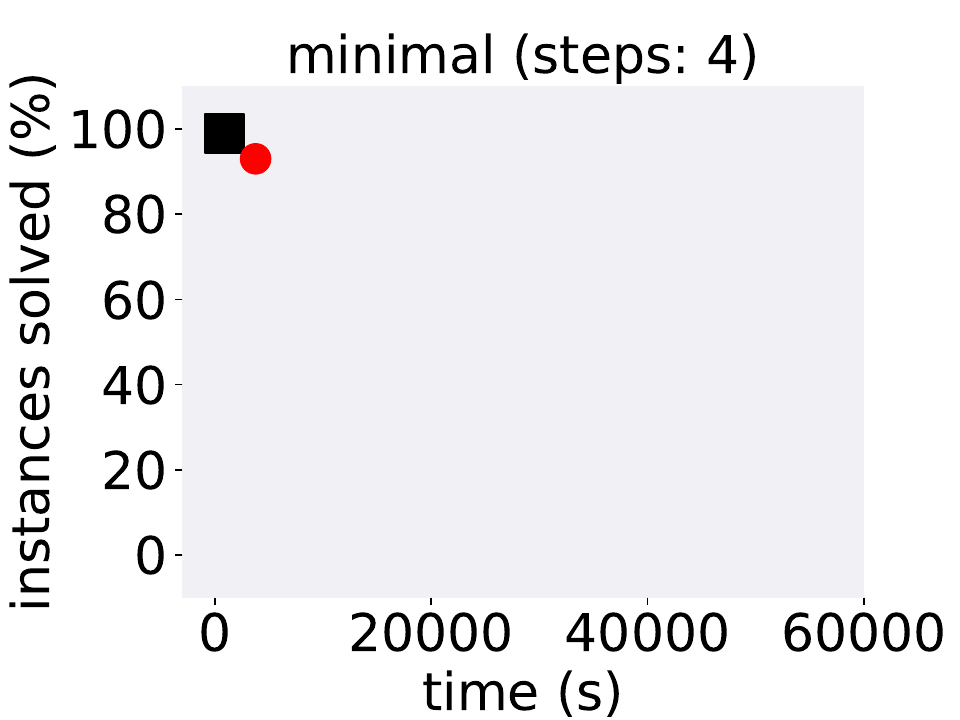}
		\includegraphics[width=0.35\linewidth]{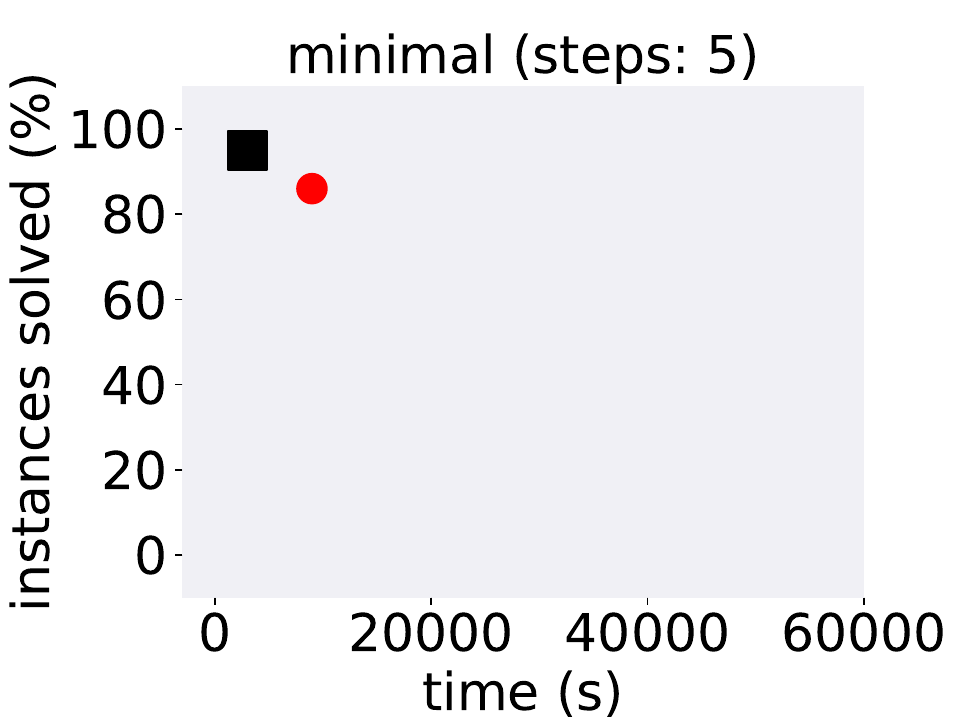}
		\includegraphics[width=0.35\linewidth]{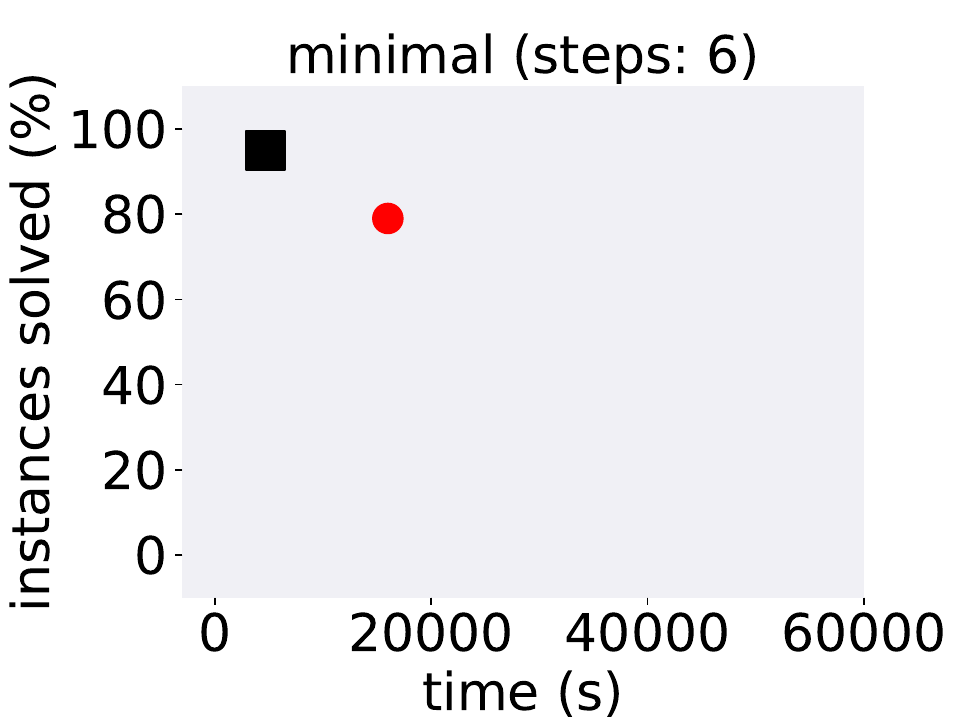}
		\includegraphics[width=0.35\linewidth]{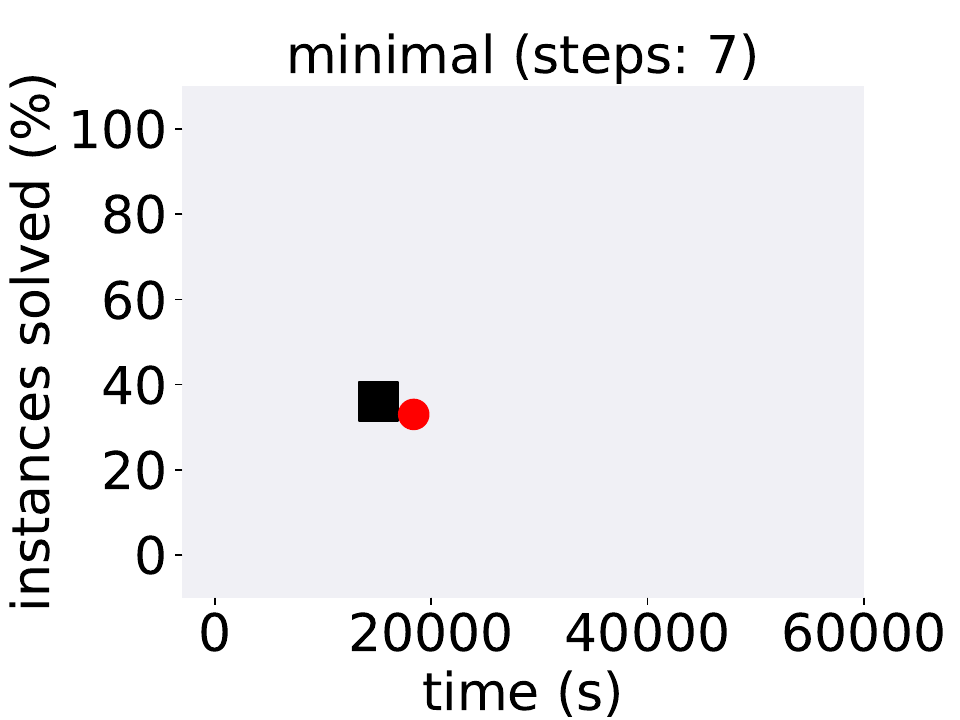}
		\includegraphics[width=0.35\linewidth]{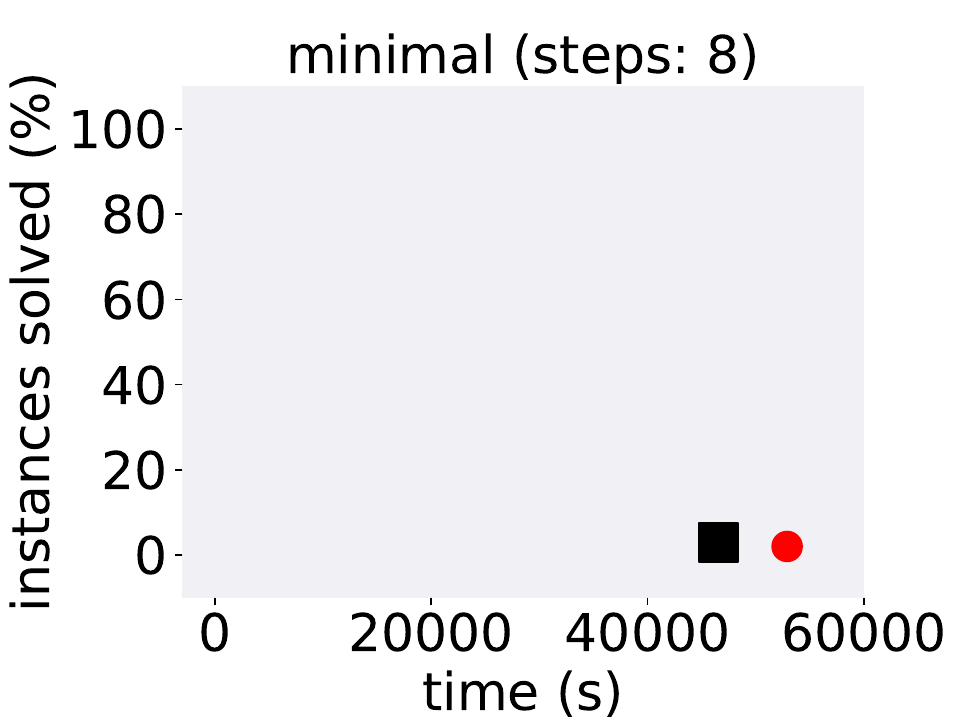}
	\end{center}
		\caption{\textit{TurtleBot}: solved instances of minimal explanation 
	search, by (accumulative) time, and $1\leq k \leq 8$ steps.}
	\label{fig:res_turtlebot_local}
\end{figure}

\begin{figure}[H]
	\centering
	\captionsetup{justification=centering}
	\begin{center}
		\includegraphics[width=0.35\linewidth]{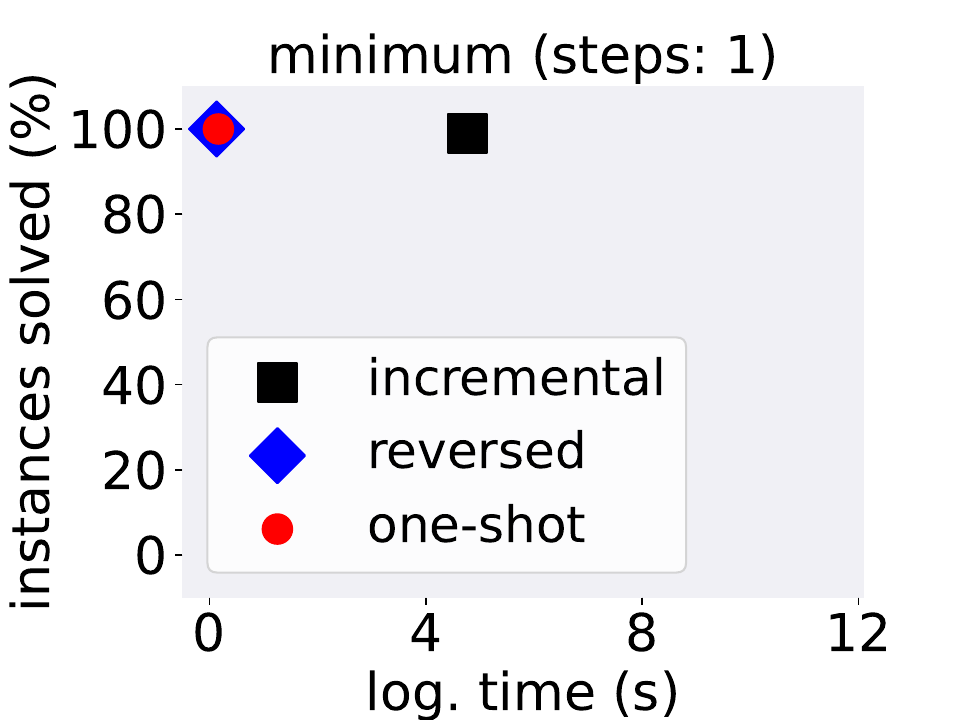}
		\includegraphics[width=0.35\linewidth]{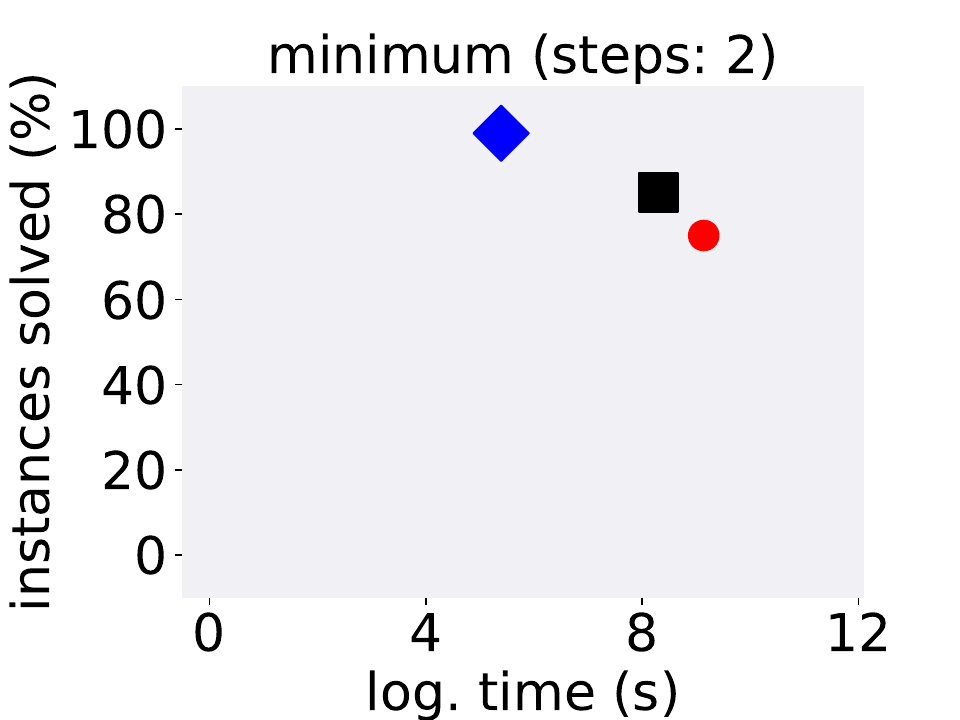}
		\includegraphics[width=0.35\linewidth]{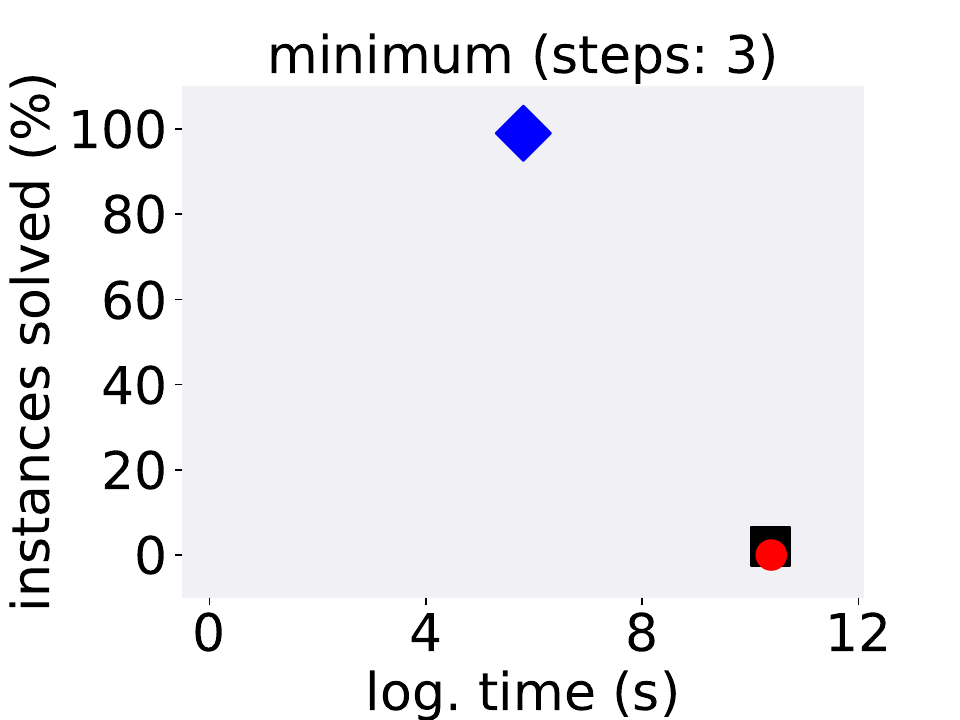}
		\includegraphics[width=0.35\linewidth]{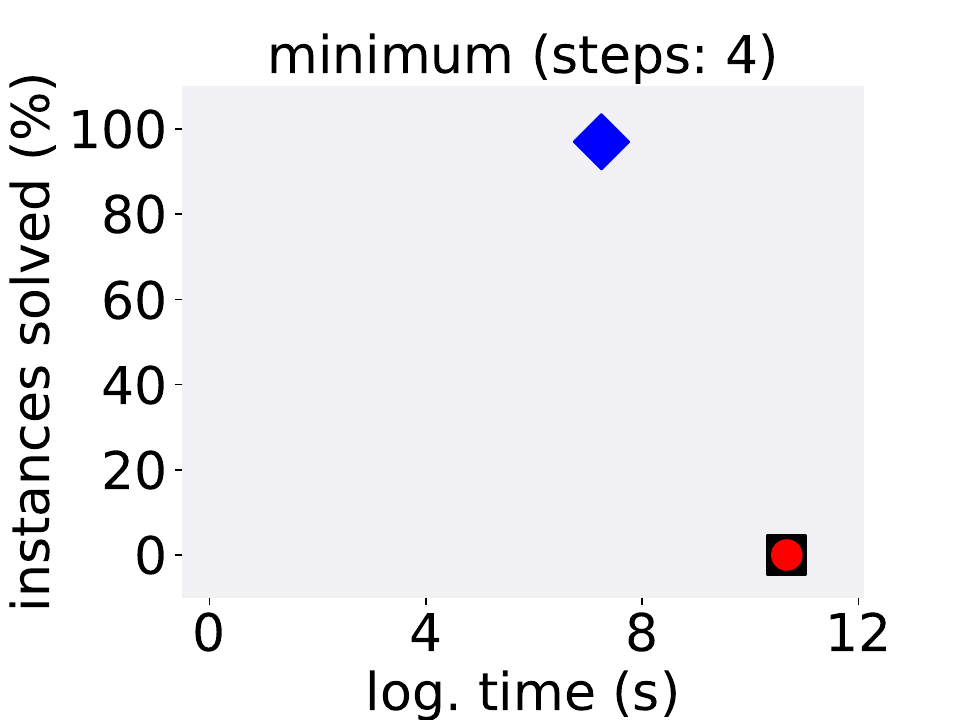}
		\includegraphics[width=0.35\linewidth]{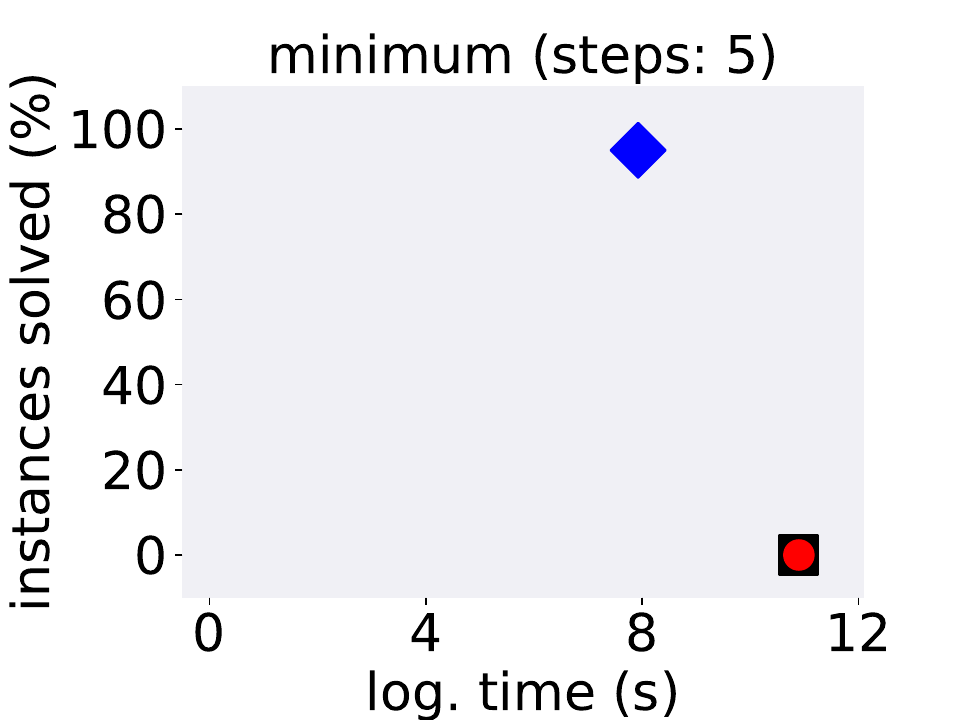}
		\includegraphics[width=0.35\linewidth]{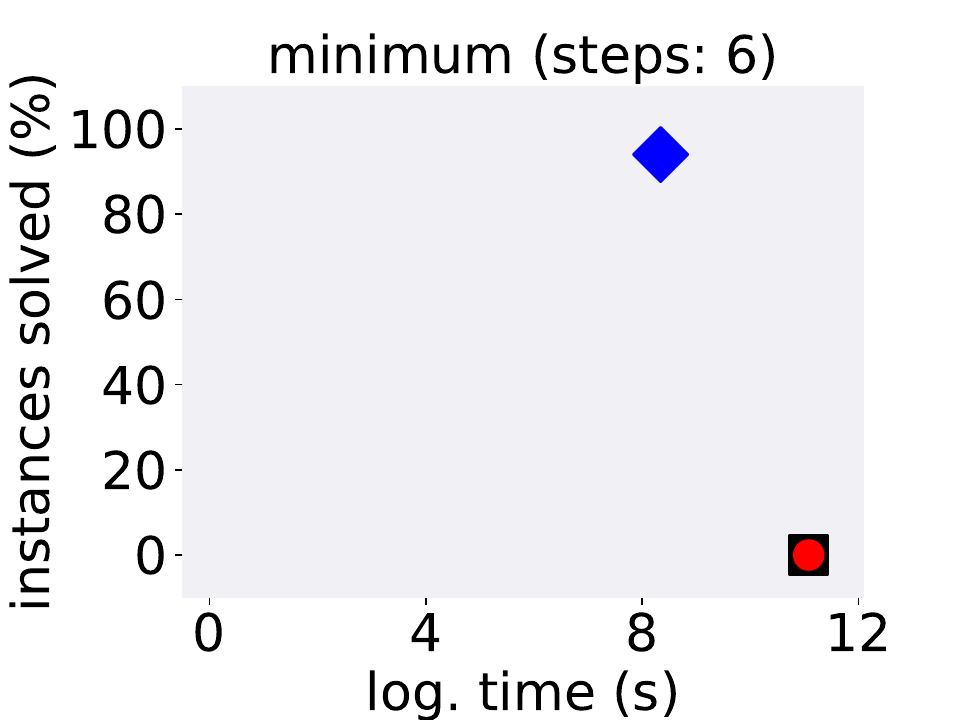}
		\includegraphics[width=0.35\linewidth]{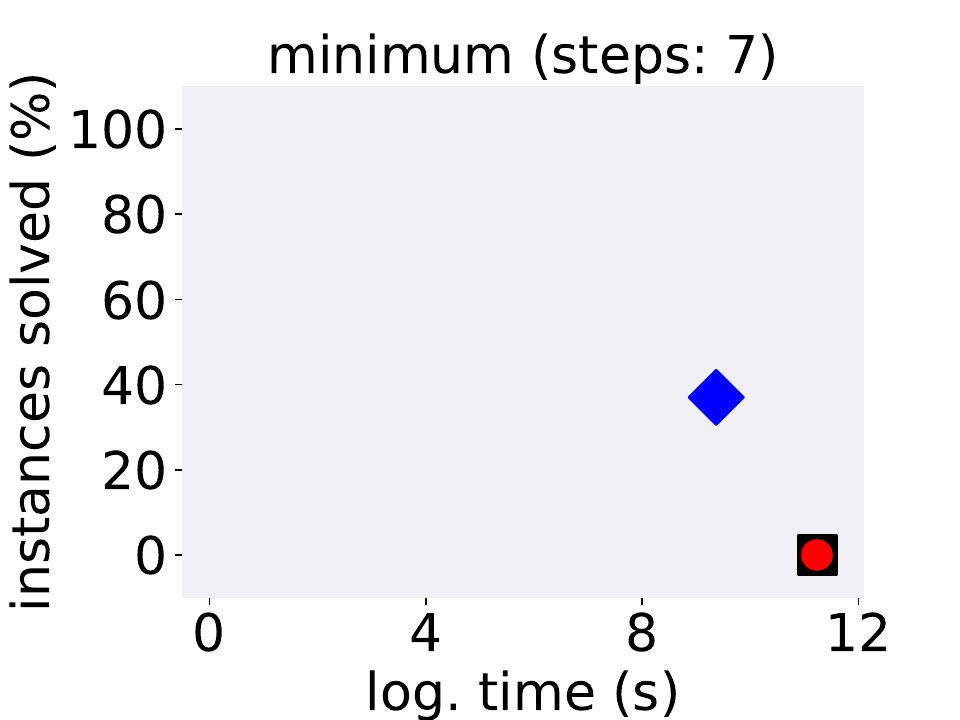}
		\includegraphics[width=0.35\linewidth]{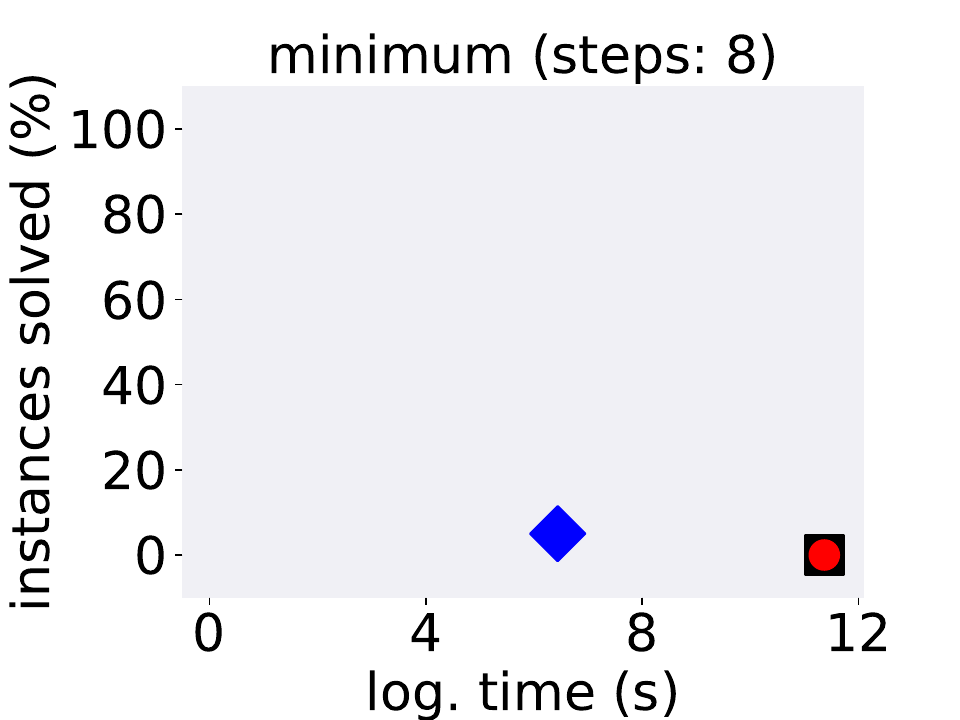}
	\end{center}
	\caption{\textit{TurtleBot}: solved instances of minimum explanation 
	search, by (accumulative) time, and $1\leq k \leq 8$ steps.}
	\label{fig:res_turtlebot_full}
\end{figure}

\newpage

\section{Comparison to Heuristic XAI Methods}
\label{sec:appendix:heuristicMethods}

Many heuristic explainable RL (XRL) methods
intervene in the training phase~\cite{juozapaitis2019explainable,
	amir2018highlights, madumal2020explainable}, and are thus unsuitable
for providing a feature-level, post-hoc explanation --- and
consequently, are incomparable to our approach.  Instead, we focused
on approaches similar to the ones suggested
in~\cite{vouros2022explainable, rizzo2019reinforcement,
	zhang2020explainable, dethise2019cracking}, which generate
explanations for DRL agents using feature-level XAI methods. We
studied two popular methods: \lime~\cite{ribeiro2016should}, and
\shap~\cite{lundberg2017unified}. Specifically, we compared our best-performing 
method, i.e., the reverse incremental enumeration method
(Method 4) to these approaches. We follow common
conventions~\cite{ignatiev2019validating} for comparing between these
heuristic methods and (our) formal XAI methods --- and allow \lime and
\shap to select explanations of the same size as our generated
explanations.
For each trace, we check whether the explanation produced by these competing 
methods (on this multi-step sequence) is valid. This is 
done by checking whether it is a valid hitting set of the produced contrastive 
examples~\cite{ignatiev2019abduction, ignatiev2020contrastive}. 

Our results (summarized in Table~\ref{table:heuristic-xai-comparison}) 
demonstrate the usefulness of our verification-driven method. Although \lime
and \shap are highly scalable, they tend to generate skewed explanations. This 
is often the case even for a single-step execution. This finding is in 
line with previous research~\cite{ignatiev2019validating, 
	camburu2019can}. In addition, it is apparent that when increasing the 
number of 
steps in the execution, the correctness of the explanations provided
by these approaches 
decreases drastically. We believe this is compelling evidence for the 
significance of our approach in generating formally provable, multi-step 
explanations of executions, which can only rarely be correctly generated by 
competing XAI approaches.

\begingroup
\setlength{\tabcolsep}{10pt}
\begin{table}[H]
	\centering
	\caption{Comparing non-verification approaches to our formal explainability 
		method. The columns indicate the ratio of correct results, per step.}
	\label{table:heuristic-xai-comparison}
	\centering
	\begin{tabular}{c||c||ccccccc} 
		\hline
		\multicolumn{9}{c}{\textbf{verified as formal explanations 
				(\%)}}                                                          
		
		\\
		
		\hline
		\multirow{2}{*}{\textbf{benchmark}} & 
		\multirow{2}{*}{\textbf{approach}} & \multicolumn{7}{c}{\textbf{steps 
				(k)}}                                                    \\ 
		\cline{3-9}
		&                                    & \textbf{1} & \textbf{2} & 
		\textbf{3} & \textbf{4} & \textbf{5} & \textbf{6} & \textbf{$\ge$7}  \\ 
		\hline
		\multirow{2}{*}{\textit{GridWorld}}            & 
		\lime                               & 15.0       & 1.0        & 
		2.0        & 0          & 0          & 0          & 0           \\
		& \shap                               & 2.0        & 0          & 
		0          & 0          & 0          & 0          & 0           \\ 
		\hline
		\multirow{2}{*}{\textit{TurtleBot}}            & 
		\lime                               & 20.0       & 0          & 
		0          & 2.1        & 1.1        & 0          & 0           \\
		& \shap                               & 23.0       & 2.0        & 
		2.0        & 1.0        & 0          & 0          & 0           \\
		\hline
	\end{tabular}
	
\end{table}

\end{document}